\documentclass[a4paper,oneside,12pt]{article}

\usepackage{graphicx,afterpage}
\usepackage{bm}
\usepackage{soul}
\usepackage{bibunits}
\usepackage{enumitem}
\defaultbibliography{references_arxiv}
\defaultbibliographystyle{plain}
\usepackage{lmodern}
\usepackage{anyfontsize}
\usepackage{empheq}
\usepackage{mathtools}
\usepackage{circuitikz}
\usepackage[top=0.6in, bottom=0.8in, left=0.6in, right=0.6in]{geometry}
\usepackage{mathrsfs}
\usepackage{stmaryrd}
\SetSymbolFont{stmry}{bold}{U}{stmry}{m}{n}
\usepackage[LGR,T1]{fontenc}
\usepackage[utf8]{inputenc}
\usepackage{color}
\usepackage{framed}
\usepackage{parskip}
\usepackage{cancel}
\usepackage{amsfonts}
\usepackage{upgreek}
\usepackage{multirow}
\usepackage{amssymb}
\usepackage{caption}
\usepackage[dvipsnames,svgnames]{xcolor}
\usepackage[linesnumbered,ruled,vlined]{algorithm2e}
\SetKwComment{Comment}{/*/ }{}

\SetCommentSty{mycommfont}
\usepackage{calc}
\usepackage{esint}
\usepackage{comment}
\usepackage{amsmath}
\usepackage{float}
\usepackage{amssymb}
\usepackage{amsthm}
\usepackage{graphicx}
\usepackage{relsize}
\usepackage{afterpage}
\usepackage[titletoc,title]{appendix}
\usepackage{authblk}
\usepackage{tikz}

\usepackage[colorlinks=true, allcolors=blue]{hyperref}
\hypersetup{CJKbookmarks,%
bookmarksnumbered,%
colorlinks,%
linkcolor=black,%
citecolor=blue,%
plainpages=false,%
pdfstartview=FitH}

\numberwithin{equation}{section}

\newcommand\tabcaption{\def\@captype{table}\caption}
\newtheorem{thm}{Theorem}[section]

\newtheorem{rem}[thm]{Remark}

\newtheorem{theorem}{Theorem}[section]

\definecolor{orange}{RGB}{255,127,0}

\def\d{{\rmd}}

\afterpage{\clearpage}

\newcommand{\ml}{\boldsymbol{\Lambda}}
\newcommand{\ms}{\boldsymbol{\Sigma}}
\newcommand{\mr}[1]{\mathbf{R}_{\text{#1}}}
\newcommand{\mrt}[1]{\mathbf{R}_{\text{#1}}(t)}

\newcommand{\vx}{\mathbf{x}}
\newcommand{\vy}{\mathbf{y}}

\newcommand{\vf}{\mathbf{f}}
\newcommand{\vxt}{\mathbf{x}(t)}
\newcommand{\vyt}{\mathbf{y}(t)}

\newcommand{\vm}[1]{\boldsymbol{\mu}_{\text{#1}}}
\newcommand{\vmt}[1]{\boldsymbol{\mu}_{\text{#1}}(t)}
\newcommand{\smooth}[1]{\overleftarrow{#1}}
\newcommand{\dt}{\Delta t}
\newcommand{\vw}{\mathbf{W}}

\newcommand{\ma}{\mathbf{A}}
\newcommand{\mb}{\mathbf{B}}

\newcommand{\cF}{\mathcal{F}}
\newcommand{\pp}{\mathbb{P}}
\newcommand{\rr}{\mathbb{R}}

\newcommand{\ee}[1]{\mathbb{E}\left[#1\right]}
\newcommand{\nf}{\normalfont{f}}
\newcommand{\ns}{\normalfont{s}}
\newcommand{\rmd}{\mathrm{d}}

\newcommand{\ceqref}[2]{\hypersetup{linkcolor=#1}\eqref{#2}\hypersetup{linkcolor=black}}
\newcommand{\cref}[2]{\hypersetup{linkcolor=#1}\ref{#2}\hypersetup{linkcolor=black}}

\makeatletter
\renewcommand*{\@cite@ofmt}{\bfseries\hbox}

\newcommand{\DESCRIPTION@original@item}{}
\let\DESCRIPTION@original@item\item
\newcommand*{\DESCRIPTION@envir}{DESCRIPTION}
\newlength{\DESCRIPTION@totalleftmargin}
\newlength{\DESCRIPTION@linewidth}
\newcommand{\DESCRIPTION@makelabel}[1]{\llap{#1}}%
\newcommand{\DESCRIPTION@item}[1][]{%
  \setlength{\@totalleftmargin}%
       {\DESCRIPTION@totalleftmargin+\widthof{\textbf{#1 }}-\leftmargin}%
  \setlength{\linewidth}
       {\DESCRIPTION@linewidth-\widthof{\textbf{#1 }}+\leftmargin}%
  \par\parshape \@ne \@totalleftmargin \linewidth
  \DESCRIPTION@original@item[\textbf{#1}]%
}

\makeatother

\title{Assimilative Causal Inference}
\author{\textbf{Marios Andreou}}
\author{\textbf{Nan Chen}\footnote{Corresponding author (chennan@math.wisc.edu)}}
\affil{\small Department of Mathematics, University of Wisconsin--Madison, 480 Lincoln Drive, Madison, WI 53706-1325, USA}
\author{\textbf{Erik Bollt}\footnote{In memory of our co-author Erik, who unexpectedly passed away on December 7, 2025}}
\affil{\small Department of Electrical and Computer Engineering, Clarkson University, 8 Clarkson Ave, Potsdam, 13699-5815, NY, USA}
\date{\large A short video about ``Assimilative Causal Inference'' using animations in \href{https://www.manim.community/}{Manim}:\\\vspace{0.25cm} \href{https://youtu.be/lrcPweSC7mQ}{https://youtu.be/lrcPweSC7mQ}}

\begin{document}
\maketitle 

{\fontsize{11.5pt}{13pt}\selectfont \vspace*{-1cm}\tableofcontents }

\normalsize

\vspace*{0.5cm}
\begin{abstract}
    Causal inference is fundamental across scientific disciplines, yet existing methods struggle to capture instantaneous, time-evolving causal relationships in complex, high-dimensional systems. In this paper, assimilative causal inference (ACI) is developed, which is a methodological framework that leverages Bayesian data assimilation to trace causes backward from observed effects. ACI solves the inverse problem rather than quantifying forward influence. It uniquely identifies dynamic causal interactions without requiring observations of candidate causes, accommodates short datasets, and, in principle, can be implemented in high-dimensional settings by employing efficient data assimilation algorithms. Crucially, it provides online tracking of causal roles that may reverse intermittently and facilitates a mathematically rigorous criterion for the causal influence range, revealing how far effects propagate. The effectiveness of ACI is demonstrated by complex dynamical systems showcasing intermittency and extreme events. ACI opens valuable pathways for studying complex systems, where transient causal structures are critical.
\end{abstract}

\begin{bibunit}

\section{Introduction}\label{Sec:Introduction}

Causal inference determines cause-and-effect relationships between variables \cite{pearl2000models, morgan2015counterfactuals, runge2019detecting}. It has found wide applications across different disciplines such as atmospheric and ocean science, economics, and neuroscience \cite{angrist2009mostly, rubin1974estimating, peters2017elements}. In addition to discovering the interactions between variables, causal inference plays a significant role in model identification, policy evaluation, and decision-making \cite{radebach2013disentangling, cafaro2015causation, runge2015identifying, yang2023finding, runge2019inferring}.

Methods for causal inference can be classified into different categories depending on the available resources for use. On the one hand, a natural way to identify the overall causal structures is to exploit multivariate time series, where temporal dependencies are often utilized to infer causal relationships. Methods falling into this category include Granger causality \cite{granger1969investigating}, transfer entropy \cite{schreiber2000measuring, barnett2009granger}, mutual information \cite{cover1999elements, frenzel2007partial}, convergent cross mapping \cite{sugihara2012detecting, monster2017causal}, causation entropy \cite{sun2015causal, sun2014causation}, and mixture causal discovery \cite{varambally2023discovering}. On the other hand, models built upon physics can assist in understanding causal dependence between different variables. Information transfer based on the ensemble forecast of the underlying model has been exploited to indicate certain causal relationships for a short term \cite{liang2005information}. In addition, linear response theory, which infers causal links by analyzing system responses to small perturbations at the equilibrium, has been utilized to reveal the attributions of variations in climate systems \cite{lucarini2024detecting, falasca2023causal, hannart2016causal}. Other recently developed methods include using Koopman operators for causal discovery \cite{rupe2024causal}, causal graphs \cite{kocaoglu2017experimental}, dynamic neural relational inference \cite{graber2020dynamic}, and approaches based on machine learning \cite{brand2023recent, prosperi2020causal, shanmugam2015learning}.

Nature can be regarded as a complex dynamical system where only a single random realization from this system is available. Due to the underlying strong nonlinearity and multiscale features, such a realization is usually intermittent and stochastic. Detecting instantaneous causal relationships as a function of time is crucial for understanding the dynamic interplay among variables, where the roles of causes and effects can shift repeatedly over time and at irregular intervals. In environmental science, extreme events, such as the triggering mechanisms of hurricanes and the abrupt transitions in weather patterns, have significant scientific and social impacts. Understanding the precursor of each extreme event and its subsequent implications requires online tracking of the causal relationship between them. Likewise, instantaneous causal analysis is of broad interest in neuroscience to discover temporary role reversals between brain regions during decision-making or cognitive tasks. However, most purely data-driven causal inference methods exploit data points over a long time series to reveal the average causal directions. In contrast, purely model-based methods, such as computing the information transfer in \cite{liang2005information}, utilize ensemble forecasts to infer the causal relationship when the system begins from a given initial condition and evolves toward statistical equilibrium. Yet, these methods are typically only able to provide time-dependent results for a very short time and become computationally challenging as the dimension of the system increases. Additionally, these model-based methods often struggle to optimally handle cases with a single observed realization.

Given the importance of discovering instantaneous causal relationships, this paper introduces a methodological framework called assimilative causal inference (ACI) to compute time-dependent causal interactions in complex systems. As in many practical applications, only a single realization of a subset of the state variables is available as observations. Meanwhile, an associated model, often turbulent and stochastic, is accessible as a supplement. Fundamentally different from traditional causal inference methods, which treat causality as a forward problem (analyzing how causes propagate into effects), ACI addresses it as an inverse problem. Using Bayesian data assimilation, ACI traces causes backward from observed effects by quantifying whether incorporating information about the effects reduces uncertainty in reconstructing potential causes.

ACI has several desirable advantages. First, it captures the evolving interplay among state variables, as their causal and effect roles can be reversed across time. Thus, ACI provides direct insights beyond time-averaged causal links, making it particularly valuable for studying turbulent systems with intermittency and regime switching. Second, ACI identifies the causal inference range (CIR), revealing how far causal effects propagate at each time instant. The framework includes a mathematically justified criterion to determine the CIR without empirical thresholds. While the importance of the CIR has been highlighted in complex networks \cite{runge2012quantifying, ye2015distinguishing}, ACI offers a unique and general method for its computation in turbulent systems with varying autocorrelation decay rates. Third, ACI scales efficiently to high-dimensional problems by leveraging computational methods developed in Bayesian data assimilation. This is fundamentally different from information-transfer methods limited to low dimensions \cite{liang2005information}. In addition, ACI can accommodate short time series and incomplete datasets, which often appear in geophysics and climate science. Notably, ACI does not require observations of candidate causes, which is a key advantage since potential drivers are usually unknown or unmeasured. Instead, uncertainty reduction in potential causal variables can be determined solely from the observed effects and governing dynamical model via Bayesian assimilation. As a final remark, while data assimilation has been used to assist in detecting the attribution of weather and climate-related events \cite{hannart2016dada, carrassi2017estimating}, these studies primarily focus on estimating the state in response to specific external perturbations, rather than online causal inference or CIR estimation. ACI also fundamentally differs from these methods in its integration of data assimilation into causal inference.

The remainder of this paper is organized as follows. Section \ref{Sec:Method} presents the operational details of the ACI framework, followed by the discussion of CIR in Section \ref{Sec:CIR}. Section \ref{Sec:Conditional_ACI} extends the ACI framework to account for the presence of non-target variables. Numerical results from applying ACI on various nonlinear complex dynamical systems with intermittency and extreme events are presented in Section \ref{Sec:Examples}. Concluding remarks are contained in Section \ref{Sec:Conclusion}. Additional mathematical derivations and supplementary test results are available in the \nameref{Sec:Supplementary_Information}.

\section{The Assimilative Causal inference (ACI) Framework} \label{Sec:Method}

In this section, we provide a high-level overview of the methodology underlying the ACI framework, including its formulation of causal inference as a Bayesian inverse problem. Rigorous mathematical justifications are provided in the \nameref{Sec:Supplementary_Information}.

\subsection{Setup and Notation} \label{Subsec:two_variable_framework}

Consider a time interval $[0, T]$ and a specific time instant $t$, where $0 \leq t \leq T$. Let $\{\mathbf{x}(t)\}_{0\leq t\leq T}$ and $\{\mathbf{y}(t)\}_{0\leq t\leq T}$ be two multivariate stochastic processes defined over this interval. Denote by $\mathbf{x}(t)$ and $\mathbf{y}(t)$ the corresponding multivariate random variables at time $t$. Their evolution in time is defined by the following stochastic dynamical system:
\begin{equation}
    \begin{aligned}
        \frac{\rmd\mathbf{x}}{\rmd t} &= \mathbf{f}^{\vx}(t, \mathbf{x},\mathbf{y})+ \boldsymbol\Sigma^{\vx}_1(t,\vx)\dot{\mathbf{W}}_1+\boldsymbol\Sigma^{\vx}_2(t,\vx)\dot{\mathbf{W}}_2,\\
        \frac{\rmd\mathbf{y}}{\rmd t} &= \mathbf{f}^{\mathbf{y}}(t, \mathbf{x},\mathbf{y}) + \boldsymbol\Sigma^{\mathbf{y}}_1(t,\vx,\mathbf{y})\dot{\mathbf{W}}_1+\boldsymbol\Sigma^{\mathbf{y}}_2(t,\vx,\mathbf{y})\dot{\mathbf{W}}_2,
    \end{aligned} \label{general_CTNDS}
\end{equation}
where $\dot{\mathbf{W}}_1$ and $\dot{\mathbf{W}}_2$ are independent mean-zero Gaussian random vectors with uncorrelated components that possess unit variance. Mathematical details regarding the structure of \eqref{general_CTNDS} are given in the \nameref{Sec:Supplementary_Information}. For notational simplicity, the dependence on $t$ is sometimes omitted when referring to these state variables and no distinction between random variables and their realizations is made. Let $\mathbf{x}(s \leq t)$ represent one realization of the stochastic process $\{\mathbf{x}(t)\}_{0\leq t\leq T}$ as a time series over the interval $[0, t]$. Similarly, $\mathbf{x}(s\leq T)$ denotes a realization of $\{\mathbf{x}(t)\}_{0\leq t\leq T}$ as a time series over the entire interval $[0, T]$. For simplicity, the time series are assumed to be continuously observed. An analogous framework can be developed for discrete-in-time observations. Throughout this paper, the time series is assumed to be free of observational noise.

According to Granger's predictive causality or its nonlinear extension, transfer entropy, $\mathbf{y}$ is the cause of $\mathbf{x}$ at time $t$ if the knowledge of $\mathbf{y}(s \leq t)$ improves the prediction or reduces the uncertainty of $\mathbf{x}(t)$ \cite{granger1969investigating, schreiber2000measuring, barnett2009granger}. Using information theory, this can be expressed as:
\begin{equation}\label{Causal_XY_1}
  \mathcal{S}\big(p(\mathbf{x}(t)|\mathbf{y}(s\leq t))\big) < \mathcal{S}\big(p(\mathbf{x}(t))\big),
\end{equation}
where $\mathcal{S}(\cdot)$ is Shannon's entropy and the history of the target variable in the conditioning, i.e., $\mathbf{x}(s< t)$, is omitted in \eqref{Causal_XY_1} for notational simplicity.  

\subsection{Causal Discovery from a Bayesian Inverse Problem Viewpoint}

The above argument can be interpreted in a reverse way, from the perspective of statistical inference through an inverse problem in uncertainty quantification. If $\mathbf{x}$ is the subsequent effect of $\mathbf{y}(t)$, then knowing future information about $\mathbf{x}$ reduces the uncertainty in inferring the current state of $\mathbf{y}$:
\begin{equation}\label{Causal_XY_2}
  \mathcal{S}\big(p(\mathbf{y}(t)|\mathbf{x}(s> t))\big) < \mathcal{S}\big(p(\mathbf{y}(t))\big).
\end{equation}
Further including the past information of $\mathbf{x}$ on both side yields
\begin{equation}\label{Causal_XY_3}
  \mathcal{S}\big(p(\mathbf{y}(t)|\mathbf{x}(s\leq T))\big) < \mathcal{S}\big(p(\mathbf{y}(t)|\mathbf{x}(s\leq t))\big),
\end{equation}
which does not change the causal relationship in \eqref{Causal_XY_2}.

A notable feature of \eqref{Causal_XY_3} is that it establishes a connection between causal discovery and Bayesian data assimilation. Suppose the underlying turbulent and potentially stochastic model governing $\mathbf{x}$ and $\mathbf{y}$ is known and given by \eqref{general_CTNDS}. Running the model forward provides a statistical estimation of the state of $\mathbf{y}(t)$ (known as `forecast'). However, this state estimation differs when also incorporating the knowledge of the time series $\mathbf{x}(s \leq t)$ or $\mathbf{x}(s \leq T)$ (known as `analysis'), as the specific observation provides additional information to reduce the uncertainty when inferring the state of $\mathbf{y}(t)$. This two-step (forecast-analysis) process is known as Bayesian data assimilation \cite{law2015data, reich2015probabilistic}, where the distribution derived solely from the model statistics is referred to as the prior, while the observed time series is used to compute the likelihood, accounting for the uncertainty in the observational process. The two conditional distributions in \eqref{Causal_XY_3} result from combining the model-based information with the observation-induced likelihood and are referred to as the posteriors. Specifically, the distribution on the left-hand side of \eqref{Causal_XY_3} corresponds to smoothing, while the one on the right-hand side corresponds to filtering \cite{evensen2022data, sarkka2023bayesian}.

\subsection{ACI} \label{Subsec:ACI}

Denote by $p_t^{\text{s}}$ and $p_t^{\text{f}}$ the smoother and filter posterior distributions at time $t$ (i.e., the two distributions on the left- and right-hand sides of \eqref{Causal_XY_3}), respectively. If the uncertainty reduction in $p_t^{\text{s}}$ is more significant than that in $p_t^{\text{f}}$, then it is due to the incorporation of the future information of $\mathbf{x}$, which indicates toward the contribution from $\mathbf{y}(t)$ to the future states of $\mathbf{x}$.

Although the entropy difference has been widely used in many other causal inference methods, the relative entropy is employed hereafter to quantify the uncertainty reduction in $p_t^{\text{s}}$ related to $p_t^{\text{f}}$. In addition to the covariance (and other higher-order moments), which measure the variability of the random variable, the relative entropy also accounts for the difference in the mean state of the two distributions. This difference is crucial for reflecting the additional information gained by incorporating future observations. Furthermore, it is coordinate-free, i.e., it is invariant under general nonlinear changes of the state variables. Notably, relative entropy has been widely used to assess uncertainty reduction in the context of data assimilation \cite{xu2007measuring, majda2018model, fowler2013observation, chen2023stochastic}.

As such, if the relative entropy between the smoother and filter distributions is nonzero, i.e.,
\begin{equation}\label{Causal_XY_4}
\mathcal{P}(p_t^{\text{s}},p_t^{\text{f}}) =\int p_t^{\text{s}}\ln(p_t^{\text{s}}/p_t^{\text{f}}) > 0,
\end{equation}
then $\mathbf{y}$ is identified as the cause of $\mathbf{x}$ at time $t$. By computing \eqref{Causal_XY_4} at different $t$, a time-dependent causal link is established.

Under \eqref{Causal_XY_4}, ACI provides a unique formulation for identifying causal relationships. In traditional approaches, causal links are inferred by running \eqref{general_CTNDS} forward to quantify the informational response on the target variables’ entropy when conditioned on the history of the candidate causes. ACI instead solves an inverse problem, thus tracing causes back from the data of the observed subsequent effects. By interpolating causality from effects to causes instead of extrapolating the latter forward-in-time, ACI offers an inherently more robust framework compared to classical predictive causality. Interpolation is typically more stable because it operates within the range of observed information, whereas extrapolation requires projecting beyond available data and is therefore far more sensitive to model error and uncertainty.

Panel (a) of Figure \ref{Fig:ACI_Framework_Schematic} provides a high-level overview of the ACI framework, while Panel (b) presents a schematic illustration of ACI using the filter and smoother estimates. The \nameref{Sec:Supplementary_Information} contains more technical details.

\begin{figure}[!t]%
\centering
\includegraphics[width=1\textwidth]{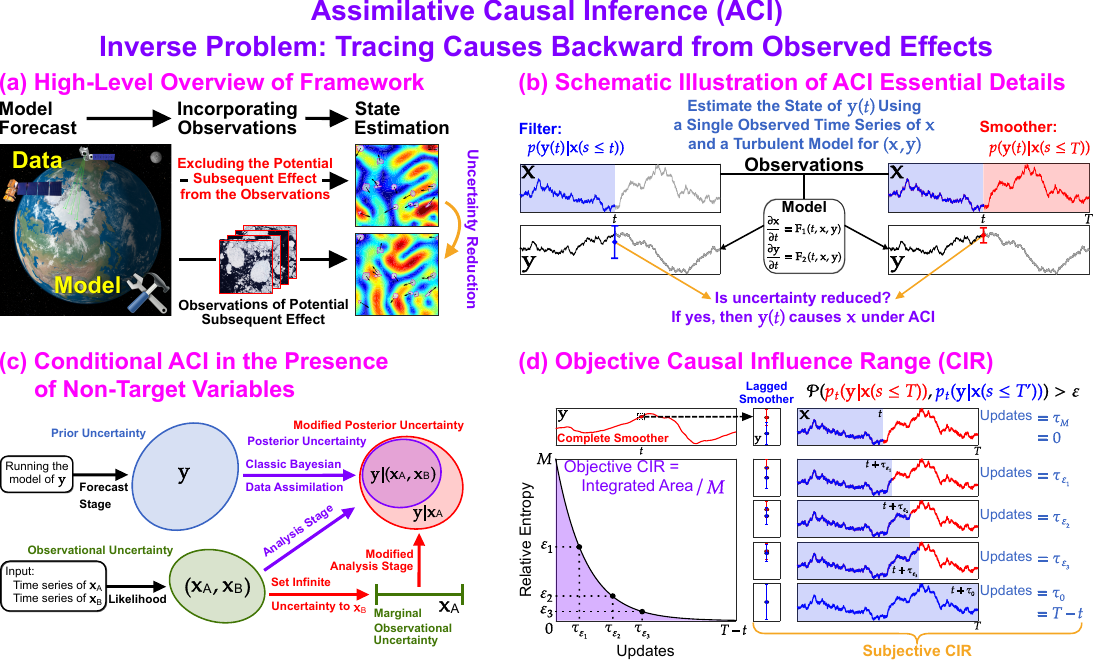}
\caption{The ACI framework. Panel (a): A high-level overview of the proposed method. Panel (b): Schematic illustration of ACI from a more technical viewpoint. Panel (c): ACI in the presence of non-target variables. Panel (d): Objective CIR, as an average of the associated subjective CIRs.}
\label{Fig:ACI_Framework_Schematic}
\end{figure}

\section{Dynamic Discovery of the Causal Influence Range (CIR)}\label{Sec:CIR}

While the uncertainty reduction in the smoother solution related to the filter solution in \eqref{Causal_XY_4} indicates an instantaneous causal link from $\mathbf{y}$ to $\mathbf{x}$ at time $t$, it does not reveal the temporal extent of this causal influence, namely, how many future units of $\mathbf{x}$ are affected by $\mathbf{y}(t)$. Since the memory of turbulent systems decays over time, $\mathbf{y}(t)$ is expected to effectively influence the future states $\mathbf{x}$ only within a limited time period, where this causal influence range can vary significantly as a function of $t$. This raises two key questions: How can we determine the CIR and, more importantly, how can we do so objectively, i.e., without relying on empirical cutoff thresholds? This section outlines the core ideas for addressing these questions. Rigorous mathematical analysis and detailed algorithms are provided in the \nameref{Sec:Supplementary_Information}.

\subsection{Online Smoothing for Dynamic Influence Tracking}

Recall from \eqref{Causal_XY_4} that the relative entropy $\mathcal{P}(p_t^{\text{s}},p_t^{\text{f}})$ between the filter and smoother estimates quantifies the causal influence from $\mathbf{y}(t)$ to $\mathbf{x}$. Here, $p_t^{\text{s}}=p(\mathbf{y}(t)|\mathbf{x}(s\leq T))$ incorporates the complete future information of $\mathbf{x}$. Due to the finite effective influence length, the complete smoother solution $p_t^{\text{s}}$ achieves a substantial uncertainty reduction compared to the lagged posterior distributions $p(\mathbf{y}(t)|\mathbf{x}(s\leq T'))$, which only utilize future information up to $t \leq T' \leq T$, for a short time period after $t$. The divergence between these distributions vanishes as $T'\rightarrow T$. We therefore define the causal influence radius or CIR through the relative entropy criterion: For a predetermined threshold $\varepsilon\geq 0$, the maximum lag $T'_{\varepsilon}$ satisfying
\begin{equation}\label{CIR_Subjective}
\mathcal{P}\big(p(\mathbf{y}(t)|\mathbf{x}(s\leq T)), p(\mathbf{y}(t)|\mathbf{x}(s\leq T'))\big) > \varepsilon,
\end{equation}
yields the innate CIR measure $\uptau_{\varepsilon}(t) = T'_{\varepsilon}-t$. This adaptive time variable $\uptau_{\varepsilon}(t)$ defines the effective causal period $[t, t+\uptau_{\varepsilon}(t)]$ during which $\mathbf{y}(t)$ exerts measurable influence on $\mathbf{x}$. A recently developed online smoother algorithm \cite{andreou2024adaptive} is employed to compute the CIR. This online smoother computes the uncertainty reduction in the estimated state of $\mathbf{y}(t)$ when the future information of $\mathbf{x}$ is sequentially added. See the \nameref{Sec:Supplementary_Information} for more technical details.

While the CIR computed via \eqref{CIR_Subjective} is intuitive, it remains inherently subjective due to its dependence on the threshold parameter $\varepsilon$. In the following, an objective CIR formulation is developed, eliminating the need for a prescribed threshold. This threshold-independent measure is subsequently integrated into the ACI framework.

\subsection{Objective CIR}

For simplicity, assume that the relative entropy between the aforementioned complete and lagged smoother distributions decreases monotonically as more future information of $\mathbf{x}$ is incorporated. This commonly holds in practice, though the theory of the objective CIR (developed rigorously in the \nameref{Sec:Supplementary_Information}) doesn't require this condition. At a given time instant $t$, let $M = \mathcal{P}(p_t^{\text{s}},p_t^{\text{f}})$ denote the maximum value of the relative entropy between the complete and lagged smoother distributions over $t\leq T'\leq T$, which is the difference between the smoother estimate and the filter estimate (where the latter can be viewed as a degenerate smoother using no future information). Given a threshold $\varepsilon\in[0,M]$, \eqref{CIR_Subjective} yields a subjective CIR as $\uptau_\varepsilon(t)=T'_{\varepsilon}-t$. The objective CIR is then naturally defined by
\begin{equation}\label{CIR_Objective}
\uptau(t) := \frac{1}{M}\int_{0}^M \uptau_\varepsilon (t) d\varepsilon.
\end{equation}
Dividing $M$ guarantees the unit of the objective CIR is ``time''. See Panel (d) of Figure \ref{Fig:ACI_Framework_Schematic} for a schematic illustration of the objective CIR and its relationship with the subjective ones. Note that numerically evaluating the integral in \eqref{CIR_Objective} requires repeated smoother computations, leading to a computational complexity that scales quadratically with the number of discretization points. An efficient computational method of \eqref{CIR_Objective} is included in the \nameref{Sec:Supplementary_Information}.

The subjective and objective CIR can be regarded as analogs of the autocorrelation function and decorrelation time, respectively. The autocorrelation function measures the memory of a turbulent system based on the duration of its values that remain above a predetermined threshold. However, determining the memory using the autocorrelation function is subjective and can vary significantly depending on the threshold value. In contrast, the decorrelation time, which integrates the autocorrelation function, provides an objective way to define the memory length and is free from any predetermined cutoff threshold.

\section{Conditional ACI in the Presence of Additional Non-Target Variables}\label{Sec:Conditional_ACI}

In the more general context of a complex system, state variables can be grouped into three disjoint subsets: $\mathbf{x}_{\text{A}}$, $\mathbf{x}_{\text{B}}$, and $\mathbf{y}$, with $\mathbf{\vx}=(\mathbf{x}_{\text{A}}, \mathbf{x}_{\text{B}})^\mathtt{T}$ under the previous notation. The goal is to determine whether $\mathbf{y}$ is a cause of $\mathbf{x}_{\text{A}}$ in the presence of the remaining variables $\mathbf{x}_{\text{B}}$, which are called the non-target variables, including the cases of confounders or mediators. Since $\mathbf{x}_{\text{B}}$ interacts with both $\mathbf{x}_{\text{A}}$ and $\mathbf{y}$, appropriately accounting for the role of $\mathbf{x}_{\text{B}}$ while inferring the causal link from $\mathbf{y}$ to $\mathbf{x}_{\text{A}}$ becomes essential. Hereafter, let us assume that a realization for each of $\mathbf{x}_{\text{A}}$ and $\mathbf{x}_{\text{B}}$, namely $\mathbf{x}_{\text{A}}(s\leq T)$ and $\mathbf{x}_{\text{B}}(s\leq T)$, is available. Therefore, the candidate cause $\mathbf{y}$ is the only variable that does not require to have observations in this framework.  Following the argument in Section \ref{Sec:Method}, the causal relationship from $\mathbf{y}(t)$ to $(\mathbf{x}_{\text{A}}, \mathbf{x}_{\text{B}})$ is identified based on the uncertainty reduction in the smoother distribution $p_t^{\text{s}}$ related to the filter one $p_t^{\text{f}}$, where
\begin{equation*}
    \begin{split}
      p_t^{\text{f}} & = p(\mathbf{y}(t)|\mathbf{x}_{\text{A}}(s\leq t), \mathbf{x}_{\text{B}}(s\leq t)),\\
      p_t^{\text{s}} & = p(\mathbf{y}(t)|\mathbf{x}_{\text{A}}(s\leq T), \mathbf{x}_{\text{B}}(s\leq T)).
    \end{split}
\end{equation*}
The main focus here is to appropriately determine and resolve the presence of $\mathbf{x}_{\text{B}}$ from $p_t^{\text{f}}$ and $p_t^{\text{s}}$ as to infer a conditional causal relationship from $\mathbf{y}(t)$ to $\mathbf{x}_{\text{A}}$ beyond the contributions of $\mathbf{x}_{\text{B}}$.

Recall from Section \ref{Sec:Method} that ACI employs Bayesian data assimilation to estimate the state of $\mathbf{y}(t)$. ACI exploits observations of $\mathbf{x}$ to reduce the uncertainty in $\mathbf{y}(t)$, where the influence of $\mathbf{x}$ on the estimation of $\mathbf{y}(t)$ depends on its own uncertainty: smaller observational uncertainty leads to a stronger impact \cite{rozovsky2012stochastic}. When extending this framework to $\mathbf{\vx}=(\mathbf{x}_{\text{A}}, \mathbf{x}_{\text{B}})^\mathtt{T}$ for conditional ACI, the observed time series of $\mathbf{x}_{\text{B}}$ is still used in the forecast step to compute the uncertainties in $\mathbf{x}_{\text{A}}$ and $\mathbf{y}$. However, during the analysis step, the uncertainty in $\mathbf{x}_{\text{B}}$'s dynamics is deliberately excluded by assigning infinite uncertainty to its marginal likelihood before posterior computation. This can be rigorously formulated through a limit operation on the analysis component of the Bayesian data assimilation update equations, which formally leads to the following filter- and smoother-based distributions:
\begin{equation*}
        p^{\text{f}|\vx_{\text{B}}}_t := \lim_{\mathrm{Var}(\vx_{\text{B}}(t))\to+\infty} p_t^{\text{f}}, \quad p^{\text{s}|\vx_{\text{B}}}_t := \lim_{\mathrm{Var}(\vx_{\text{B}}(t))\to+\infty} p_t^{\text{s}}.
\end{equation*}
Similar to \eqref{Causal_XY_4}, if the relative entropy of $p^{\text{s}|\vx_{\text{B}}}_t$ from $p^{\text{f}|\vx_{\text{B}}}_t$ is positive, $\mathcal{P}(p^{\text{s}|\vx_{\text{B}}}_t,p^{\text{f}|\vx_{\text{B}}}_t)>0$, then $\mathbf{y}$ is identified as the cause of $\vx_{\text{A}}$ conditional to $\vx_{\text{B}}$ at time $t$. This means future data of $\vx_{\text{A}}$ incur additional information gain about $\mathbf{y}(t)$ related to its history, beyond the observational contributions of $\vx_{\text{B}}$. By computing this metric at each $t$, a time-dependent conditional causal relationship is established.

Inflating the uncertainty in $\mathbf{x}_{\text{B}}$ is straightforward using Bayesian inference for state estimation. This approach ensures that $\mathbf{x}_{\text{B}}$ does not affect the uncertainty reduction in the state estimation of $\mathbf{y}(t)$ when treated as an observational process. Consequently, any uncertainty reduction in $\mathbf{y}(t)$ results solely from $\mathbf{x}_{\text{A}}$ under this proposed pipeline. See Panel (c) of Figure \ref{Fig:ACI_Framework_Schematic} for an illustration. For dynamical systems with explicit governing equations (see \eqref{general_CTNDS}), a shortcut to exclude the influence in the uncertainty of $\mathbf{x}_{\text{B}}$ is to treat $\mathbf{x}_{\text{B}}$ as a prescribed forcing term in the reduced system defined by $(\mathbf{x}_{\text{A}},\mathbf{y})$ during Bayesian inference, with values defined by its observed time series. Nevertheless, this framework does not necessarily require modifications to the underlying dynamical system, making it compatible with any given model, potentially including operational models in atmospheric and ocean science, as the model structure and integrity remains intact.

It is important to emphasize that while treating $\mathbf{x}_{\text{B}}$ as a conditioning variable may appear analogous to the role of non-target variables in transfer entropy and other traditional causal inference methods, the way of handling the uncertainty in ACI is fundamentally different. Conventional methods treat $\mathbf{x}_{\text{B}}$ as a fixed component. In contrast, ACI explicitly accounts for $\mathbf{x}_{\text{B}}$'s dynamical influence while systematically prohibiting its uncertainty from affecting state estimation. This unique treatment allows the inferred causal relationships to reflect true dynamical interactions by isolating the contribution of uncertainty reduction from the non-target variables. As a final remark, if Bayesian data assimilation is used to compute $p(\mathbf{y}(t),\mathbf{x}_{\text{B}}(t)|\mathbf{x}_{\text{A}}(s\leq T))$ and then marginalize over $\mathbf{x}_{\text{B}}(t)$ to obtain $p(\mathbf{y}(t)|\mathbf{x}_{\text{A}}(s\leq T))$, this manipulation may lead to incorrect conclusions. This is because the correlation between $\mathbf{y}(t)$ and $\mathbf{x}_{\text{B}}(t)$ unavoidably alters the uncertainty in $\mathbf{y}(t)$ during Bayesian data assimilation, potentially introducing a spurious causal relationship between $\mathbf{y}$ and $\mathbf{x}_{\text{A}}$. This is true even if we condition over $\mathbf{x}_{\text{B}}(t)$ to obtain $p(\mathbf{y}(t)|\mathbf{x}_{\text{A}}(s\leq T),\vx_{\text{B}}(t))$. For instance, for the causal chain $\mathbf{y} \to \mathbf{x}_{\text{B}} \to \mathbf{x}_{\text{A}}$, such a method might falsely imply a direct causal link from $\mathbf{y}(t)$ to $\mathbf{x}_{\text{A}}$.

The details of the conditional ACI implementation, along with the verification of (conditional) nil causality principles under ACI, are provided in the \nameref{Sec:Supplementary_Information}.

\section{Applications to Nonlinear Systems with Intermittency and Extreme Events} \label{Sec:Examples}

We apply the ACI framework to identify causal relationships across the state space and evaluate the corresponding CIRs in complex nonlinear dynamical systems exhibiting intermittency and extreme events. The dynamical properties of these systems and the mechanisms that trigger crucial observed phenomena are well understood. This knowledge is exploited as guidance to validate the results when applying ACI. In addition, a perfect-model setup is used in the case studies carried out, meaning the system generates the synthetic observational data and the forecast model in ACI at the same. This facilitates the recovery of causal relationships in the absence of model error. Because the underlying system is chaotic and subject to stochastic forcing, a perfect model means that the model used in ACI shares the same governing equations and long-term statistical behavior as the system that generates the true signal. Additional results, including further numerical experiments, case studies, and one crucial application using sparse real-world observations, are provided in the \nameref{Sec:Supplementary_Information}.

The operational outline of the ACI metric used to infer instantaneous causal relationships in the following experiments, including the details of how we determine the associated CIRs and account for non-target variables, is described in Sections \ref{Sec:Method}--\ref{Sec:Conditional_ACI}. The developed framework is also summarized via the schematic diagram provided in Figure \ref{Fig:ACI_Framework_Schematic}. Additional mathematical results and analyses, as well as computational implementation details, can be found in the \nameref{Sec:Supplementary_Information}.

\subsection{A Nonlinear Dyad Model with Extreme Events}

Let us start with a two-dimensional model, which nevertheless has strong nonlinear features with observed extreme events. The model reads as:
\begin{subequations}\label{Dyad_model}
\begin{align}
\frac{\d x}{\d t} &= -d_x x + \gamma xy + f_x + \sigma_x\dot{W}_x\label{Dyad_model_x}\\
\frac{\d y}{\d t} &= -d_y y - \gamma x^2 + f_y + \sigma_y \dot{W}_y,\label{Dyad_model_y}
\end{align}
\end{subequations}
where $\dot{W}_x$ and $\dot{W}_y$ are independent standardized normal random variables. Within the proposed ACI framework, only $x$ is assumed to be observed over time. This is a reduced-order conceptual model for atmospheric variability. It has been used to analyze the effects of various coarse-grained procedures on processes exhibiting intermittency, large-scale bifurcations, and microscale phase transitions. It is defined by an energy-conserving condition on its quadratic nonlinearities \cite{majda2012physics}. The following parameter values are used for this model:
\begin{equation}
\begin{gathered}
    d_x=0.5, \quad \gamma=2, \quad f_x=0.5, \quad \sigma_x=0.5, \\ d_y=0.5, \quad f_y=1, \quad \sigma_y=1.
\end{gathered}
\end{equation}
Figure \ref{Fig:dyad_interaction_fig} illustrates the ACI and CIRs from $y$ to $x$. Panel (a) shows that extreme events in $x$ occur intermittently. When the combined coefficient $-d_x+\gamma y$ in \eqref{Dyad_model_x} becomes positive, the dynamics of $x$ exhibit anti-damping, leading to amplitude growth and extreme events. Conversely, when the coefficient is negative, $x$ behaves as a damped system, with fluctuations primarily driven by random noise. Panel (b) displays the ACI, revealing that phases with significant ACI values largely coincide with the onset and peak phases of extreme events in $x$. This aligns with intuition, as $y$ acts as an anti-damping source and the primary driver of these events. Once $x$ peaks and begins to decay, the ACI value drops sharply, reflecting the strong negative feedback from $x$ to $y$, which then dampens $x$.
The whisker plot in Panel (a) highlights the objective CIR, indicating a sustained influence from $y$ to $x$ during the triggering phase of extreme events but minimal influence during their demise. Notably, the long-range influence from the CIR reveals that extreme events develop gradually, with triggering conditions established well in advance. From a data assimilation perspective, when the observable signal $x$ starts to strengthen, it enhances the ability of the smoother to estimate $y$ as the signal in the future contains useful information beyond noise. As a result, a longer CIR is obtained. In contrast, once $x$ becomes sufficiently strong, its high signal-to-noise ratio ensures the filter captures all relevant triggering dynamics, leaving little room for the smoother to improve estimates. Hence, the outset of the ACI metric's decline and shorter CIR. The \nameref{Sec:Supplementary_Information} contains the filter and smoother distributions and the subjective CIR with different threshold values $\varepsilon$.

\begin{figure}[!ht]%
\centering
\includegraphics[width=1\textwidth]{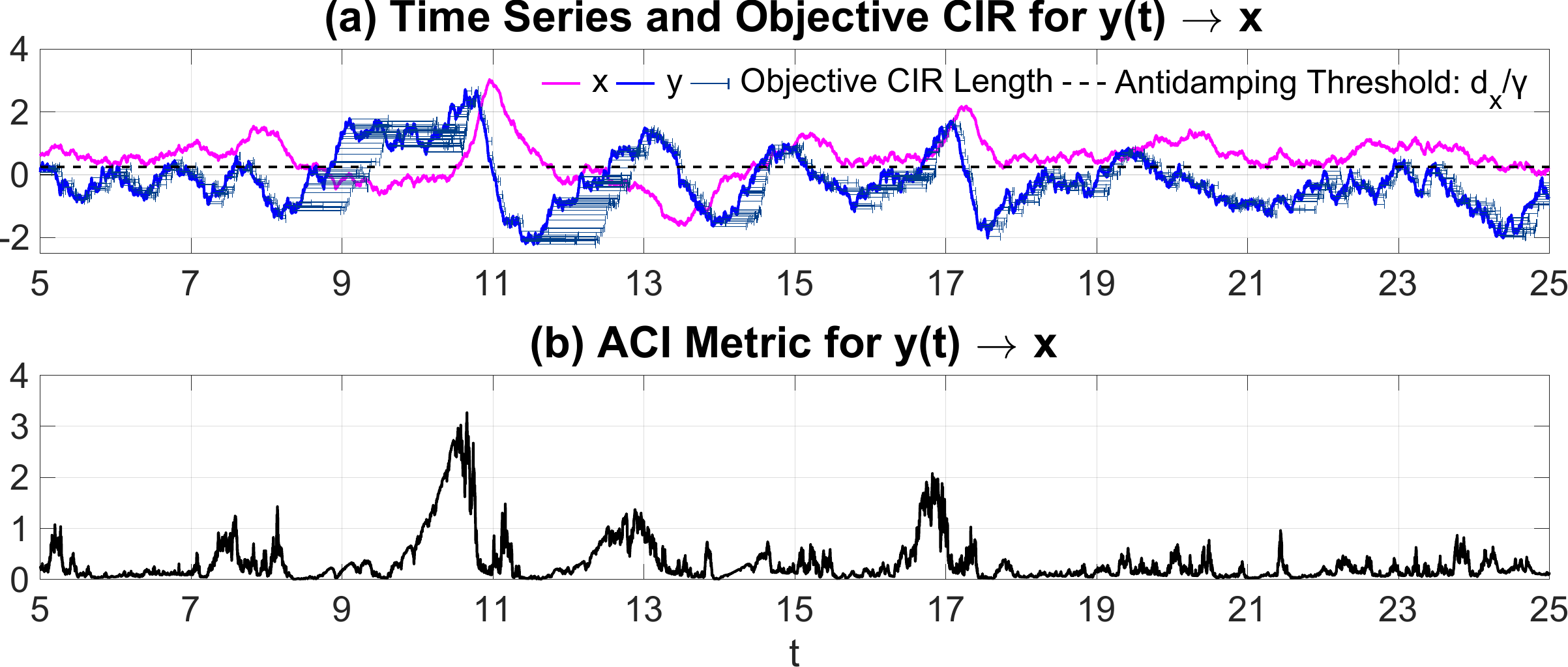}
\caption{ACI values and CIRs from $y$ to $x$ at each time instant $t$ for the nonlinear dyad model \eqref{Dyad_model}. Panel (a): Time series of $x$ (magenta) and $y$ (blue), the objective CIR represented by a whisker plot emanating forward in time from each $y(t)$ value, and the anti-damping threshold line at $d_x/\gamma$. Panel (b): ACI from $y$ to $x$ as a function of time.}\label{Fig:dyad_interaction_fig}
\end{figure}

\subsection{A Stochastic Model Capturing the El Ni\~no-Southern Oscillation (ENSO) Diversity}

El Ni\~no--Southern Oscillation (ENSO) is a dominant climate phenomenon characterized by quasi-regular periodic behaviors in sea surface temperatures (SSTs) and in the atmospheric circulation across the equatorial Pacific Ocean \cite{clarke2008introduction, sarachik2010nino}. ENSO exhibits remarkable diversity in its spatial patterns, temporal evolution, and impacts, which can be broadly categorized into two main types: Eastern Pacific (EP) and Central Pacific (CP) El Ni\~nos, where an anomalous warming center occurs in the eastern and central Pacific, respectively \cite{capotondi2015understanding}. The opposite phases with anomalous cooling SSTs are called La Ni\~na. Understanding the causes of different ENSO events is crucial for improving climate predictions and mitigating socio-economic consequences.

Although few models can accurately capture ENSO diversity, a recently developed stochastic conceptual model successfully reproduces its diverse behaviors and non-Gaussian statistics \cite{chen2022multiscale}. This model has been highlighted in a recent review \cite{vialard2025nino}, making it a suitable testbed for studying the ENSO diversity. Mathematical details are provided in the \nameref{Sec:Supplementary_Information}. The model consists of six state variables: ocean zonal current in the CP ($u$), western Pacific (WP) thermocline depth ($h_W$), CP SST ($T_C$), EP SST ($T_E$), atmospheric winds ($\tau$, intraseasonal), and background Walker circulation ($I$, decadal). The first four variables ($u$, $h_W$, $T_C$, and $T_E$) operate on interannual timescales. As a nonlinear system with state-dependent noise, the model generates extreme events and intermittency. The two SST variables allow reconstruction of spatiotemporal SST patterns across the equatorial Pacific, providing an intuitive way to identify different ENSO event types.

The variables $T_C$, $h_W$, and $\tau$ directly influence $T_E$. Figure \ref{Fig:ENSO} displays their conditional ACI and CIR values to $T_E$ (non-target variables have been resolved via the framework proposed in Section \ref{Sec:Conditional_ACI}). For EP El Ni\~no events (positive $T_E$ anomalies, shown in red), $T_C$ exhibits the strongest ACI value, slightly preceding the $T_E$ peak. This timing aligns with physical understanding, as SSTs in these regions are strongly coupled: during El Ni\~no, warm water propagates from CP to EP, producing the observed ACI lead.
The $\tau$ to $T_E$ ACI value appears noisier due to $\tau$'s intraseasonal variability, yet its evolution confirms $\tau$'s robust impact on $T_E$. Winds drive warm water propagation and exert near-instantaneous SST effects. In contrast, while $h_W$ significantly contributes to $T_E$, its ACI amplitude is weaker than $T_C$ or $\tau$.
The discharge--recharge theory \cite{jin1997equatorial} posits $h_W$--$T_E$ oscillator dynamics, but in models with refined CP variables, $h_W$'s influence on $T_E$ becomes indirect: $h_W$ first affects $T_C$, which then propagates signals to $T_E$ (WP$\rightarrow$CP$\rightarrow$EP). Consistent with this paradigm, $h_W$'s ACI value to $T_E$ peaks months before EP El Ni\~no maxima. CIRs further corroborate these physical mechanisms: $T_C$ shows the longest influence, $h_W$'s more indirect role yields intermediate CIRs, and $\tau$’s intraseasonal nature produces the shortest impacts. The \nameref{Sec:Supplementary_Information} provides additional ACI and CIR analyses for all variables across different ENSO event types. In addition, in the \nameref{Sec:Supplementary_Information}, a case study using real-world ENSO observations is examined, which does not assume a perfect model setup. The conclusions derived about the causal relationships using real-world observational data remain qualitatively similar to those with model-generated data. The results provide preliminary evidence that the ACI framework remains reasonably robust under realistic conditions, including observational noise, moderate model error, and numerical discretization effects.

\begin{figure}[!ht]%
\centering
\includegraphics[width=1\textwidth]{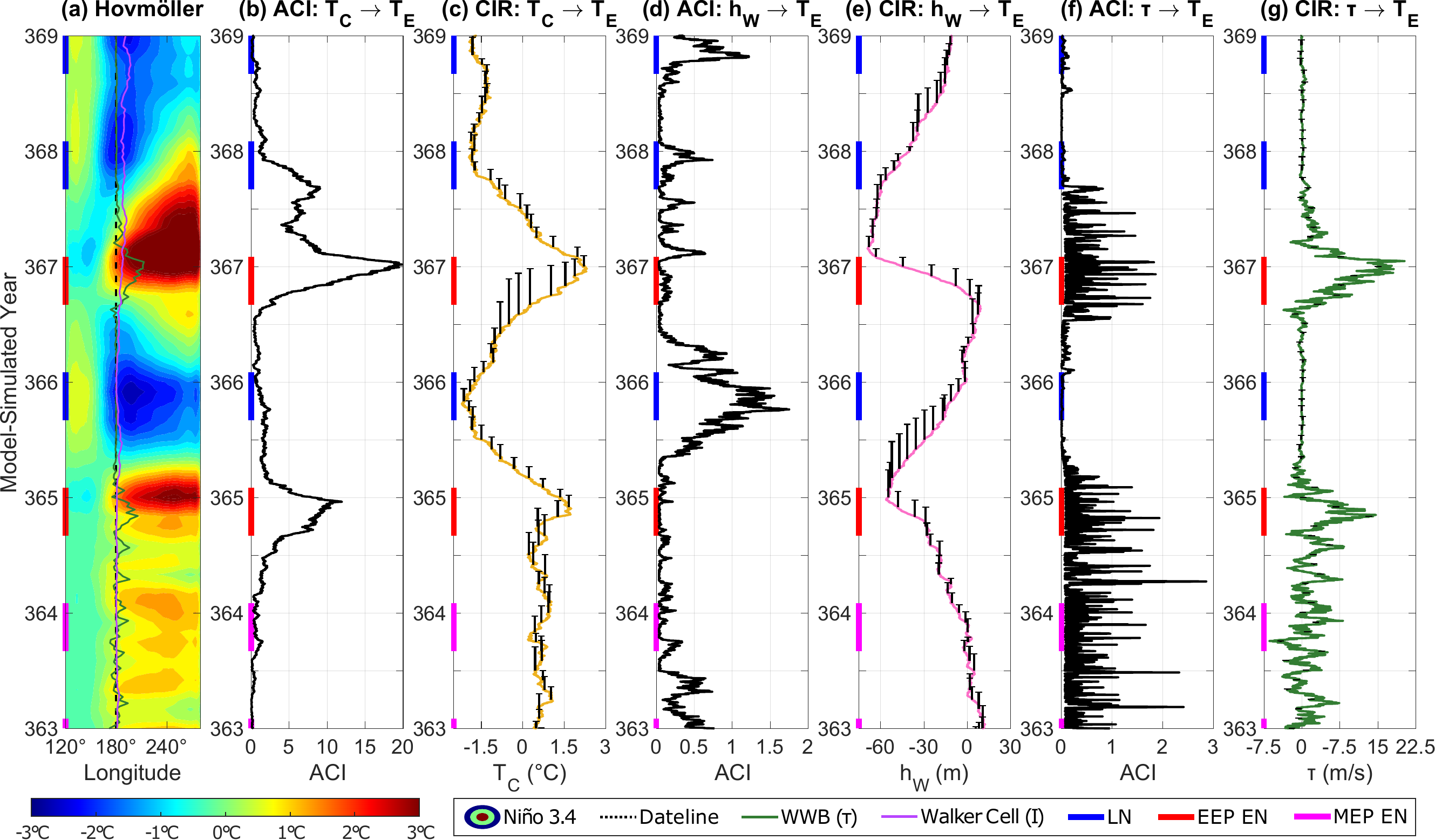}
\caption{Conditional ACI values and CIRs of the stochastic conceptual model for the ENSO diversity with $T_E$ as the effect variable over a six-year model-simulated period (where EEP EN and MEP EN stand for extreme and moderate EP El Ni\~nos, respectively, and LN for La Ni\~na). Panel (a): Hovm\"{o}ller diagram of the SST anomalies in the equatorial Pacific reconstructed from $T_C$ and $T_E$ via a spatiotemporal multivariate regression technique \cite{chen2022multiscale}. Panel (b): ACI values from $T_C$ to $T_E$ as a function of time. Panel (c): Time series of the observable $T_C$ and the objective CIR represented by a whisker plot emanating forward in time from the value of $T_C$ at each month over this six-year period. Panels (d) and (e): Same as Panels (b) and (c) but for the causal link from $h_W$ to $T_E$. Panels (f) and (g): Same as Panels (b) and (c) but for the causal link from $\tau$ to $T_E$.}\label{Fig:ENSO}
\end{figure}

\section{Conclusion and Discussions} \label{Sec:Conclusion}

In this paper, the assimilative causal inference (ACI) framework is developed for the detection of instantaneous causal relationships and of their associated causal influence range (CIR) in complex dynamical systems. ACI leverages Bayesian data assimilation to trace causes backward from observed effects. It uniquely identifies dynamic causal interactions without requiring observations from candidate causes, accommodates short datasets, and, in principle, can be scaled efficiently to high dimensions using efficient Bayesian data assimilation algorithms. Crucially, it provides online tracking of causal roles, which may reverse intermittently, and facilitates a mathematically rigorous criterion for the causal influence range, revealing how far causal impacts propagate temporally. Its ability to dynamically detect causal influence ranges and handle high-dimensional systems by employing computationally efficient Bayesian techniques highlights its potential for applications in climate science, neuroscience, and other fields. Numerical tests on nonlinear systems with extreme events and regime transitions demonstrate the effectiveness and robustness of the framework.

It is worth highlighting how ACI differs from other state-of-the-art methods in terms of expressivity, computational complexity, and underlying assumptions. Different from most methods that identify cause-and-effect relationships by running the system forward and quantifying informational responses in the target variables’ entropy conditioned on past candidate causes, which is a predictive setting typical of information-flow-based methods, ACI instead formulates causality as an inverse problem in uncertainty quantification. Within the Bayesian data assimilation framework, candidate causal variables are inferred backward from observations of their potential effects. Thus, while many existing methods extrapolate causes forward in time, ACI traces causes back from observed effects, effectively interpolating rather than extrapolating causality. This distinction makes ACI more expressive in capturing instantaneous causal structures and, in practice, often more robust.
From a computational standpoint, ACI can leverage well-developed Bayesian data assimilation tools that scale effectively to high-dimensional systems \cite{evensen2003ensemble, law2015data}. In contrast, information transfer and other model-based methods often require computing the exact density functions and their derivatives, which is computationally demanding and subject to the curse of dimensionality. Finally, ACI requires only a single time series but assumes access to a dynamical model, which enables it to identify instantaneous causal relationships. In contrast, most time-series-based methods operate under model-free assumptions and therefore focus on aggregated or time-averaged causal links rather than on causality at specific moments. Because ACI incorporates model dynamics, it does not require long time series to detect causes and effects. Instead, it can infer instantaneous causal links from a single data point combined with model information. In contrast, most purely data-driven methods rely on long time series and can only reveal causal relationships in a time-averaged sense.

Future work will focus on several important directions. First, because model error is ubiquitous in practice, it is essential to study its impact on causal inference. Understanding how model error affects inferred causal relationships can create opportunities not only to detect and correct model error, but also to leverage this information to develop parsimonious models for complex dynamical systems. Similarly, observational data may contain additional noise that can affect causal inference. The ACI framework provides a natural way to analyze how observational noise contaminates causal detection. In such a context, data assimilation can serve as a pre-filtering step to reduce the influence of observational noise, which allows for a systematic examination of its role in shaping causal inference outcomes.
Second, by building on the computationally efficient methods developed in Bayesian data assimilation, it is essential to apply ACI to higher-dimensional systems. Ensemble data assimilation is a well-established technique that routinely uses localization, covariance inflation, and other adjustments to overcome the curse of dimensionality \cite{anderson2012localization, d2014scalable, anderson2007scalable}. As a result, the ensemble size (typically on the order of $O(100)$) is often chosen independently of the system's dimension, which contributes to the scalability of Bayesian data assimilation and, consequently, of ACI. However, the sampling errors inherent in ensemble methods may introduce additional biases. Thus, it is necessary to examine whether such errors affect ACI. In addition, various coarse-grained approaches, such as approximating the full distribution with a Gaussian, are widely used in practice. Although these approximations introduce errors in computing the PDF, they are robust and highly accurate when computing diagnostic indicators, such as the relative entropy, that are used for model identification and sensitivity analysis. Thus, they may substantially accelerate ACI in practice.
Third, building upon the current forward-in-time CIR analysis, the framework allows for the development of a crucial complement that focuses on dying the backward-in-time event attribution to past causes. The forward and backward CIR tools would support distinguishing different types of tipping points in Earth systems, identifying the causes of regime switching, and detecting early warning signals.
Finally, continuous-in-time observations are adopted in this work for simplicity of presentation. Nevertheless, an analogous framework can be naturally extended to discrete-in-time observations using the corresponding Bayesian data assimilation algorithms. An interesting scientific question in this context is to understand the uncertainty in causal detection arising from the discrete nature of the data, particularly at time intervals between successive observations. This line of inquiry could also highlight an important distinction between ACI and purely data-driven methods. In ACI, the dynamical model provides a natural way to interpolate causal relationships at time instants when observations are unavailable. Therefore, it may help bridge observational gaps and maintain temporal continuity in causal inference.

\section*{Funding}
The research of N.C. is funded by Office of Naval Research N00014-24-1-2244 and Army Research Office W911NF-23-1-0118. M.A. is partially supported as a research assistant under the second grant. The research of E.B is supported by the ONR, ARO, DARPA RSDN, and the NIH and NSF under CRCNS.

\section*{Data Availability}
In Section 4.3.3 of the \nameref{Sec:Supplementary_Information} (ACI on Real-World ENSO Data), daily and monthly observations for the state variables of the ENSO model were curated from the following reanalysis datasets: \href{https://psl.noaa.gov/data/gridded/data.godas.html}{NCEP Global Ocean Data Assimilation System (GODAS)} and \href{https://psl.noaa.gov/data/gridded/data.ncep.reanalysis.html}{NCEP-NCAR Reanalysis 1 Project}. All other data in the simulations are model-generated.

\subsection*{Code Availability}
The MATLAB code used in the analyses and to generate the figures in this work can be found in the following GitHub repository: \url{https://github.com/marandmath/ACI_code}. The codebase also includes implementations for additional experiments and real-world observations that can be used for ACI and CIR analysis in the ENSO case study.

\renewcommand{\refname}{Main Text References}

\putbib

\end{bibunit}

\clearpage

\begin{bibunit}

\renewcommand\appendixname{SI.\hspace*{-0.15cm}}
\renewcommand\appendixpagename{SI.\hspace*{-0.15cm}}
\renewcommand{\figurename}{Supplementary Figure}
\setcounter{figure}{0}

\section*{Supplementary Information} 
\label{Sec:Supplementary_Information}
\addcontentsline{toc}{section}{Supplementary Information}

\begin{appendices}
    
\section{Mathematical Foundations and Theoretical Analysis of Assimilative Causal Inference}\label{Sec:Math_Foundation}

\subsection{Mathematical Framework and Notational Conventions}

In this document, \textbf{boldface} variables denote multidimensional quantities. Specifically, \textbf{l}owercase boldface variables denote column vectors, while \textbf{U}ppercase ones denote matrices. The only exception is $\vw$ (with some subscript), which denotes a vector-valued Wiener process due to literary tradition. Furthermore, for simplicity, we follow the notational convention from physics and do not distinguish between random variables and their realizations.

Let $\mathscr{B}=(\Omega, \cF, \mathbb{F}, \mathbb{P})$ be an augmented probability space for filtering over the time interval $[0,T]$, $T\in(0,+\infty)$. Denote by $\big(\mathbf{x}(t,\omega),\mathbf{y}(t,\omega)\big)\in\mathbb{R}^{k+l}$, for $t\in[0,T]$ and $\omega\in\Omega$, a $(k+l)$-dimensional partially observable stochastic process on $\mathscr{B}$, where $\vx$ is the $k$-dimensional observable component while $\vy$ is the $l$-dimensional unobservable part. Without loss of generality, we consider real-valued processes, otherwise we work with the real-valued joint vector formed by the real and imaginary parts of $(\vx,\vy)$ and for notational simplicity we henceforth drop the sample space ($\Omega$) dependence, but it is always implied. We assume $(\vx,\vy)$ is adapted to the filtration $\mathbb{F}$ and that $\vx(s\leq t)$ represents a realization of $\vx$ over $[0,t]$, i.e., a time series for a fixed $\omega\in\Omega$. For convenience, the time series of $\vx$ is assumed to be continuously observed, with the development of an analogous framework for discrete-in-time observations being possible. Finally, for explicitness, we hereafter write $\big(\cdot\big|\vx(s\leq t)\big)$ to indicate the fact that we are conditioning on the $\sigma$-algebra generated by $\{\vx(s)\}_{s\leq t}$.

\subsection{Bayesian Data Assimilation with Continuous-Time Observations}
The assimilative causal inference (ACI) framework builds fundamentally upon Bayesian data assimilation methods. Here we outline their essential mathematical formulation, particularly as applied to probabilistic state estimation in complex turbulent dynamical systems.

\subsubsection{Underlying Dynamics}
We assume that the evolution of the partially observed process $(\vx,\vy)$ is governed by the following stochastic system of coupled It\^o diffusions over $t\in[0,T]$ \cite{rozovsky2012stochastic}:
\begin{subequations} \label{eq:general_CTNDS}
    \begin{align}
    \rmd\mathbf{x}(t) &= \mathbf{f}^{\vx}(t, \mathbf{x},\mathbf{y})\rmd t+ \boldsymbol\Sigma^{\vx}_1(t,\vx,\vy)\rmd\mathbf{W}_1(t)+\boldsymbol\Sigma^{\vx}_2(t,\vx,\vy)\rmd\mathbf{W}_2(t), \label{eq:general_CTNDS1}\\
    \rmd\mathbf{y}(t) &= \mathbf{f}^{\vy}(t, \mathbf{x},\mathbf{y})\rmd t + \boldsymbol\Sigma^{\vy}_1(t,\vx,\vy)\rmd\mathbf{W}_1(t)+\boldsymbol\Sigma^{\vy}_2(t,\vx,\vy)\rmd\mathbf{W}_2(t), \label{eq:general_CTNDS2}
    \end{align}
\end{subequations}
where $\mathbf{W}_1\in\mathbb{R}^{d_1}$ and $\mathbf{W}_2\in\mathbb{R}^{d_2}$ are two real-valued independent Wiener processes that are also independent from the distribution of $(\vx(0),\vy(0))$. We note that all subsequent random ordinary and stochastic differential equations in this document are written in differential form. In \eqref{eq:general_CTNDS}, we assume the following \cite{rozovsky2012stochastic, ikeda2014stochastic, oksendal2003stochastic}.
    (I) Almost every sample path of $\vx$ and $\vy$ is continuous in $[0,T]$.
    (II) The functions appearing in \eqref{eq:general_CTNDS} are predictable for each $\vx\in\mathbb{R}^k$ and $\vy\in\mathbb{R}^l$.
    (III) Sufficient regularity conditions are enforced such that \eqref{eq:general_CTNDS} has a unique solution.
    (IV) The diffusion coefficients of $\vx$ do not depend on $\vy$: $\boldsymbol\Sigma^{\vx}_m(t,\vx,\vy)=\boldsymbol\Sigma^{\vx}_m(t,\vx)$, for $m=1,2$.
    (V) The sum of the row-based Gramians of the observational noise feedback matrices, $(\ms^\vx\circ\ms^\vx)(t,\vx):=\ms^\vx_1(t,\vx)\ms^\vx_1(t,\vx)^\mathtt{T} +\ms^\vx_2(t,\vx)\ms^\vx_2(t,\vx)^\mathtt{T}$,
    is invertible for each $t\in[0,T]$ and $\vx\in\rr^k$.

{(IV)}--{(V)} function as identifiability conditions, since they ensure that the conditional distribution of $\vy$ given the data of $\vx$ contains all available information about $\vy$ in \eqref{eq:general_CTNDS} \cite{rozovsky2012stochastic}.

\subsubsection{Bayesian Data Assimilation}
Bayesian data assimilation provides a probabilistic state estimation for the unobserved variables $\vyt$ conditioned on the observations of $\vx$. The process consists of two sequential steps:
\begin{itemize}
    \item \emph{Forecast step:} Forward integration of \eqref{eq:general_CTNDS} generates a model-based prior distribution for the unobserved state $\vyt$.
    \item \emph{Analysis step:} Observational data of $\vx$ is assimilated into the forecast through Bayesian updating, reducing the uncertainty and bias in the estimation of $\vyt$.
\end{itemize}
In Bayesian data assimilation for state estimation, the prior distribution (obtained from model forecasts) is combined with the likelihood of observations through Bayes' theorem, yielding an updated posterior distribution.

\subsubsection{Filter and Smoother}
Bayesian data assimilation is naturally divided into two approaches based on the assimilated observational time window during analysis: filtering incorporates only current and past observations ($\vx(s\leq t)$), while smoothing additionally utilizes future data ($\vx(s\leq T)$), typically producing more accurate estimates through the broader temporal context.

Under mild spatial regularity assumptions on the parameters in \eqref{eq:general_CTNDS}, both the filter and smoother distributions are absolutely continuous with respect to the Lebesgue measure on \cite{rozovsky2012stochastic}. This guarantees the existence of their probability density functions (PDFs) for all $t\in[0,T]$, which we define as:
\begin{equation}
    \begin{aligned}
        \text{\underline{\textbf{Filter}:}} \quad &p_t^{\text{f}}(\vy|\vx):=p\big(\vyt|\vx(s\leq t)\big),\\
        \text{\underline{\textbf{Smoother}:}} \quad &p_t^{\text{s}}(\vy|\vx):=p\big(\vyt|\vx(s\leq T)\big).
    \end{aligned} \label{eq:filter_smoother_distr}
\end{equation}

\subsection{Explicit Expression of the ACI Framework for Gaussian Distributions} \label{sec:ACI_Foundations}

The relative entropy (also known as the Kullback-Leibler divergence) between two given PDFs $p(\mathbf{u})$ and $q(\mathbf{u})$ is defined as \cite{kleeman2011information}:
\begin{equation} \label{eq:relative_entropy}
    \mathcal{P}\big(p(\mathbf{u}),q(\mathbf{u})\big)=\int p(\mathbf{u})\log\left(\frac{p(\mathbf{u})}{q(\mathbf{u})}\right)\d\mathbf{u},
\end{equation}
which is positive unless $p(\mathbf{u})=q(\mathbf{u})$ and is invariant under general nonlinear changes of the state variables. If
\begin{equation}  \label{eq:RE_filter_smoother}
    \mathcal{P}\big(p_t^{\text{s}}(\vy|\vx),p_t^{\text{f}}(\vy|\vx)\big)>0
\end{equation}
holds at time $t\in[0,T]$, we can then establish an instantaneous assimilative causal link where $\vyt$ is identified as the cause of $\vx$ under the ACI framework, which we denote as:
\begin{equation} \label{eq:ACI_cause_notation}
    \vyt \rightarrow \vx.
\end{equation}
In the main text, ACI is formulated through $\mathcal{P}$. Nonetheless, the ACI framework is divergence-independent, so other $f$-divergences can be adopted as to measure the discrepancy between the filter and smoother distributions. This is because their positive-definiteness solely depends on the PDFs themselves.

The relative entropy \eqref{eq:relative_entropy} benefits from a simple and explicit formula when the distributions $p$ and $q$ are real-valued Gaussian densities. Specifically, for $p(\mathbf{u})$ being the PDF of $\mathcal{N}_M(\vm{p},\mr{p})$ and $q(\mathbf{u})$ of $\mathcal{N}_M(\vm{q},\mr{q})$, then \cite{cai2002mathematical, kleeman2002measuring}:
\begin{equation} \label{eq:signaldispersion}
    \mathcal{P}\big(p(\mathbf{u}),q(\mathbf{u})\big)=\frac{1}{2}(\vm{p}-\vm{q})^\mathtt{T}\mr{q}^{-1}(\vm{p}-\vm{q}) +\frac{1}{2}\left(\mathrm{tr}(\mr{p}\mr{q}^{-1})-M-\log(\mathrm{det}(\mr{p}\mr{q}^{-1}))\right),
\end{equation}
where $M$ is the dimension of the PDFs. The quadratic-form term on the right-hand side of \eqref{eq:signaldispersion} is called the signal and measures the information gain in the mean which is weighted by $\mr{q}^{-1}$, while the second term is called the dispersion and involves only the covariance ratio $\mr{p}\mr{q}^{-1}$. Hence, \eqref{eq:signaldispersion} is known as the {signal-dispersion decomposition}.

In Algorithm \ref{algo:base_ACI} we provide a brief outline of how to implement ACI in pseudocode.

\begin{algorithm}[H]
\caption{ACI for the identification of $\vyt\rightarrow\vx$ over $t\in[0,T]$}
\label{algo:base_ACI}
\SetAlgoLined
\KwData{Observations of $\vx$ in $[0,T]$ (continuous or discrete), $p_0^{\text{f}}(\vy|\vx)$}
\KwResult{$\mathcal{P}\big(p_t^{\text{s}}(\vy|\vx),p_t^{\text{f}}(\vy|\vx)\big)$ for $t\in[0,T]$, used to assess $\vyt \rightarrow \vx$ at each time}

\For{$0\leq t\leq T$}{
Use \eqref{eq:general_CTNDS} and the observations of $\vx$ to compute $p_t^{\text{f}}(\vy|\vx)$\Comment*[l]{E.g., see \ceqref{DarkGreen}{eq:filterpdf}--\ceqref{DarkGreen}{eq:kalmangain} for \ceqref{DarkGreen}{eq:general_CTNDS} or Theorem \cref{DarkGreen}{thm:filtercgns} for the system in \ceqref{DarkGreen}{eq:CGNS}}}

\For{$T\geq t\geq 0$ \normalfont{(moving backward)}}{
Use $p_t^{\text{f}}(\vy|\vx)$ and the observations of $\vx$ to compute $p_t^{\text{s}}(\vy|\vx)$\Comment*[l]{E.g., see \ceqref{DarkGreen}{eq:smootherpdf}--\ceqref{DarkGreen}{eq:adjointoperators} for \ceqref{DarkGreen}{eq:general_CTNDS} or Theorem \cref{DarkGreen}{thm:smoothercgns} for the system in \ceqref{DarkGreen}{eq:CGNS}}}

\For{$0\leq t\leq T$}{
Calculate $\mathcal{P}\big(p_t^{\text{s}}(\vy|\vx),p_t^{\text{f}}(\vy|\vx)\big)=\int p_t^{\text{s}}(\vy|\vx)\log\left(\displaystyle\frac{p_t^{\text{s}}(\vy|\vx)}{p_t^{\text{f}}(\vy|\vx)}\right)\d\vyt$\Comment*[l]{Can also use any other type of information metric, e.g., a different $f$-divergence}
\eIf{$\mathcal{P}\big(p_t^{\text{s}}(\vy|\vx),p_t^{\text{f}}(\vy|\vx)\big)>0$ }{
    $\vyt \rightarrow \vx$\;
  }{$\vy(t)\ \cancel{\rightarrow}\ \vx$\;}
}
\end{algorithm}

\subsection{Conditional ACI in the Presence of Additional Non-Target Variables} \label{sec:ACI_Generalization}
To discuss conditional ACI, which handles the case where additional non-target (or ancillary) variables are present, we split the observed variables $\vx$ into $(\vx_{\text{\normalfont{A}}},\vx_{\text{\normalfont{B}}})\in\rr^{k_{\text{\normalfont{A}}}+k_{\text{\normalfont{B}}}}$, where $k=k_{\text{\normalfont{A}}}+k_{\text{\normalfont{B}}}$. We aim to determine whether $\vy$ is the cause of $\vx_{\text{\normalfont{A}}}$ in the presence of the remaining interfering variables $\vx_{\text{\normalfont{B}}}$. Recall that $\vyt\rightarrow\vx$ is assessed by quantifying the uncertainty reduction in the smoother distribution beyond the filter one:
\begin{equation} \label{eq:filter_smoother_ancillary}
    \mathcal{P}\big(p_t^{\text{s}}(\vy|\vx_{\text{\normalfont{A}}},\vx_{\text{\normalfont{B}}}),p_t^{\text{f}}(\vy|\vx_{\text{\normalfont{A}}},\vx_{\text{\normalfont{B}}})\big), \quad t\in[0,T].
\end{equation}
Since $\vx_{\text{B}}$ interacts with both $\vx_{\text{A}}$ and $\vy$ in \eqref{eq:general_CTNDS}, it directly affects the state estimation of $\vyt$. Therefore, establishing an instantaneous conditional assimilative causal link from $\vyt$ to $\vx_{\text{A}}$ beyond $\vx_{\text{B}}$ requires careful treatment of $\vx_{\text{B}}$'s contributions in \eqref{eq:filter_smoother_ancillary}.

To accurately assess $\vx_{\text{B}}$'s influence on $\vyt$'s estimation for establishing the conditional assimilative causal link from $\vyt$ to $\vx_{\text{A}}$ beyond the presence of $\vx_{\text{B}}$, we have to examine its contributions during the two-step Bayesian data assimilation pipeline: (1) through the model dynamics during the forecast step, and (2) via the likelihood distribution during the analysis step. This systematic approach allows us to isolate and remove $\vx_{\text{B}}$'s effects on the state estimation of $\vyt$, thereby ensuring the measured uncertainty reduction in $\vyt$ stems solely from $\vx_{\text{A}}$'s influence.

\subsubsection{Analyzing how \texorpdfstring{$\vx_{\text{\normalfont{B}}}$}{xB} Affects the State Estimation of \texorpdfstring{$\vyt$}{y(t)}}
To analytically characterize how non-target variables $\vx_{\text{B}}$ influence the $\vyt$ estimation during Bayesian data assimilation, it suffices to examine the governing evolution equations for the unnormalized filter and smoother densities associated with \eqref{eq:general_CTNDS} \cite{rozovsky2012stochastic,liptser2001statistics}. Under sufficient regularity conditions, we have that the unnormalized filter density $u_t^{\text{\normalfont{f}}}(\vy|\vx)$ satisfies a linear random partial differential equation (PDE) for $0\leq t\leq T$:
\begin{equation} \label{eq:filterpdf}
    \partial u_t^{\text{\normalfont{f}}}(\vy|\vx)=\underbrace{\mathcal{L}^*u_t^{\text{\normalfont{f}}}(\vy|\vx)\rmd t}_{\displaystyle\mathclap{\text{Forecast}}} + \underbrace{\mathcal{M}^*u_t^{\text{\normalfont{f}}}(\vy|\vx)\cdot \rmd \vx}_{\displaystyle\mathclap{\text{Analysis}}},
\end{equation}
where $\displaystyle \mathcal{L}^*$ is the \textit{forward Kolmogorov} or \textit{Fokker-Planck operator}:
\begin{equation} \label{eq:fokkerplanck}
    \mathcal{L}^* u_t(\vy):=-\nabla_{\vy}\cdot \big(\mathbf{f}^\vy u_t\big)+\frac{1}{2}\nabla_{\vy}\cdot \nabla_{\vy}\cdot\big((\ms^\vy\circ\ms^\vy)u_t\big),
\end{equation}
and $\displaystyle \mathcal{M}^*$ is the \textit{forward} or \textit{filter Kalman update operator}:
\begin{equation} \label{eq:kalmangain}
    \mathcal{M}^*u_t(\vy):=(\ms^\vx\circ\ms^\vx)^{-1/2}\big(-\nabla_\vy\cdot \big((\ms^\vx\circ\ms^\vx)^{-1/2}(\ms^\vx\circ\ms^\vy) u_t\big)+(\ms^\vx\circ\ms^\vx)^{-1/2}\vf^\vx u_t\big),
\end{equation}
while the unnormalized smoother density $u_t^{\text{\normalfont{s}}}(\vy|\vx)$ satisfies a backward linear random PDE for $T\geq t\geq 0$:
\begin{equation} \label{eq:smootherpdf}
    \smooth{\partial u_t^{\text{\normalfont{s}}}}(\vy|\vx)=\underbrace{\mathcal{L}u_t^{\text{\normalfont{s}}}(\vy|\vx)\rmd t}_{\displaystyle\mathclap{\text{Forecast}}} + \underbrace{\mathcal{M}u_t^{\text{\normalfont{s}}}(\vy|\vx)\cdot \smooth{\rmd \vx}}_{\displaystyle\mathclap{\text{Analysis}}},
\end{equation}
where $\mathcal{L}$ is the \textit{backward Kolmogorov operator} and formal adjoint of $\displaystyle \mathcal{L}^*$ in \eqref{eq:fokkerplanck}, while $\mathcal{M}$ is the \textit{backward} or \textit{smoother Kalman update operator} and formal adjoint of the operator $\displaystyle \mathcal{M}^*$ appearing in \eqref{eq:kalmangain}, which are both explicitly given by:
\begin{subequations} \label{eq:adjointoperators}
   \begin{align}
        \mathcal{L} u_t(\vy)&=\mathbf{f}^\vy\cdot \nabla_{\vy}u_t+\frac{1}{2}\mathrm{tr}\big((\ms^\vy\circ\ms^\vy)\nabla_{\vy}^2u_t\big), \label{eq:adjointoperators1}\\
        \mathcal{M}u_t(\vy)&=(\ms^\vx\circ\ms^\vx)^{-1/2}\big((\ms^\vx\circ\ms^\vx)^{-1/2}(\ms^\vx\circ\ms^\vy)\nabla_\vy u_t+(\ms^\vx\circ\ms^\vx)^{-1/2}\vf^\vx u_t\big). \label{eq:adjointoperators2}
    \end{align}
\end{subequations}
The diffusion interactions $(\ms^\vy\circ\ms^\vy)$ and $(\ms^\vx\circ\ms^\vy)$ are defined in the same manner as $\ms^\vx\circ\ms^\vx$, while $\displaystyle \smooth{\partial u_t^{\text{\normalfont{s}}}}$ and $\displaystyle \smooth{\rmd \vx}$ denote backward stochastic It\^{o} integrals (i.e., the negative of the usual differentials up to the principal (linear) part).

\subsubsection{Managing the Influence of \texorpdfstring{$\vx_{\text{B}}$}{xB} on \texorpdfstring{$\vyt$}{y(t)}'s Estimation}
\label{subsec:handling_effects}
From \eqref{eq:filterpdf} and \eqref{eq:smootherpdf}, we can deduce that as to eliminate the effect of $\vx_{\text{B}}$ on the state estimation of $\vyt$, while retaining only $\vx_{\text{A}}$'s observational influence, we can simply assign infinite uncertainty to $\vx_{\text{B}}$ during the analysis step. This follows from the Kalman update's inverse dependence on the observational uncertainty in \eqref{eq:kalmangain} and \eqref{eq:adjointoperators2}. This manipulation effectively nullifies $\vx_{\text{B}}$'s contributions to the state update, thus preventing it from impacting the uncertainty reduction (the key mechanism behind ACI), while still preserving its role in the forecast dynamics.

\paragraph{Implementation via posterior densities.}
We implement this proposition through a formal limit operation on the posterior filter and smoother PDFs (normalizations of the densities in \eqref{eq:filterpdf} and \eqref{eq:smootherpdf}, respectively),  therefore defining for each $t\in[0,T]$:

\begin{subequations} \label{eq:filter_smoother_ancillary_inf_uncert}
    \begin{align}
        p^{\text{f}|\vx_{\text{B}}}_t(\vy|\vx_{\text{A}}) &:= \lim_{\mathrm{Var}(\vx_{\text{B}}(t))\to+\infty} p\big(\vy(t)\big|\vx_{\text{A}}(s\leq t),\vx_{\text{B}}(s\leq t)\big) \label{eq:filter_smoother_ancillary_inf_uncert1} \\
        p^{\text{s}|\vx_{\text{B}}}_t(\vy|\vx_{\text{A}}) &:= \lim_{\mathrm{Var}(\vx_{\text{B}}(t))\to+\infty} p\big(\vy(t)\big|\vx_{\text{A}}(s\leq T),\vx_{\text{B}}(s\leq T)\big) \label{eq:filter_smoother_ancillary_inf_uncert2}
    \end{align}
\end{subequations}

The limit $\mathrm{Var}(\vx_{\text{B}}(t)) \to +\infty$ represents taking $\vx_{\text{B}}$'s marginal likelihood uncertainty to infinity, which removes $\vx_{\text{B}}$'s influence in the Kalman updates, which occurs through the Kalman gain's effect on the innovation process in \eqref{eq:kalmangain} and \eqref{eq:adjointoperators2}, while preserving its dynamic role during forecasts in \eqref{eq:fokkerplanck} and \eqref{eq:adjointoperators1}.

\paragraph{Conditional assimilative causal links.}
Using the distributions in \eqref{eq:filter_smoother_ancillary_inf_uncert}, we generalize the ACI framework. When:
\begin{equation} \label{eq:RE_filter_smoother_general}
    \mathcal{P}\big(p^{\text{s}|\vx_{\text{B}}}_t(\vy|\vx_{\text{A}}), p^{\text{f}|\vx_{\text{B}}}_t(\vy|\vx_{\text{A}})\big) > 0,
\end{equation}
we establish a \textit{conditional assimilative causal link} at $t\in[0,T]$, denoted as:
\begin{equation} \label{eq:ACI_cause_cond_notation}
    \big(\vyt\rightarrow\vx_{\text{A}}\big) \big| \vx_{\text{B}}.
\end{equation}
This indicates $\vyt$ as the cause of $\vx_{\text{A}}$ conditioned on $\vx_{\text{B}}$ when the relative entropy between the smoother- and filter-based distributions in \eqref{eq:filter_smoother_ancillary_inf_uncert} is nonzero.

In Algorithm \ref{algo:cond_ACI} we provide a brief outline of how to implement conditional ACI in pseudocode.

\begin{algorithm}[H]
\caption{Conditional ACI for the identification of $\big(\vyt\rightarrow\vx_{\text{A}}\big) \big| \vx_{\text{B}}$ over $t\in[0,T]$}
\label{algo:cond_ACI}
\SetAlgoLined
\KwData{Observations of $\vx$ in $[0,T]$ (continuous or discrete), $p_0^{\text{f}}(\vy|\vx)$}
\KwResult{$\mathcal{P}\big(p^{\text{\normalfont{s}}|\vx_{\text{\normalfont{B}}}}_t(\vy|\vx_{\text{\normalfont{A}}}), p^{\text{\normalfont{f}}|\vx_{\text{\normalfont{B}}}}_t(\vy|\vx_{\text{\normalfont{A}}})\big)$ for $t\in[0,T]$, used to assess $\big(\vyt\rightarrow\vx_{\text{A}}\big) \big| \vx_{\text{B}}$ at each time}

\For{$0\leq t\leq T$}{
Use \eqref{eq:general_CTNDS} and the observations of $\vx=(\vx_{\text{A}},\vx_{\text{B}})^\mathtt{T}$ to compute $p^{\text{\normalfont{f}}|\vx_{\text{\normalfont{B}}}}_t(\vy|\vx_{\text{\normalfont{A}}})$ by letting $\mathrm{Var}(\vx_{\text{B}}(t))\to+\infty$ during the corresponding analysis step of Bayesian data assimilation, defined by the forward/filter Kalman update operator $\mathcal{M}^*$ in \eqref{eq:kalmangain}\Comment*[l]{E.g., see Theorem \cref{DarkGreen}{thm:filtercgns} and the proof of Theorem \cref{DarkGreen}{thm:nilcondcausality} for the system in \ceqref{DarkGreen}{eq:CGNS} as a concrete example}}

\For{$T\geq t\geq 0$ \normalfont{(moving backward)}}{
Use $p^{\text{\normalfont{f}}|\vx_{\text{\normalfont{B}}}}_t(\vy|\vx_{\text{\normalfont{A}}})$ and the observations of $\vx=(\vx_{\text{A}},\vx_{\text{B}})^\mathtt{T}$ to compute $p^{\text{\normalfont{s}}|\vx_{\text{\normalfont{B}}}}_t(\vy|\vx_{\text{\normalfont{A}}})$ by letting $\mathrm{Var}(\vx_{\text{B}}(t))\to+\infty$ during the corresponding analysis step of Bayesian data assimilation, defined by the backward/smoother Kalman update operator $\mathcal{M}$ in \eqref{eq:adjointoperators2}\Comment*[l]{E.g., see Theorem \cref{DarkGreen}{thm:smoothercgns} and the proof of Theorem \cref{DarkGreen}{thm:nilcondcausality} for the system in \ceqref{DarkGreen}{eq:CGNS} as a concrete example}}

\For{$0\leq t\leq T$}{
Calculate $\mathcal{P}\big(p^{\text{\normalfont{s}}|\vx_{\text{\normalfont{B}}}}_t(\vy|\vx_{\text{\normalfont{A}}}), p^{\text{\normalfont{f}}|\vx_{\text{\normalfont{B}}}}_t(\vy|\vx_{\text{\normalfont{A}}})\big)=\int p^{\text{\normalfont{s}}|\vx_{\text{\normalfont{B}}}}_t(\vy|\vx_{\text{\normalfont{A}}})\log\left(\displaystyle\frac{p^{\text{\normalfont{s}}|\vx_{\text{\normalfont{B}}}}_t(\vy|\vx_{\text{\normalfont{A}}})}{p^{\text{\normalfont{f}}|\vx_{\text{\normalfont{B}}}}_t(\vy|\vx_{\text{\normalfont{A}}})}\right)\d\vyt$\Comment*[l]{Can also use any other type of information metric, e.g., a different $f$-divergence}
\eIf{$\mathcal{P}\big(p^{\text{\normalfont{s}}|\vx_{\text{\normalfont{B}}}}_t(\vy|\vx_{\text{\normalfont{A}}}), p^{\text{\normalfont{f}}|\vx_{\text{\normalfont{B}}}}_t(\vy|\vx_{\text{\normalfont{A}}})\big) > 0$ }{
    $\big(\vyt\rightarrow\vx_{\text{A}}\big) \big| \vx_{\text{B}}$\;
  }{$ \big(\vyt\ \cancel{\rightarrow}\ \vx_{\text{A}}\big) \big| \vx_{\text{B}}$\;}
}
\end{algorithm}

\subsubsection{Arguments in Favor of the Conditional ACI Framework} \label{sec:ACI_Generalized_Justification}
The generalized ACI framework presented in Section \ref{sec:ACI_Generalization} and defined by \eqref{eq:filter_smoother_ancillary_inf_uncert}--\eqref{eq:RE_filter_smoother_general} offers several key advantages for establishing \eqref{eq:ACI_cause_cond_notation}:

\begin{enumerate}[leftmargin=*,label=\textbf{(\arabic*)},itemsep=0.5ex]
    \item Complete elimination of $\vx_{\text{B}}$'s influence during data assimilation, ensuring any uncertainty reduction in $\vyt$ stems solely from $\vx_{\text{A}}$, consistent with our ACI objectives. Furthermore, during analysis, the correlations between $\vx_{\text{A}}$ and $\vx_{\text{B}}$ in their likelihood are necessarily sent to zero as the uncertainty of $\vx_{\text{B}}$ goes to infinity. This further resolves the potential for spurious associations arising over the observable state space that can flow over when inferring causality from $\vyt$ to $\vx_{\text{A}}$; see Theorem \ref{thm:nilcondcausality}. A rough analogue to this approach would be the assignment of an improper prior as an uninformative or diffuse prior in parametric Bayesian inference.

    \item While discounting $\vx_{\text{B}}$'s observational influence, we retain its time series in $\vyt$'s estimation, analogous to transfer entropy approaches where conditional causal links account for spurious effects. Our framework naturally conditions on $\vx_{\text{B}}$'s observations while assigning them uniform likelihood weight through the infinite uncertainty assumption.

    \item Preservation of the original dynamical system structure. Unlike methods that modify dynamics or eliminate governing equations (potentially yielding nonphysical results), our approach maintains model integrity while properly handling $\vx_{\text{B}}$'s spurious contributions through the Bayesian framework.

    \item Straightforward implementation. When the filter/smoother distributions in \eqref{eq:filter_smoother_distr} solve continuous random PDEs, the limiting distributions in \eqref{eq:filter_smoother_ancillary_inf_uncert} emerge naturally under continuous parameter dependence, particularly as $\mathrm{Var}(\vx_{\text{B}}(t))\to+\infty$ for the relevant covariance matrix elements.
\end{enumerate}

\subsubsection{Further Justification: Why Infinite Uncertainty Outperforms Marginalization}
To further justify our approach for establishing temporally-varying conditional assimilative causal links in the presence of ancillary variables, consider treating the non-target variables $\vx_{\text{B}}$ as unobserved. In this scenario, the smoother PDF at time $t\in[0,T]$ becomes:
\begin{equation*}
    p_t\big(\vy,\vx_{\text{B}}\big|\vx_{\text{A}}(s\leq T)\big).
\end{equation*}
One might attempt to infer a conditional causal relationship from $\vyt$ to $\vx_{\text{A}}$ by marginalizing over $\vx_{\text{B}}(t)$:
\begin{equation*}
    p_t\big(\vy\big|\vx_{\text{A}}(s\leq T)\big)=\int p_t\big(\vy,\vx_{\text{B}}\big|\vx_{\text{A}}(s\leq T)\big)\d\vx_{\text{B}}(t).
\end{equation*}
However, this approach proves problematic even for simple turbulent nonlinear dynamical systems. Crucially, the issue persists even when conditioning on $\vx_{\text{B}}(t)$ rather than marginalizing.

The fundamental flaw lies in how correlations between $\vyt$ and $\vx_{\text{B}}(t)$ affect uncertainty reduction during state estimation. These correlations can introduce spurious causal relationships that become unavoidable once the model forecast begins propagating forward, regardless of time-step size. Consider the following causal network:

\begin{figure}[H]
\centering
\resizebox{0.4\textwidth}{!}{%
\begin{circuitikz}
\tikzstyle{every node}=[font=\Huge]
\draw [ line width=2pt ] (2.5,14.75) circle (0.75cm);
\node [font=\Huge] at (2.5,14.75) {$\vy$};
\draw [ line width=2pt ] (7.25,14.75) circle (0.75cm);
\node [font=\Huge] at (7.25,14.75) {$\vx_{\text{\normalfont{B}}}$};
\draw [line width=2pt, ->, >=Stealth] (3.25,14.75) -- (6.5,14.75);
\draw [ line width=2pt ] (12,14.75) circle (0.75cm);
\node [font=\Huge] at (12,14.75) {$\vx_{\text{\normalfont{A}}}$};
\draw [line width=2pt, ->, >=Stealth] (8,14.75) -- (11.25,14.75);
\end{circuitikz}
}%
\end{figure}
For this causal chain, this marginalization (or conditioning) method might falsely suggest a direct causal link from $\vyt$ to $\vx_{\text{A}}$. While these filter- and smoother-based distributions of $\vyt$ given $\vx_{\text{A}}$ may differ after marginalization (or conditioning), we cannot attribute causation to $\vx_{\text{A}}$ under \eqref{eq:RE_filter_smoother}. The uncertainty reduction in $\vyt$ inherently incorporates information transfer from $\vx_{\text{B}}(t)$ due to their coupled dynamics, unlike the distributions in \eqref{eq:filter_smoother_ancillary_inf_uncert} where $\vx_{\text{B}}$ trajectories are known but carry infinite uncertainty in their likelihood.

\subsection{Dynamic Discovery of the Causal Influence Range (CIR)} \label{sec:CIR}
The relative entropy between the smoother and filter solutions in \eqref{eq:RE_filter_smoother} quantifies both the existence and strength of an assimilative causal link from $\vyt$ to $\vx$. Similarly, the generalized metric in \eqref{eq:RE_filter_smoother_general} measures the conditional causal link $\big(\vyt\rightarrow\vx_{\text{A}}\big) \big| \vx_{\text{B}}$ and its intensity. However, these ACI metrics alone cannot determine the \emph{temporal extent} of causal influence, namely, how many future values of $\vx$ (or $\vx_{\text{A}}$ conditioned on $\vx_{\text{B}}$) are affected by $\vyt$ over $[t,T]$. We define this temporal characteristic as the {causal influence range} ({CIR}), representing the future time window where $\vyt$'s causal impact persists. Below we develop the theory for unconditional CIR; the conditional case ({CCIR}) follows analogously by substituting the appropriate equations with those from Section \ref{sec:ACI_Generalization}.

In chaotic turbulent systems, new observations of $\vx$ influence $\vyt$'s estimation only within finite time windows. More chaotic dynamics accelerate memory decay, typically exhibiting exponential decay modulated by factors like $\vx$'s signal-to-noise ratio \cite{andreou2024adaptive}. Our framework estimates this CIR duration for $\vyt \rightarrow \vx$ by determining how many future $\vx$ values are meaningfully affected.

For arbitrary $t\in[0,T]$ and $T'\in[t,T]$, we compare two smoother distributions:
\begin{itemize}
    \item \emph{Complete} smoother $p_{t}(\vy|\vx(s\leq T))$: Optimal (minimum variance) estimation using all available data.
    \item \emph{Lagged} smoothers $p_{t}(\vy|\vx(s\leq T'))$: Suboptimal estimations using partial data ($T'\leq T$).
\end{itemize}
The divergence between these distributions reveals the temporal decay of $\vyt$'s causal influence on future $\vx$ values as $T'$ increases from $t$ toward $T$.

\subsubsection{Quantifying the CIR}

The finite nature of the CIR motivates measuring how the complete and lagged smoother distributions differ as functions of observational time $T'\in[t,T]$ when evaluating the recovery of the state of $\vy$ at $t$. We naturally quantify this discrepancy using the relative entropy:
\begin{equation} \label{eq:metric_CIR}
    \delta(T';t) := \mathcal{P}\big(p_t(\vy|\vx(s\leq T)), p_t(\vy|\vx(s\leq T'))\big), \quad 0\leq t \leq T'\leq T
\end{equation}
where the second input in $\delta$, i.e., $t$, is the time on which the posterior distributions of $\vy$ are evaluated. For analysis, we normalize the domain of $\delta$ to $\mathrm{I}=[0,1]$ via the transformation $h(\tau) = t + \tau(T-t)$:
\begin{equation} \label{eq:std_metric_CIR}
    \hat{\delta}(\tau;t) := \delta(t + \tau(T-t); t), \quad \tau\in\mathrm{I}.
\end{equation}
Let $\mathcal{M}(t) := \|\hat{\delta}(\cdot;t)\|_{L^\infty(\mathrm{I})}$ denote the maximum divergence, which exists in practical applications. Notably:
\begin{itemize}
    \item At $\tau=1$, $\hat{\delta}(1;t)=0$ by the positive-definiteness of the relative entropy.
    \item At $\tau=0$, $\hat{\delta}(0;t)$ recovers the standard ACI metric from \eqref{eq:RE_filter_smoother}.
    \item For $\tau\in(0,1)$, $\hat{\delta}(\tau;t)$ measures the lack of information from incorporating a limited portion of the future observations of $\vx$ after $t$, i.e., up until $h(\tau)$, $\vx(s\leq h(\tau))$, but not up to $T$.
\end{itemize}
For chaotic systems, $\mathcal{M}(t)$ typically occurs at $\tau=0$ or shortly thereafter. The interval where $\hat{\delta}(\tau;t)$ remains large suggests significant causal influence under ACI, since it indicates that there is substantial information gain to be incurred by additionally incorporating the future observations of $\vx$ from $[h(\tau),T]$.

\paragraph{Formal definition of the CIR.}
For a threshold $\varepsilon\geq 0$, define:
\begin{equation}\label{eq:Jt}
    \mathrm{J}_t(\varepsilon) := \{\tau\in[0,1] : \hat{\delta}(\tau;t) > \varepsilon\}.
\end{equation}
The CIR length of $\mathbf{y}(t)\rightarrow\vx$ is then given by:
\begin{equation} \label{eq:subjective_CIR_length}
    \overset{\sim}{\uptau}_{\vyt\rightarrow\vx}(\varepsilon) := (T-t)\sup\mathrm{J}_t(\varepsilon),
\end{equation}
with the convention $\overset{\sim}{\uptau}_{\vyt\rightarrow\vx}(\varepsilon)=0$ when $\mathrm{J}_t(\varepsilon)=\emptyset$ (true for $\varepsilon\geq \mathcal{M}(t)$). The associated CIR interval becomes:
\begin{equation} \label{eq:subjective_CIR}
    \widetilde{\mathrm{CIR}}_{\vyt\rightarrow\vx}(\varepsilon) := [t, t + \overset{\sim}{\uptau}_{\vyt\rightarrow\vx}(\varepsilon)].
\end{equation}
As aforementioned, for most practical chaotic dynamical systems, the maximum information deficit $\mathcal{M}(t)$ typically occurs at $\tau=0$ or shortly thereafter. This reflects the characteristic memory decay in such systems, where the influence on the state estimation of $\vyt$ from the future observations of $\vx$ in $[h(\tau),T]$ diminishes with increasing lead time $\tau$.
Note also that under the mild assumption that $\hat{\delta}(\cdot;t)$ is Lebesgue measurable for each $t\in[0,T]$, which is a condition satisfied when $\delta(\cdot;t)$ is Borel measurable (that holds for the relative entropy and most complex turbulent systems found in practice), we obtain the fundamental inequality:
\begin{equation} \label{eq:subjective_CIR_length_ineq}
    \overset{\sim}{\uptau}_{\vyt\rightarrow\vx}(\varepsilon)\geq  (T-t)\lambda_\mathrm{I}\big(\mathrm{J}(\varepsilon)\big),
\end{equation}
where $\lambda_\mathrm{I}(\cdot)$ represents the Lebesgue measure restricted to $\mathrm{I}$. Thus, $\widetilde{\mathrm{CIR}}$ provides a liberal estimate of the temporal window where $\vyt$ causally influences $\vx$.

\paragraph{Interpretation.}
The decreasing trend of $\delta(T';t)$ reflects the finite memory of chaotic dynamics:
\begin{itemize}
    \item $\delta(t;t)$ assesses whether $\vyt\rightarrow\vx$ via ACI.
    \item $\delta(T;t)=0$ shows complete information incorporation.
    \item The significant $\delta(T';t)$ region marks where $\vyt$ substantially affects future $\vx$.
\end{itemize}

\subsubsection{Subjective and Objective Perspectives of the CIR}

The current CIR framework, while intuitive, introduces subjectivity through its dependence on the threshold parameter $\varepsilon$. We therefore distinguish between two perspectives:
\begin{itemize}
    \item \textbf{Subjective CIR:} Denoted by $\overset{\sim}{\uptau}_{\vyt\rightarrow\vx}(\varepsilon)$ (length) and $\widetilde{\mathrm{CIR}}_{\vyt\rightarrow\vx}(\varepsilon)$ (interval), marked with tildes to emphasize their $\varepsilon$-dependence.
    \item \textbf{Objective CIR:} An $\varepsilon$-independent alternative developed in the sequel.
\end{itemize}

\paragraph{Objective CIR definition.}
The natural objective measure averages the subjective CIR lengths over all possible thresholds:
\begin{equation} \label{eq:objective_CIR_length}
    \uptau_{\vyt\rightarrow\vx} := \frac{1}{\mathcal{M}(t)}\int_0^{\mathcal{M}(t)} \overset{\sim}{\uptau}_{\vyt\rightarrow\vx}(\varepsilon) \d\varepsilon, \quad t\in[0,T],
\end{equation}
where $\mathcal{M}(t) = \max_{\tau\in[0,1]} \hat{\delta}(\tau;t)$. Dividing by $\mathcal{M}(t)$ guarantees the unit of the objective CIR is temporal. This averaging process:
(a) yields dimensionally consistent results (time units), (b) guarantees $\uptau_{\vyt\rightarrow\vx} \in [0,T-t]$, and (c) defaults to $0$ when $\mathcal{M}(t)=0$ (no causal link).

The corresponding {objective CIR interval} is then defined as:
\begin{equation} \label{eq:objective_CIR}
    \mathrm{CIR}_{\vyt\rightarrow\vx} := [t, t + \uptau_{\vyt\rightarrow\vx}], \quad t\in[0,T].
\end{equation}
By construction, this objective interval is always contained within the maximal subjective CIR and always contains the minimal one.

\paragraph{Analogy with correlation analysis.}
These CIR measures parallel concepts in correlation analysis:
\begin{itemize}
    \item \textbf{Subjective CIR} resembles the autocorrelation function (ACF), where memory duration depends on a subjective threshold.
    \item \textbf{Objective CIR} mirrors the decorrelation time, providing a threshold-free measure through integration.
\end{itemize}
Just as decorrelation time objectively quantifies system memory by integrating the ACF, our objective CIR integrates over all possible thresholds to remove subjectivity, by assigning an impartial uniform weight on each subjective CIR, while capturing the essential temporal influence structure.

\subsubsection{Efficient Computation of the Objective CIR}\label{sec:Computing_CIR}
While \eqref{eq:objective_CIR_length} offers a theoretically sound definition of the objective CIR length, its practical implementation faces significant computational challenges. The subjective CIR length $\overset{\sim}{\uptau}_{\vyt\rightarrow\vx}(\varepsilon)$ generally admits no analytical solution across the full parameter space $\varepsilon\in[0,\mathcal{M}(t)]$ and $t\in[0,T]$. Numerically evaluating the integral in \eqref{eq:objective_CIR_length} requires repeated smoother computations and leads to a computational complexity that scales at best quadratically with the number of discretization points adopted for quadrature.

We overcome this computational limitation by developing an efficient lower-bound approximation that becomes exact when the information loss metric $\delta(T';t)$ in \eqref{eq:metric_CIR} is non-increasing with respect to future observations (i.e., non-increasing in $T'$). This is the typical case in most applications as the memory usually decays as the lead time increases. The following theorem formalizes this approach, requiring only that $\hat{\delta}(\cdot;t)$ is Lebesgue measurable for each $t\in[0,T]$.

\begin{theorem}[Computing the Objective CIR]
    \label{thm:obj_subj_CIR_connection}
    Assume that there exists $\mathcal{M}(t):=\lVert \hat{\delta}(\cdot;t)\rVert_{L^\infty(\mathrm{I})}=\lVert \delta(\cdot;t)\rVert_{L^\infty(\mathrm{[t,T]})}>0$ for each $t\in[0,T]$. Then, for $t\in[0,T]$:
    \begin{equation}
        \begin{aligned}
            \frac{1}{\mathcal{M}(t)}\int^T_t\delta(T';t)\rmd T'&=\frac{T-t}{\mathcal{M}(t)}\int^1_0\hat{\delta}(\tau;t)\rmd\tau\\
            &=\frac{T-t}{\mathcal{M}(t)}\int_0^{\mathcal{M}(t)}\lambda_\mathrm{I}\big(\mathrm{J}_t(\varepsilon)\big)\rmd\varepsilon\\
            &\leq \frac{1}{\mathcal{M}(t)}\int_0^{\mathcal{M}(t)} \overset{\sim}{\uptau}_{\vyt\overset{{\text{\normalfont{ACI}}}}{\longrightarrow}\vx}(\varepsilon) \rmd\varepsilon=\uptau_{\vyt\overset{{\text{\normalfont{ACI}}}}{\longrightarrow}\vx},
        \end{aligned} \label{eq:obj_subj_CIR_connection}
    \end{equation}
    where $\hat{\delta}(\tau;t)$, $\mathrm{J}_t(\varepsilon)$, and $\overset{\sim}{\uptau}_{\vyt\overset{{\text{\normalfont{ACI}}}}{\longrightarrow}\vx}(\varepsilon)$ are defined in \eqref{eq:std_metric_CIR}, \eqref{eq:Jt}, and \eqref{eq:subjective_CIR_length}, respectively.
    The equality (in the last inequality in \eqref{eq:obj_subj_CIR_connection}) is achieved if and only if $\hat{\delta}(\cdot;t)$ is a nonincreasing function in $\mathrm{I}$. In that case, we also have:
    \begin{equation*}
        \mathcal{M}(t)=\hat{\delta}(0;t)=\delta(t;t)=\mathcal{P}\big(p_{t}^{\text{\ns}}(\vy|\vx),p_{t}^{\text{\nf}}(\vy|\vx)\big).
    \end{equation*}
\end{theorem}

\begin{proof}[\underline{Proof of Theorem~\ref{thm:obj_subj_CIR_connection}}]
    First, observe that $\lambda_\mathrm{I}(\cdot)$ defines a probability measure on $\mathrm{I} = [0,1]$. Under the Lebesgue measurability assumption for $\hat{\delta}(\cdot;t)$, we can then interpret $\hat{\delta}(\cdot;t)$ as a continuous random variable on the probability space $(\mathrm{I}, \mathcal{L}_\mathrm{I}, \lambda_\mathrm{I})$, $\mathcal{L}_\mathrm{I}$ as the $\sigma$-algebra of Lebesgue measurable subsets of $\mathrm{I}$, where the support of $\hat{\delta}(\cdot;t)$ is $[0, \mathcal{M}(t)]$.

    $\hat{\delta}(\cdot;t)$'s survival function $\lambda_\mathrm{I}(\mathrm{J}_t(\varepsilon))$ is well-defined because (a) it is nonincreasing and right-continuous with left limits, (b) it satisfies $\lambda_\mathrm{I}(\mathrm{J}_t(\varepsilon)) = 1$ for $\varepsilon \leq 0$, and (c) it satisfies $\lambda_\mathrm{I}(\mathrm{J}_t(\varepsilon)) = 0$ for $\varepsilon \geq \mathcal{M}(t)$.

    Applying the expected survival time formula to this nonnegative, compactly supported random variable yields:
    \begin{equation*}
        \mathbb{E}[\hat{\delta}(\cdot;t)]
        := \int_0^1 \hat{\delta}(\tau;t) \rmd\tau
        \equiv \frac{1}{T-t}\int_t^T \delta(T';t) \rmd T'
        = \int_0^{\mathcal{M}(t)} \lambda_\mathrm{I}(\mathrm{J}_t(\varepsilon)) \rmd\varepsilon.
    \end{equation*}

    The theorem's main result follows by dividing through by $\mathcal{M}(t) > 0$,  multiplying by $T-t$, and applying the inequality in \eqref{eq:subjective_CIR_length_ineq}. The equality condition follows immediately from \eqref{eq:subjective_CIR_length_ineq}, since the equality there holds if and only if $\hat{\delta}(\cdot;t)$ is nonincreasing on $\mathrm{I}$.
\end{proof}

The result of Theorem \ref{thm:obj_subj_CIR_connection} provides a consistent and computationally efficient way to approximate from below the objective CIR length:
\begin{equation} \label{eq:objective_CIR_length_efficient}
    \uptau^{\text{approx}}_{\vyt\rightarrow\vx}:=\frac{T-t}{\mathcal{M}(t)}\int^1_0\hat{\delta}(\tau;t)\rmd\tau=\frac{1}{\mathcal{M}(t)}\int^T_t\delta(T';t)\rmd T'\leq \uptau_{\vyt\rightarrow\vx}, \quad t\in[0,T].
\end{equation}
The inequality reduces to an equality precisely when $\hat{\delta}(\cdot;t)$ is nonincreasing on $\mathrm{I}$, which is a condition satisfied in most practical applications.
A dimensional analysis of \eqref{eq:objective_CIR_length_efficient} reveals that $\uptau^{\text{approx}}_{\vyt\rightarrow\vx}$ maintains consistent time units, confirming the physical validity of this approximation. This consistency mirrors the dimensional properties established for the exact objective CIR length in \eqref{eq:objective_CIR_length}.

\section{Analytically Tractable Nonlinear Systems for ACI Analysis} \label{sec:Computational_Tools}

In this section, we introduce a broad class of nonlinear stochastic dynamical systems that possess analytically tractable Bayesian data assimilation solutions. These closed-form solutions allow direct investigation of ACI without relying on ensemble approximation methods.

\subsection{Conditional Gaussian Nonlinear Systems (CGNS)} \label{sec:CGNS}

Numerical approximations, such as ensemble or particle methods, have to be adopted to find the filter and smoother solutions for general nonlinear dynamics systems. Nevertheless, analytical solutions are available for the data assimilation solutions for a broad class of nonlinear systems, known as conditional Gaussian nonlinear systems (CGNSs) \cite{liptser2001statistics}.

\subsubsection{The Modeling Framework}

A CGNS consists of two It\^{o} diffusion processes and has the following form \cite{liptser2001statistics}:
\begin{subequations} \label{eq:CGNS}
    \begin{align}
        \d\vxt &= \big(\ml^\vx(t,\vx)\vyt+\vf^\vx(t,\vx)\big)\d t+\ms_1^\vx(t,\vx) \rmd\vw_1(t)+\ms_2^\vx(t,\vx)\rmd\vw_2(t), \label{eq:CGNS1}\\
        \d\vyt &= \big(\ml^\vy(t,\vx)\vyt+\vf^\vy(t,\vx)\big)\d t+\ms_1^\vy(t,\vx)\rmd\vw_1(t)+\ms_2^\vy(t,\vx)\rmd\vw_2(t). \label{eq:CGNS2}
    \end{align}
\end{subequations}
where the matrices $\ml^\vx$, $\ml^\vy$, $\ms_1^\vx$, $\ms_2^\vx$, $\ms_1^\vy$, $\ms_2^\vy$ and the vectors $\vf^\vx$, $\vf^\vy$ can contain arbitrarily nonlinear functions of $\vx$. The state variable $\vy$ appears in the system in a conditionally linear way but it can interact nonlinearly with $\vx$. Therefore, the system is overall highly nonlinear and can generate strong non-Gaussian feature in both the marginal and joint distributions, corresponding to extreme events, intermittency, regime switching, etc. When all these matrices and vectors become only a function of time, i.e., they do not depend on $\vx$, then the filtering and smoothing of \eqref{eq:CGNS} collapses to the setup of the classical Kalman-Bucy filter \cite{kalman1961new} and Rauch-Tung-Striebel smoother \cite{rauch1965maximum}.

The CGNS framework broadly applies to many complex nonlinear systems across disciplines, including physics-constrained stochastic models (e.g., noisy Lorenz systems, low-order Charney-DeVore flows, and topographic mean-flow interaction paradigms), stochastically coupled reaction-diffusion systems like FitzHugh-Nagumo neural models and SIR epidemic models, and multiscale geophysical flow models such as stochastic Boussinesq equations and forced rotating shallow water equations. This framework has also been successfully adapted to model realistic phenomena like the Madden-Julian oscillation and Arctic sea ice dynamics.  See a gallery of examples in  \cite{chen2018conditional}. On the other hand, many multiscale systems with general nonlinearities can be effectively approximated as CGNSs with minimal error. In geophysical and fluid systems, where nonlinearities often take quadratic forms, this approximation preserves all nonlinear interactions involving the large-scale variables $\vx$ and their couplings with the small-scale fluctuations in $\vy$. Only the self-interactions of $\vy$ are replaced by effective noise and damping terms, which is an approach rigorously justified when $\vy$ represents fast, unresolved scales. This strategy aligns with the stochastic mode reduction framework developed by Majda, Timofeyev, and Vanden-Eijnden (the MTV method) \cite{majda2003systematic}.

The conditional linearity of $\vy$ given $\vx$ allows closed-form analytic solutions for the posterior distributions of $\vy$ in Bayesian data assimilation, despite the overall nonlinearity of the system. This property fundamentally distinguishes CGNSs from linear Gaussian systems. For instance, the posterior covariance evolves according to a random Riccati equation, leading to temporal fluctuations rather than the asymptotic convergence seen in linear cases. These dynamic covariance variations are essential for capturing extreme events and intermittency with high fidelity.

\subsubsection{Filtering and Smoothing Distributions for CGNSs}

\begin{theorem}[Optimal nonlinear filter state estimation equations for CGNSs \cite{liptser2001statistics}]
    \label{thm:filtercgns}
    Let $\vxt$ and $\vyt$ satisfy \eqref{eq:CGNS1}--\eqref{eq:CGNS2}. Then, under suitable regularity conditions, the posterior distribution of $\vy(t)$ given a realization of the trajectory $\vx$ up to the current time instant $t$ (namely, the optimal filter solution) is Gaussian,
    \begin{equation*}
        \pp\big(\vyt\big|\vx(s\leq t)\big) \overset{(\rmd)}{\sim}\mathcal{N}_l\big(\vmt{\nf},\mrt{\nf}\big),
    \end{equation*}
    where
    \begin{equation*}
        \vmt{\nf}:=\ee{\vyt\big|\vx(s\leq t)} \quad\text{and}\quad \mrt{\nf}:=\ee{(\vyt-\vmt{\nf})(\vyt-\vmt{\nf})^\mathtt{T}\big|\vx(s\leq t)},
    \end{equation*}
    are the unique continuous solutions to the system of optimal nonlinear filter equations for $0\leq t\leq T$:
    \begin{subequations} \label{eq:filter}
    \begin{align}
    \rmd \vmt{\nf}&=(\ml^\vy\vm{\nf}+\vf^\vy)\rmd t+\mathbf{K}_{\text{\nf}}(\rmd \vx -(\ml^\vx\vm{\nf}+\vf^\vx)\rmd t), \label{eq:filter1}\\
    \rmd \mrt{\nf}&=\big(\ml^\vy\mr{\nf}+\mr{\nf}(\ml^\vy)^\mathtt{T}+(\ms^\vy\circ\ms^\vy)-\mathbf{K}_{\text{\nf}}(\ms^\vx\circ\ms^\vx)\mathbf{K}_{\text{\nf}}^\mathtt{T}\big)\rmd t, \label{eq:filter2}
    \end{align}
    \end{subequations}
    with
    \begin{equation} \label{eq:kalmangainfilter}
        \mathbf{K}_{\text{\nf}}(t,\vx):=\big((\ms^\vy\circ \ms^\vx)+\mr{\nf}(\ml^\vx)^\mathtt{T}\big)(\ms^\vx\circ\ms^\vx)^{-1},
    \end{equation}
    being the filter Kalman gain operator of the CGNS and $\rmd \vx -(\ml^\vx\vm{\nf}+\vf^\vx)\rmd t$ being its filter innovation process.
\end{theorem}

\begin{theorem}[Optimal nonlinear smoother state estimation backward equations for CGNSs \cite{liptser2001statistics}]
    \label{thm:smoothercgns}
    Let $\vxt$ and $\vyt$ satisfy \eqref{eq:CGNS1}--\eqref{eq:CGNS2}. Then, under suitable regularity conditions, the posterior distribution of $\vy(t)$ given a realization of the trajectory $\vx$ up to the end point $T$ (namely, the optimal smoother solution) is Gaussian,
    \begin{equation*}
        \pp\big(\vyt\big|\vx(s\leq T)\big) \overset{(\rmd)}{\sim}\mathcal{N}_l\big(\vmt{\ns},\mrt{\ns}\big),
    \end{equation*}
    where
    \begin{equation*}
        \vmt{\ns}:=\ee{\vyt\big|\vx(s\leq T)} \quad\text{and}\quad \mrt{\ns}:=\ee{(\vyt-\vmt{\ns})(\vyt-\vmt{\ns})^\mathtt{T}\big|\vx(s\leq T)},
    \end{equation*}
    are the unique continuous solutions to the system of optimal nonlinear smoother backward equations for $T\geq t\geq 0$ ($t$ running backward):
    \begin{subequations} \label{eq:smoother}
    \begin{align}
        \smooth{\rmd \vm{\ns}}(t)&=-(\ml^\vy\vm{\ns}+\vf^\vy-\mb\mr{\nf}^{-1}(\vm{\nf}-\vm{\ns}))\rmd t+\mathbf{K}_{\text{\ns}}\big(\smooth{\rmd\vx}+(\ml^\vx\vm{\ns}+\vf^\vx)\rmd t\big), \label{eq:smoother1} \\
        \smooth{\rmd \mr{\ns}}(t)&= -\big((\ma+\mb\mr{\nf}^{-1})\mr{\ns}+\mr{\ns}(\ma+\mb\mr{\nf}^{-1})^\mathtt{T}-(\ms^\vy\circ \ms^\vy)+\mathbf{K}_{\text{\ns}}(\ms^\vx\circ\ms^\vx)\mathbf{K}_{\text{\ns}}^\mathtt{T}\big)\rmd t,\label{eq:smoother2}
    \end{align}
    \end{subequations}
    with
    \begin{equation} \label{eq:kalmangainsmoother}
        \mathbf{K}_{\text{\ns}}(t,\vx):=(\ms^\vy\circ \ms^\vx)(\ms^\vx\circ\ms^\vx)^{-1},
    \end{equation}
    being the smoother Kalman gain operator of the CGNS and $\smooth{\rmd\vx}+(\ml^\vx\vm{\ns}+\vf^\vx)\rmd t$ being its smoother innovation process, where the auxiliary matrices $\ma$ and $\mb$ are defined by:
    \begin{align*}
        \ma(t,\vx)&:=\ml^\vy(t,\vx)-(\ms^\vy\circ \ms^\vx)(t,\vx)(\ms^\vx\circ \ms^\vx)^{-1}(t,\vx)\ml^\vx(t,\vx)\in\mathbb{R}^{l \times l},\\
        \mb(t,\vx) &:=  (\ms^\vy\circ \ms^\vy)(t,\vx)-(\ms^\vy\circ \ms^\vx)(t,\vx)(\ms^\vx\circ \ms^\vx)^{-1}(t,\vx)(\ms^\vx\circ \ms^\vy)(t,\vx)\in\mathbb{R}^{l \times l}.
    \end{align*}
    The backward-arrow notation in denotes a backward It\^{o} integral and is to be understood as:
    \begin{equation*}
        \smooth{\rmd \mathbf{u}}(t):=\lim_{\Delta t\to 0^+} \left(\mathbf{u}(t)-\mathbf{u}(t+\Delta t)\right),
    \end{equation*}
    i.e., $\smooth{\rmd\cdot}$ corresponds to the negative of the usual differential up to its principal (linear) part. Since \eqref{eq:smoother} are solved backward, $\left(\vm{\ns}(T),\mr{\ns}(T)\right)=\left(\vm{\nf}(T),\mr{\nf}(T)\right)$.
\end{theorem}

\subsection{Online Smoother for CGNSs} \label{sec:Online_Smoother}

The theoretical framework in Section~\ref{sec:CIR} requires computation of $\delta(T';t)$ in \eqref{eq:metric_CIR} for all $t\in[0,T]$ and $T'\in[t,T]$ to evaluate either the subjective or objective CIR of an (conditional) assimilative causal link. For CGNSs, this calculation is now feasible thanks to recent advances in online smoothing algorithms \cite{andreou2024adaptive}, which allow real-time computation of smoother distributions as $\vx$ observations arrive sequentially.

Let $\{t_j\}_{j\in\llbracket N\rrbracket}$ be a uniform partition of $[0,T]$ with mesh size $0<\dt\ll 1$ and $t_j=j\dt$, for $j\in\llbracket N\rrbracket:=\{0,1,\ldots, N\}$, where $N=T/\dt$. The assumption of uniform $\dt$ is without loss of generality \cite{andreou2024adaptive}. In what follows we use a superscript notation $\cdot^{\, j}$ to denote the discrete approximation to the continuous form of the respective function when evaluated on $t_j$, e.g., $\ml^{\vx,j}:=\ml^\vx(t_j,\vx(t_j))$ and $\vx^j:=\vx(t_j)$. Similar to the continuous-time setup, $\big(\cdot\big|\vx(s\leq n)\big)$ denotes that we are conditioning on the $\sigma$-algebra generated by $\{\vx^0,\ldots,\vx^n\}$. Define the following auxiliary matrices for $j\in\llbracket N\rrbracket$:
\begin{align*}
    \mathbf{E}^j&:=\mathbf{I}_{l\times l}+\left[(\ms^\vy\circ\ms^\vx)^j\big((\ms^\vx\circ\ms^\vx)^{j}\big)^{-1}\mathbf{G}^{\vx,j}-\mathbf{G}^{\vy,j}\right]\dt\in\mathbb{R}^{l\times l},\\
    \begin{split}
        \mathbf{F}^j&:=-\mr{\nf}^j\Big[(\mathbf{K}^j)^{\mathtt{T}}+\left((\mathbf{G}^{\vx,j})^\mathtt{T}\mathbf{K}^j\mr{\nf}^j(\mathbf{K}^j)^{\mathtt{T}}-(\mr{\nf}^j)^{-1}(\mathbf{H}^{j})^\mathtt{T}\mr{\nf}^j(\mathbf{K}^j)^\mathtt{T}+(\ml^{\vy,j})^{\mathtt{T}}(\mathbf{K}^{j})^\mathtt{T}\right)\dt\\
        &\hspace*{6cm} -(\ml^{\vx,j})^\mathtt{T}\left(\big((\ms^\vx\circ\ms^\vx)^{j}\big)^{-1}+\mathbf{K}^j\mr{\nf}^j(\mathbf{K}^j)^\mathtt{T}\dt\right)\Big]\in\mathbb{R}^{l\times k},
    \end{split}
\end{align*}
where
\begin{align*}
    \mathbf{G}^{\vx,j}&:=\ml^{\vx,j}+(\ms^\vx\circ\ms^\vy)^j(\mr{\nf}^j)^{-1}\in\mathbb{R}^{k\times l},\\
    \mathbf{G}^{\vy,j}&:=\ml^{\vy,j}+(\ms^\vy\circ\ms^\vy)^j(\mr{\nf}^j)^{-1}\in\mathbb{R}^{l\times l}, \\
    \mathbf{H}^j&:=(\mr{\nf})^{-1}\left(\ml^{\vy,j}\mr{f}^j+\mr{f}^j(\ml^{\vy,j})^{\mathtt{T}}+(\ms^\vy\circ\ms^\vy)^j\right)\in\mathbb{R}^{l\times l},\\
    \mathbf{K}^j&:=\big((\ms^\vx\circ\ms^\vx)^{j}\big)^{-1}\mathbf{G}^{\vx,j}\in\mathbb{R}^{k\times l}.
\end{align*}

\begin{theorem}[Optimal online smoother for CGNSs \cite{andreou2024adaptive}] \label{thm:onlinesmoother}
Let $\vxt$ and $\vyt$ satisfy \eqref{eq:CGNS1}--\eqref{eq:CGNS2}. Then, under suitable regularity conditions, the discrete-time smoother distribution up to the $n$-th observation $\vx^n$ at $t_j$ is Gaussian,
\begin{equation*}
    \pp\big(\vy^j\big|\vx(s\leq n)\big) \overset{(\d)}{\sim}\mathcal{N}_l(\vm{s}^{j,n},\mr{s}^{j,n}),
\end{equation*}
where
\begin{gather*}
    \vm{\normalfont{s}}^{j,n}:=\ee{\vy^j\big|\vx(s\leq n)},\quad \mr{\normalfont{s}}^{j,n}:=\ee{(\vy^j-\vm{\normalfont{s}}^{j,n})(\vy^j-\vm{\normalfont{s}}^{j,n})^\mathtt{T}\big|\vx(s\leq n)},
\end{gather*}
for $n\in\llbracket N\rrbracket$ and $j=n$ are given by $\vm{\ns}^{n,n}=\vm{\nf}^n$ and $\mr{\ns}^{n,n}=\mr{\nf}^n$, since the smoother and filter Gaussian statistics coincide at the current end point $t_n$. Then, for $n\in\{1,\ldots,N\}$ and $j=n-1$ they are instead given by the following equations:
\begin{align*}
    \vm{\ns}^{n-1,n}&=\mathbf{E}^{n-1}\vm{\nf}^{n}+\mathbf{b}^{n-1}, \\
    \mr{\ns}^{n-1,n}&=\mathbf{E}^{n-1}\mr{\nf}^{n}(\mathbf{E}^{n-1})^\mathtt{T}+\mathbf{P}^{n-1}_n,
\end{align*}
where the $\mathbf{b}^{n-1}$ and $\mathbf{P}^{n-1}_n$ auxiliary residual terms are defined by
\begin{align*}
    \mathbf{b}^{n-1}&:=\vm{\nf}^{n-1}-\mathbf{E}^{n-1}\left((\mathbf{I}_{l\times l}+\ml^{\vy,n-1}\dt)\vm{\nf}^{n-1}+\vf^{\vy,n-1}\dt\right)\\
    &\hspace{1.55cm}+\mathbf{F}^{n-1}\left(\vx^{n}-\vx^{n-1}-(\ml^{\vx,n-1}\vm{\nf}^{n-1}+\vf^{\vx,n-1})\dt\right),\\
    \mathbf{P}^{n-1}_{n}&:= \mr{\nf}^{n-1}-\mathbf{E}^{n-1}(\mathbf{I}_{l\times l}+\ml^{\vy,n-1}\dt)\mr{\nf}^{n-1}-\mathbf{F}^{n-1}\ml^{\vx,n-1}\mr{\nf}^{n-1}\dt,
\end{align*}
while finally for $n\in\{2,\ldots,N\}$ and $j\in\llbracket n-2\rrbracket$ they are the unique solutions to the following system of recursive backward difference equations:
\begin{subequations} \label{eq:onlinerecursive}
    \begin{align}
        \vm{\ns}^{j,n}&=\vm{\ns}^{j,n-1}+\mathbf{D}^{j,n-2}\left(\vm{\ns}^{n-1,n}-\vm{\nf}^{n-1}\right), \label{eq:onlinerecursive1}\\
        \mr{\ns}^{j,n}&=\mr{\ns}^{j,n-1}+\mathbf{D}^{j,n-2}\left(\mr{\ns}^{n-1,n}-\mr{\nf}^{n-1}\right)(\mathbf{D}^{j,n-2})^\mathtt{T},\label{eq:onlinerecursive2}
    \end{align}
\end{subequations}
with the update matrices $\mathbf{D}^{j,n-2}$ being defined as
\begin{equation} \label{eq:updatematrix}
    \mathbf{D}^{n-1,n-2}:=\mathbf{I}_{l\times l} \quad\&\quad \mathbf{D}^{j,n-2}:=\overset{\mathlarger{\curvearrowright}}{\prod^{n-2}_{i=j}} \mathbf{E}^i:=\mathbf{E}^j\mathbf{E}^{j+1}\cdots \mathbf{E}^{n-2}.
\end{equation}
\end{theorem}

Theorem~\ref{thm:onlinesmoother} yields $O(\Delta t)$-accurate discrete-time Gaussian smoother statistics for each new observation $\vx^n$ at $t_n = n\Delta t$ ($n \in \llbracket N \rrbracket$) by solving \eqref{eq:onlinerecursive1}--\eqref{eq:onlinerecursive2} backward over $j \in \llbracket n \rrbracket$, where the forward operator $\mathbf{D}^{j,n-2}$ represents the discrete-time smoother Kalman gain at $t_j$ for the $n$-th observation, while $\vm{\ns}^{n-1,n} - \vm{\nf}^{n-1}$ and $\mr{\ns}^{n-1,n} - \mr{\nf}^{n-1}$ respectively denote the discrete-time innovation mean and covariance corresponding to $\vx^n$.

\subsection{Calculation of the Subjective and Objective CIRs for CGNSs}

Using the online smoother for CGNSs, for each $j\in\llbracket N\rrbracket$ we can compute both the subjective and objective CIRs over $[t_j,T]$ for the potential assimilative causal link $\vy^j\rightarrow\vx$ (with analogous results for the CCIRs). Let $p_n(\vy^j|\vx)$, for $j\in\llbracket N\rrbracket$ and $n\in\{j,\ldots,N\}$ denote the $\mathcal{N}_l(\vm{s}^{j,n},\mr{s}^{j,n})$ PDF from Theorem~\ref{thm:onlinesmoother}. At each $t_j=j\Delta t$, we quantify the information gain in the complete smoother distribution $p_N(\vy^j|\vx)$ beyond its lagged counterpart $p_n(\vy^j|\vx)$ using the signal-dispersion decomposition from \eqref{eq:signaldispersion} as:

\begin{equation} \label{eq:lackinfoupdate_CIR}
    \begin{aligned}
        \mathcal{P}^{j}_n := \mathcal{P}\big(p_N(\vy^j|\vx), p_n(\vy^j|\vx)\big)
        &= \frac{1}{2}\big(\vm{s}^{j,N}-\vm{s}^{j,n}\big)^\mathtt{T}(\mr{s}^{j,n})^{-1}\big(\vm{s}^{j,N}-\vm{s}^{j,n}\big) \\
        &\quad + \frac{1}{2}\Big(\mathrm{tr}\big(\mr{s}^{j,N}(\mr{s}^{j,n})^{-1}\big)-l-\ln\big(\det\big(\mr{s}^{j,N}(\mr{s}^{j,n})^{-1}\big)\big)\Big).
    \end{aligned}
\end{equation}

Per Section \ref{sec:CIR} and \eqref{eq:metric_CIR}, we identify $\delta(t_n;t_j) = \mathcal{P}^{j}_n$, for $j\in\llbracket N\rrbracket$ and $n\in\{j,\ldots,N\}$. This yields:

\begin{itemize}
    \item The \textit{subjective CIR length} for a threshold $\varepsilon\geq0$:
    \begin{equation} \label{eq:subjective_CIR_length_cgns}
        \overset{\sim}{\uptau}_{\vy(t_j)\rightarrow\vx}(\varepsilon) = \max_{n\in\{j,\ldots,N\}}\big\{\mathcal{P}^{j}_n > \varepsilon\big\}\Delta t - t_j \in [0,T-t_j].
    \end{equation}

    \item The \textit{objective CIR length} via integration of \eqref{eq:subjective_CIR_length_cgns} over $\varepsilon\in\big[0,\max_n\{\mathcal{P}^{j}_n\}\big]$.

    \item Its efficient approximation:
    \begin{equation} \label{eq:objective_CIR_length_efficient_cgns}
        \uptau^{\text{approx}}_{\vy(t_j)\rightarrow\vx} = \frac{\Delta t}{\max_n\{\mathcal{P}^{j}_n\}}\sum_{n=j}^N \mathcal{P}^{j}_n \in [0,T-t_j].
    \end{equation}
\end{itemize}

We maintain the conventions from Section \ref{sec:CIR}: $\overset{\sim}{\uptau}_{\vy(t_j)\rightarrow\vx}(\varepsilon)=0$ when no maximum exists, and $\uptau_{\vy(t_j)\rightarrow\vx}=\uptau^{\text{approx}}_{\vy(t_j)\rightarrow\vx}=0=\overset{\sim}{\uptau}_{\vy(t_j)\rightarrow\vx}(\varepsilon)$ when $\max_n\{\mathcal{P}^{j}_n\}=0$. For strongly intermittent systems and the objective CIR, we employ the computationally efficient scheme in \eqref{eq:objective_CIR_length_efficient_cgns} rather than the full integral definition.

In Algorithm \ref{algo:CIR} we provide a brief outline of how to calculate the subjective and objective CIR lengths of $\vy^j\rightarrow\vx$ over $[t_j,T]$ for $j\in\llbracket N\rrbracket$ in pseudocode, specifically for CGNSs. An analogous algorithm extends to CCIRs for CGNSs, as well as to the general complex system in \eqref{eq:general_CTNDS} by appropriately computing $\delta(t_n;t_j) = \mathcal{P}^{j}_n$ for each $j\in\llbracket N\rrbracket$ and $n\in\{j,\ldots,N\}$.

\begin{algorithm}
\caption{CIRs of $\vy^j\rightarrow\vx$ over $[t_j,T]$ for $j\in\llbracket N\rrbracket$ in the case of a CGNS}\label{algo:CIR}
\SetAlgoLined
\KwData{$N\in\mathbb{N}$, $0<\dt\ll1$, $\{t_j=j\dt\}_{j\in\llbracket N\rrbracket}$, Observations of $\vx$: $\{\vx^0,\ldots,\vx^N\}$, $\vm{\nf}^0$, $\mr{\nf}^0$, $\varepsilon\geq 0$}
\KwResult{$\overset{\sim}{\uptau}_{\vy(t_j)\rightarrow\vx}(\varepsilon)$ for arbitrary $\varepsilon$, $\uptau_{\vy(t_j)\rightarrow\vx}$, and $\uptau^{\text{approx}}_{\vy(t_j)\rightarrow\vx}$, all for $j\in\llbracket N\rrbracket$}
\For{$j\in\llbracket N\rrbracket$}{
    \For{$n\in\{j,\ldots,N\}$}{
    Calculate $p_n(\vy^j|\vx)$, the Gaussian PDF corresponding to $\mathcal{N}_l(\vm{s}^{j,n},\mr{s}^{j,n})$, per Theorem \ref{thm:onlinesmoother}\;
    }
    Calculate $\delta(t_n;t_j) = \mathcal{P}^{j}_n$ using \eqref{eq:lackinfoupdate_CIR} for each $n\in\{j,\ldots,N\}$\;
    \eIf{$\varepsilon< \max_{n}\{\mathcal{P}^{j}_n\}$}{$\overset{\sim}{\uptau}_{\vy(t_j)\rightarrow\vx}(\varepsilon) = \max_{n\in\{j,\ldots,N\}}\big\{\mathcal{P}^{j}_n > \varepsilon\big\}\Delta t - t_j$\;}{$\overset{\sim}{\uptau}_{\vy(t_j)\rightarrow\vx}(\varepsilon) = 0$\;}

    \eIf{$\max_{n}\{\mathcal{P}^{j}_n\}>0$}{\If{$\overset{\sim}{\uptau}_{\vy(t_j)\rightarrow\vx}(\varepsilon)$ \normalfont{is known for many} $\varepsilon\in\big[0, \max_{n}\{\mathcal{P}^{j}_n\}\big]$}{$\displaystyle\uptau_{\vy(t_j)\rightarrow\vx}=\frac{1}{\max_{n}\{\mathcal{P}^{j}_n\}}\int_0^{ \max_{n}\{\mathcal{P}^{j}_n\}} \overset{\sim}{\uptau}_{\vyt\rightarrow\vx}(\varepsilon) \d\varepsilon$\;}
    $\displaystyle\uptau^{\text{approx}}_{\vy(t_j)\rightarrow\vx} = \frac{\Delta t}{\max_n\{\mathcal{P}^{j}_n\}}\sum_{n=j}^N \mathcal{P}^{j}_n$\Comment*[l]{If $\max_{n}\{\mathcal{P}^{j}_n\}$ is essentially negligible, e.g., $\max_{n}\{\mathcal{P}^{j}_n\}<10^{-4}$, we can set $\uptau^{\text{approx}}_{\vy(t_j)\rightarrow\vx} =0$, which necessarily coincides with $\uptau_{\vy(t_j)\rightarrow\vx}\approx0$, as to avoid operationally inflated values due to numerical imprecision}
    }{$\uptau_{\vy(t_j)\rightarrow\vx}=\uptau^{\text{approx}}_{\vy(t_j)\rightarrow\vx}=0$\;}
    
}
\end{algorithm}

\section{ACI-Based Principles of Nil Causality} \label{sec:ACI_Theorems}
Before investigating complex turbulent systems, we first validate the ACI framework's theoretical foundations using the analytically tractable CGNS structure. Crucially, any reliable causal metric must satisfy two requirements: (1)~It should be asymmetric in its state components, sometimes referred to as Schreiber's principle \cite{schreiber2000measuring}, and (2)~it should correctly identify unidirectional relationships, yielding zero measure in non-causal directions; in other words, it must obey the \textit{principle of nil causality} \cite{liang2016information}, which is ``an event cannot cause another if their dynamics are independent''. Both ACI and its conditional variant naturally satisfy (1) for general complex systems. As for the principle of nil causality, for CGNSs this principle takes a rigorous form: when $\vx$'s evolution is independent of $\vy$, the ACI framework guarantees $\vy(t)\ \cancel{\rightarrow}\ \vx$ for all $t\in[0,T]$, confirming the absence of spurious causality. This extends to conditional ACI, mutatis mutandis.

\subsection{Principle of Nil Assimilative Causality for CGNSs}
\begin{theorem}[Principle of Nil Assimilative Causality for CGNSs] \label{thm:nilcausality}
    Let $\vxt$ and $\vyt$ satisfy \eqref{eq:CGNS1}--\eqref{eq:CGNS2}. When $\ml^\vx\equiv \mathbf{0}_{k\times l}$ and $(\ms^{\vy}\circ \ms^{\vx})\equiv \mathbf{0}_{l\times k}$ for every $t$ and $\vx$, then the ground-truth causal network is:

    \begin{figure}[H]
    \centering
    \resizebox{0.23\textwidth}{!}{%
    \begin{circuitikz}
    \tikzstyle{every node}=[font=\LARGE]
    \draw [ line width=2pt ] (2.5,14.75) circle (0.75cm);
    \node [font=\Huge] at (2.5,14.75) {$\vy$};
    \draw [ line width=2pt ] (7.25,14.75) circle (0.75cm);
    \node [font=\Huge] at (7.25,14.75) {$\vx$};
    \draw [line width=2pt, ->, >=Stealth] (3.25,14.75) -- (6.5,14.75);
    \draw [line width=2pt, short] (6.25,15.5) -- (3.5,14);
    \end{circuitikz}
    }%
    \end{figure}

    \noindent This is validated by the ACI framework per \eqref{eq:filter_smoother_distr} and \eqref{eq:RE_filter_smoother}--\eqref{eq:ACI_cause_notation}:
    \begin{equation*}
        \vy(t)\ \cancel{\rightarrow}\ \vx, \quad \forall\ t\in[0,T],
    \end{equation*}
    since in this case we have
    \begin{equation*}
        p_t^{\text{\ns}}(\vy|\vx)=p_t^{\text{\nf}}(\vy|\vx), \quad t\in[0,T].
    \end{equation*}
\end{theorem}
\noindent\underline{\textbf{Note on Theorem \ref{thm:nilcausality}:}} The condition $\ml^\vx\equiv \mathbf{0}_{k\times l}$ prohibits $\vy$ from entering the mean dynamics of $\vx$, while $(\ms^{\vy}\circ \ms^{\vx})\equiv \mathbf{0}_{l\times k}$ nullifies the cross-interactions between their noise feedbacks. The combination of these two assumptions removes the possibility of $\vy$ contributing to the evolution of $\vx$ in the dynamics, both through its drift coefficient, as well as via the full diffusion coefficient of the system in \eqref{eq:CGNS}. As a result: $\vy\ \cancel{\rightarrow}\ \vx$.
\begin{proof}[\underline{Proof of Theorem \ref{thm:nilcausality}}]
    Since $\ml^\vx\equiv \mathbf{0}_{k\times l}$ and $(\ms^{\vy}\circ \ms^{\vx})\equiv \mathbf{0}_{l\times k}$, by \eqref{eq:kalmangainfilter} we have that the filter Kalman gain operator of the CGNS vanishes. As a result, the filter equations for the mean and covariance matrix reduce to their model-forecast part:
    \begin{align*}
    \rmd \vmt{\nf}&=(\ml^\vy\vm{\nf}+\vf^\vy)\rmd t, \\
    \rmd \mrt{\nf}&=\big(\ml^\vy\mr{\nf}+\mr{\nf}(\ml^\vy)^\mathtt{T}+(\ms^\vy\circ\ms^\vy)\big)\rmd t.
    \end{align*}
    The equations for the filter statistics become decorrelated under this regime, meaning they can be solved independently. Similarly to its filter counterpart, the smoother Kalman gain operator of the CGNS also vanishes, meaning the smoother equations likewise become:
    \begin{align*}
        \smooth{\rmd \vm{\ns}}(t)&=-(\ml^\vy\vm{\ns}+\vf^\vy-(\ms^\vy\circ \ms^\vy)\mr{\nf}^{-1}(\vm{\nf}-\vm{\ns}))\rmd t, \\
        \smooth{\rmd \mr{\ns}}(t)&= -\big((\ml^\vy+(\ms^\vy\circ \ms^\vy)\mr{\nf}^{-1})\mr{\ns}+\mr{\ns}(\ml^\vy+(\ms^\vy\circ \ms^\vy)\mr{\nf}^{-1})^\mathtt{T}-(\ms^\vy\circ \ms^\vy)\big)\rmd t.
    \end{align*}
    Taking the difference between the filter and smoother equations we have
    \begin{equation} \label{eq:diff_filter_smoother}
        \smooth{\rmd (\vm{\ns}-\vm{\nf})}(t)=\smooth{\rmd \vm{\ns}}(t)+\rmd\vm{\nf}(t)=-(\ml^\vy+(\ms^\vy\circ \ms^\vy)\mr{\nf}^{-1})(\vm{\ns}-\vm{\nf})\rmd t, \quad T\geq t\geq 0,
    \end{equation}
    and so by $\vm{\ns}(T)=\vm{\nf}(T)$, the linearity of \eqref{eq:diff_filter_smoother}, and its uniqueness of solution, we retrieve
    \begin{equation*}
        \vm{\ns}(t)=\vm{\nf}(t), \quad t\in[0,T].
    \end{equation*}
    As an immediate consequence, we have
    \begin{align*}
        \mrt{\ns}&=\ee{(\vyt-\vmt{\ns})(\vyt-\vmt{\ns})^\mathtt{T}\big|\vx(s\leq T)}\\
        &=\ee{(\vyt-\vmt{\nf})(\vyt-\vmt{\nf})^\mathtt{T}\big|\vx(s\leq T)}=\mrt{\nf},\quad t\in[0,T],
    \end{align*}
    where in the last equality we have used the stability property of conditional expectations, $\cF_t^\vx\subseteq\cF_T^\vx$, and the $\cF_t^\vx$-measurability of $\mrt{\nf}$, with $\cF_t^\vx$ denoting the $\sigma$-algebra generated by $\{\vx(s)\}_{s\leq t}$ for $t\in[0,T]$ \cite{liptser2001statistics}. Combining these, we end up with:
    \begin{equation*}
        p_t^{\text{\ns}}(\vy|\vx)=p_t^{\text{\nf}}(\vy|\vx), \quad t\in[0,T],
    \end{equation*}
    which by \eqref{eq:RE_filter_smoother} yields $\vy(t)\ \cancel{\rightarrow}\ \vx$ for each $t\in[0,T]$.
\end{proof}

\subsection{Principle of Nil Conditional Assimilative Causality for CGNSs}
The following theorem extends Theorem~\ref{thm:nilcausality} to establish the principle of nil \textit{conditional} assimilative causality for CGNSs. To formulate this result, we first introduce some necessary notation. Following Section~\ref{sec:ACI_Generalization}, we consider the state variables $(\vx_{\text{A}},\vx_{\text{B}},\vy)$ and reformulate the CGNS in \eqref{eq:general_CTNDS} as:
\begin{subequations} \label{eq:CGNS_cond}
    \begin{align}
        \d\vx_{\text{\normalfont{A}}}(t) &= \big(\ml^{\vx_{\text{\normalfont{A}}}}(t,\vx)\vyt+\vf^{\vx_{\text{\normalfont{A}}}}(t,\vx)\big)\d t+\ms_1^{\vx_{\text{\normalfont{A}}}}(t,\vx) \rmd\vw_1(t)+\ms_2^{\vx_{\text{\normalfont{A}}}}(t,\vx)\rmd\vw_2(t), \label{eq:CGNS_cond1} \\
        \d\vx_{\text{\normalfont{B}}}(t) &= \big(\ml^{\vx_{\text{\normalfont{B}}}}(t,\vx)\vyt+\vf^{\vx_{\text{\normalfont{B}}}}(t,\vx)\big)\d t+\ms_1^{\vx_{\text{\normalfont{B}}}}(t,\vx) \rmd\vw_1(t)+\ms_2^{\vx_{\text{\normalfont{B}}}}(t,\vx)\rmd\vw_2(t), \label{eq:CGNS_cond2} \\
        \d\vyt &= \big(\ml^\vy(t,\vx)\vyt+\vf^\vy(t,\vx)\big)\d t+\ms_1^\vy(t,\vx)\rmd\vw_1(t)+\ms_2^\vy(t,\vx)\rmd\vw_2(t), \label{eq:CGNS_cond3}
    \end{align}
\end{subequations}
where by using block-matrix notation we have:
\begin{gather*}
    \ml^\vx=\begin{pmatrix}
        \ml^{\vx_{\text{\normalfont{A}}}}\\ \ml^{\vx_{\text{\normalfont{B}}}}
    \end{pmatrix}, \quad \vf^\vx=\begin{pmatrix}
        \vf^{\vx_{\text{\normalfont{A}}}}\\ \vf^{\vx_{\text{\normalfont{B}}}}
    \end{pmatrix}, \quad \ms_m^\vx=\begin{pmatrix}
        \ms_m^{\vx_{\text{\normalfont{A}}}}\\ \ms_m^{\vx_{\text{\normalfont{B}}}}
    \end{pmatrix}, \ m=1,2,\\
    \ml^{\vx_{\square}}\in\mathbb{R}^{k_{\square}\times l},\quad \vf^{\vx_{\square}}\in\mathbb{R}^{k_{\square}}, \quad \ms_m^{\vx_{\square}}\in\mathbb{R}^{k_{\square}\times d_m}, \quad m=1,2, \ \square\in\{\text{\normalfont{A}},\text{\normalfont{B}}\}.
\end{gather*}
Under this formulation, $(\ms^\vx\circ\ms^\vx)$ can be depicted as a $2\times 2$ block matrix:
\begin{equation*}
    (\ms^\vx\circ\ms^\vx)=\ms_1^\vx(\ms_1^\vx)^\mathtt{T}+\ms_2^\vx(\ms_2^\vx)^\mathtt{T}=\begin{pmatrix}
        (\ms^{\vx_{\text{\normalfont{A}}}}\circ\ms^{\vx_{\text{\normalfont{A}}}}) & (\ms^{\vx_{\text{\normalfont{A}}}}\circ\ms^{\vx_{\text{\normalfont{B}}}}) \\
        (\ms^{\vx_{\text{\normalfont{B}}}}\circ\ms^{\vx_{\text{\normalfont{A}}}}) & (\ms^{\vx_{\text{\normalfont{B}}}}\circ\ms^{\vx_{\text{\normalfont{B}}}})
    \end{pmatrix}.
\end{equation*}
We also define the Schur complement of the $(\ms^{\vx_{\text{\normalfont{A}}}}\circ\ms^{\vx_{\text{\normalfont{A}}}})$ and $(\ms^{\vx_{\text{\normalfont{B}}}}\circ\ms^{\vx_{\text{\normalfont{B}}}})$ blocks with respect to $(\ms^\vx\circ\ms^\vx)$ \cite{gallier2019notes}:
\begin{equation}
    \begin{gathered}
        (\ms^\vx\circ\ms^\vx) / (\ms^{\vx_{\text{\normalfont{A}}}}\circ\ms^{\vx_{\text{\normalfont{A}}}}):= (\ms^{\vx_{\text{\normalfont{B}}}}\circ\ms^{\vx_{\text{\normalfont{B}}}}) - (\ms^{\vx_{\text{\normalfont{B}}}}\circ\ms^{\vx_{\text{\normalfont{A}}}})(\ms^{\vx_{\text{\normalfont{A}}}}\circ\ms^{\vx_{\text{\normalfont{A}}}})^{-1}(\ms^{\vx_{\text{\normalfont{A}}}}\circ\ms^{\vx_{\text{\normalfont{B}}}}),\\
        (\ms^\vx\circ\ms^\vx) / (\ms^{\vx_{\text{\normalfont{B}}}}\circ\ms^{\vx_{\text{\normalfont{B}}}}):= (\ms^{\vx_{\text{\normalfont{A}}}}\circ\ms^{\vx_{\text{\normalfont{A}}}}) - (\ms^{\vx_{\text{\normalfont{A}}}}\circ\ms^{\vx_{\text{\normalfont{B}}}})(\ms^{\vx_{\text{\normalfont{B}}}}\circ\ms^{\vx_{\text{\normalfont{B}}}})^{-1}(\ms^{\vx_{\text{\normalfont{B}}}}\circ\ms^{\vx_{\text{\normalfont{A}}}}),
    \end{gathered} \label{eq:schur_complements}
\end{equation}
under the tacit assumption that $(\ms^{\vx_{\text{\normalfont{A}}}}\circ\ms^{\vx_{\text{\normalfont{A}}}})$ and $(\ms^{\vx_{\text{\normalfont{B}}}}\circ\ms^{\vx_{\text{\normalfont{B}}}})$ are invertible; $(\ms^\vx\circ\ms^\vx)$ is positive definite if and only if either one of $(\ms^{\vx_{\text{\normalfont{A}}}}\circ\ms^{\vx_{\text{\normalfont{A}}}})$ or $(\ms^{\vx_{\text{\normalfont{B}}}}\circ\ms^{\vx_{\text{\normalfont{B}}}})$ and its associated Schur complement are positive definite matrices \cite{gallier2019notes}. In the case where both $(\ms^{\vx_{\text{\normalfont{A}}}}\circ\ms^{\vx_{\text{\normalfont{A}}}})$ and $(\ms^{\vx_{\text{\normalfont{B}}}}\circ\ms^{\vx_{\text{\normalfont{B}}}})$ are positive definite, then $(\ms^\vx\circ\ms^\vx)^{-1}$ enjoys the following explicit representation:
\begin{equation}
   \begin{aligned}
       (\ms^\vx\circ\ms^\vx)^{-1}&=\begin{pmatrix}
           \big((\ms^\vx\circ\ms^\vx) / (\ms^{\vx_{\text{\normalfont{B}}}}\circ\ms^{\vx_{\text{\normalfont{B}}}})\big)^{-1} & \mathbf{0}_{k_{\text{\normalfont{A}}} \times k_{\text{\normalfont{B}}}} \\
           \mathbf{0}_{k_{\text{\normalfont{B}}} \times k_{\text{\normalfont{A}}}} & \big((\ms^\vx\circ\ms^\vx) / (\ms^{\vx_{\text{\normalfont{A}}}}\circ\ms^{\vx_{\text{\normalfont{A}}}})\big)^{-1}
       \end{pmatrix}\\
       &\hspace*{2cm}\times \begin{pmatrix}
           \mathbf{I}_{k_{\text{\normalfont{A}}}\times k_{\text{\normalfont{A}}}} & -(\ms^{\vx_{\text{\normalfont{A}}}}\circ\ms^{\vx_{\text{\normalfont{B}}}})(\ms^{\vx_{\text{\normalfont{B}}}}\circ\ms^{\vx_{\text{\normalfont{B}}}})^{-1} \\
           -(\ms^{\vx_{\text{\normalfont{B}}}}\circ\ms^{\vx_{\text{\normalfont{A}}}})(\ms^{\vx_{\text{\normalfont{A}}}}\circ\ms^{\vx_{\text{\normalfont{A}}}})^{-1} & \mathbf{I}_{k_{\text{\normalfont{B}}}\times k_{\text{\normalfont{B}}}}
       \end{pmatrix}.
   \end{aligned} \label{eq:obs_gramian_explicit}
\end{equation}

With these preliminaries, we now establish the principle of nil conditional assimilative causality for CGNSs: when (i)~$\vx_{\text{A}}$'s evolution is independent of $\vy$, and (ii)~the uncertainty levels of $\vx_{\text{A}}$ and $\vx_{\text{B}}$ are non-interacting, the generalized ACI framework correctly identifies the absence of conditional causation. Formally, $\big(\vy(t)\ \cancel{\rightarrow}\ \vx_{\text{\normalfont{A}}}\big)\big|\vx_{\text{\normalfont{B}}}$, for all $t\in[0,T]$.

\begin{theorem}[Principle of Nil Conditional Assimilative Causality for CGNSs] \label{thm:nilcondcausality}
    Let $\vxt$ and $\vyt$ satisfy \eqref{eq:CGNS1}--\eqref{eq:CGNS2}. When $\ml^{\vx_{\text{\normalfont{A}}}}\equiv \mathbf{0}_{k_{\text{\normalfont{A}}}\times l}$, $(\ms^{\vy}\circ \ms^{\vx_{\text{\normalfont{A}}}})\equiv \mathbf{0}_{l\times k_{\text{\normalfont{B}}}}$, and $(\ms^{\vx_{\text{\normalfont{A}}}}\circ \ms^{\vx_{\text{\normalfont{B}}}})\equiv \mathbf{0}_{k_{\text{\normalfont{A}}}\times k_{\text{\normalfont{B}}}}$ for every $t$ and $\vx$, then the ground-truth causal network is:

    \begin{figure}[H]
    \centering
    \resizebox{0.4\textwidth}{!}{%
    \begin{circuitikz}
    \tikzstyle{every node}=[font=\Huge]
    \draw [ line width=2pt ] (2.5,14.75) circle (0.75cm);
    \node [font=\Huge] at (2.5,14.75) {$\vy$};
    \draw [ line width=2pt ] (7.25,14.75) circle (0.75cm);
    \node [font=\Huge] at (7.25,14.75) {$\vx_{\text{\normalfont{B}}}$};
    \draw [line width=2pt, <->, >=Stealth] (3.25,14.75) -- (6.5,14.75);
    \draw [ line width=2pt ] (12,14.75) circle (0.75cm);
    \node [font=\Huge] at (12,14.75) {$\vx_{\text{\normalfont{A}}}$};
    \draw [line width=2pt, <->, >=Stealth] (8,14.75) -- (11.25,14.75);
    \draw [line width=2pt, ->, >=Stealth] (2.5,15.5) .. controls (4,16.5) and (10.5,16.75) .. (12,15.5) ;
    \draw [line width=2pt, short] (9.25,16.75) -- (5,15.75);
    \end{circuitikz}
    }%
    \end{figure}

    \noindent This is validated by the ACI framework per \eqref{eq:filter_smoother_ancillary_inf_uncert}--\eqref{eq:ACI_cause_cond_notation}:
    \begin{equation*}
        \big(\vy(t)\ \cancel{\rightarrow}\ \vx_{\text{\normalfont{A}}}\big)\big|\vx_{\text{\normalfont{B}}}, \quad \forall\ t\in[0,T],
    \end{equation*}
    since in this case we have
    \begin{equation*}
        p_t^{\text{\ns}|\vx_{\text{\normalfont{B}}}}(\vy|\vx_{\text{\normalfont{A}}})=p_t^{\text{\nf}|\vx_{\text{\normalfont{B}}}}(\vy|\vx_{\text{\normalfont{A}}}), \quad t\in[0,T].
    \end{equation*}
\end{theorem}
\noindent\underline{\textbf{Note on Theorem \ref{thm:nilcondcausality}:}}
The assumption $\ml^{\vx_{\text{A}}} \equiv \mathbf{0}_{k_{\text{A}} \times l}$ eliminates direct influence of $\vy$ on $\vx_{\text{A}}$'s mean dynamics, while $(\ms^{\vy} \circ \ms^{\vx_{\text{A}}}) \equiv \mathbf{0}_{l \times k_{\text{B}}}$ removes their noise feedback coupling. However, these conditions alone are insufficient to prevent indirect $\vy$-$\vx_{\text{A}}$ interactions through $\vx_{\text{B}}$ via observable noise cross-correlations. We therefore additionally impose $(\ms^{\vx_{\text{A}}} \circ \ms^{\vx_{\text{B}}}) \equiv \mathbf{0}_{k_{\text{A}} \times k_{\text{B}}}$, rendering $(\ms^\vx \circ \ms^\vx)$ block-diagonal. Collectively, these ensure:
$\big(\vy\ \cancel{\longrightarrow}\ \vx_{\text{\normalfont{A}}}\big)\big|\vx_{\text{\normalfont{B}}}$.

\begin{proof}[\underline{Proof of Theorem \ref{thm:nilcondcausality}}]
    For CGNS, the filter and smoother means evolve according to linear random ODEs, while their covariance matrices satisfy forward Riccati and backward symmetric Sylvester random ODEs respectively \cite{kandil2003matrix}. Under the regularity conditions of Theorems~\ref{thm:filtercgns} and~\ref{thm:smoothercgns}, these Gaussian statistics exhibit continuous dependence on both the model parameters in \eqref{eq:CGNS} and their initial/terminal conditions \cite{han2017random,neckel2013random}. This continuity allows direct computation of the posterior PDFs in \eqref{eq:filter_smoother_ancillary_inf_uncert} through the limit $\mathrm{Var}(\vx_{\text{B}}(t))\to+\infty$ applied to \eqref{eq:filter} and \eqref{eq:smoother}. Crucially, Gaussianity is preserved under this limit since a normal distribution depends continuously on its mean and covariance. Specifically, $p^{\text{f}|\vx_{\text{B}}}_t(\vy|\vx_{\text{A}})$ remains Gaussian $\mathcal{N}_l\big(\boldsymbol{\mu}_{\text{f}|\vx_{\text{B}}}(t),\mathbf{R}_{\text{f}|\vx_{\text{B}}}(t)\big)$, where:
    \begin{subequations} \label{eq:ancillary_filter_cgns}
        \begin{align}
            \rmd\boldsymbol{\mu}_{\text{\nf}|\vx_{\text{\normalfont{B}}}}(t)&=(\ml^\vy\boldsymbol{\mu}_{\text{\nf}|\vx_{\text{\normalfont{B}}}}+\vf^\vy)\rmd t+\mathbf{K}_{\text{\nf}|\vx_{\text{\normalfont{B}}}}(\rmd \vx -(\ml^\vx\boldsymbol{\mu}_{\text{\nf}|\vx_{\text{\normalfont{B}}}}+\vf^\vx)\rmd t), \label{eq:ancillary_filter_cgns1}\\
            \rmd\mathbf{R}_{\text{\nf}|\vx_{\text{\normalfont{B}}}}(t)&=\big(\ml^\vy\mathbf{R}_{\text{\nf}|\vx_{\text{\normalfont{B}}}}+\mathbf{R}_{\text{\nf}|\vx_{\text{\normalfont{B}}}}(\ml^\vy)^\mathtt{T}+(\ms^\vy\circ\ms^\vy)-\mathbf{K}_{\text{\nf}|\vx_{\text{\normalfont{B}}}}(\ms^\vx\circ\ms^\vx)\mathbf{K}_{\text{\nf}|\vx_{\text{\normalfont{B}}}}^\mathtt{T}\big)\rmd t, \label{eq:ancillary_filter_cgns2}
        \end{align}
    \end{subequations}
    with
    \begin{equation*}
        \mathbf{K}_{\text{\nf}|\vx_{\text{\normalfont{B}}}}(t,\vx):=\lim_{\mathrm{Var}(\vx_{\text{\normalfont{B}}}(t))\to+\infty}\mathbf{K}_{\text{\nf}}(t,\vx),
    \end{equation*}
    while
    $p^{\text{s}|\vx_{\text{\normalfont{B}}}}_t(\vy|\vx_{\text{\normalfont{A}}})$ is a Gaussian density corresponding to $\mathcal{N}_l\big(\boldsymbol{\mu}_{\text{\ns}|\vx_{\text{\normalfont{B}}}}(t),\mathbf{R}_{\text{\ns}|\vx_{\text{\normalfont{B}}}}(t)\big)$, where
    \begin{subequations} \label{eq:ancillary_smoother_cgns}
        \begin{align}
            \smooth{\rmd\boldsymbol{\mu}_{\text{\ns}|\vx_{\text{\normalfont{B}}}}}(t)&=-\big(\ml^\vy\boldsymbol{\mu}_{\text{\ns}|\vx_{\text{\normalfont{B}}}}+\vf^\vy-\mathbf{B}_{|\vx_{\text{\normalfont{B}}}}\mathbf{R}_{\text{\nf}|\vx_{\text{\normalfont{B}}}}^{-1}(\boldsymbol{\mu}_{\text{\nf}|\vx_{\text{\normalfont{B}}}}-\boldsymbol{\mu}_{\text{\ns}|\vx_{\text{\normalfont{B}}}})\big)\rmd t\nonumber\\
            &\hspace*{5cm}+\mathbf{K}_{\text{\ns}|\vx_{\text{\normalfont{B}}}}\big(\smooth{\rmd\vx}+(\ml^\vx\boldsymbol{\mu}_{\text{\ns}|\vx_{\text{\normalfont{B}}}}+\vf^\vx)\rmd t\big), \label{eq:ancillary_smoother_cgns1}\\
            \smooth{\rmd\mathbf{R}_{\text{\ns}|\vx_{\text{\normalfont{B}}}}}(t)&=-\big((\mathbf{A}_{|\vx_{\text{\normalfont{B}}}}+\mathbf{B}_{|\vx_{\text{\normalfont{B}}}}\mathbf{R}_{\text{\nf}|\vx_{\text{\normalfont{B}}}}^{-1})\mathbf{R}_{\text{\ns}|\vx_{\text{\normalfont{B}}}}+\mathbf{R}_{\text{\ns}|\vx_{\text{\normalfont{B}}}}(\mathbf{A}_{|\vx_{\text{\normalfont{B}}}}+\mathbf{B}_{|\vx_{\text{\normalfont{B}}}}\mathbf{R}_{\text{\nf}|\vx_{\text{\normalfont{B}}}}^{-1})^\mathtt{T}\nonumber\\
            &\hspace*{5cm}-(\ms^\vy\circ \ms^\vy)+\mathbf{K}_{\text{\ns}|\vx_{\text{\normalfont{B}}}}(\ms^\vx\circ\ms^\vx)\mathbf{K}_{\text{\ns}|\vx_{\text{\normalfont{B}}}}^\mathtt{T}\big)\rmd t, \label{eq:ancillary_smoother_cgns2}
        \end{align}
    \end{subequations}
    with
    \begin{gather*}
        \mathbf{K}_{\text{\ns}|\vx_{\text{\normalfont{B}}}}(t,\vx):=\lim_{\mathrm{Var}(\vx_{\text{\normalfont{B}}}(t))\to+\infty}\mathbf{K}_{\text{\ns}}(t,\vx),\\
        \mathbf{A}_{|\vx_{\text{\normalfont{B}}}}(t,\vx):=\lim_{\mathrm{Var}(\vx_{\text{\normalfont{B}}}(t))\to+\infty}\ma(t,\vx),\\
        \mathbf{B}_{|\vx_{\text{\normalfont{B}}}}(t,\vx):=\lim_{\mathrm{Var}(\vx_{\text{\normalfont{B}}}(t))\to+\infty}\mb(t,\vx).
    \end{gather*}
    As the measurability of the posterior Gaussian statistics remains unaffected under the formal limit $\mathrm{Var}(\vx_{\text{\normalfont{B}}}(t))\to+\infty$, we just need to prove
    \begin{equation*}
        \boldsymbol{\mu}_{\text{\ns}|\vx_{\text{\normalfont{B}}}}(t)=\boldsymbol{\mu}_{\text{\nf}|\vx_{\text{\normalfont{B}}}}(t), \quad t\in[0,T],
    \end{equation*}
    as this immediately yields $\mathbf{R}_{\text{\ns}|\vx_{\text{\normalfont{B}}}}\equiv\mathbf{R}_{\text{\nf}|\vx_{\text{\normalfont{B}}}}$, similarly to the proof of Theorem \ref{thm:nilcausality}.

    Following a similar procedure to the proof of Theorem \ref{thm:nilcausality}, we first determine how the filter and smoother Kalman gain operators reduce subject to the assumptions of this theorem. Under this regime, by using block-matrix algebra and \eqref{eq:schur_complements}--\eqref{eq:obs_gramian_explicit}, we have from \eqref{eq:kalmangainfilter} that
    \begin{align*}
        \mathbf{K}_{\text{\nf}}&=\big((\ms^\vy\circ \ms^\vx)+\mr{\nf}(\ml^\vx)^\mathtt{T}\big)(\ms^\vx\circ\ms^\vx)^{-1}\\
        &=\begin{pmatrix}
            \mathbf{0}_{l\times k_{\text{\normalfont{A}}}} & (\ms^\vy\circ \ms^{\vx_{\text{\normalfont{B}}}})+\mr{\nf}(\ml^{\vx_{\text{\normalfont{B}}}})^\mathtt{T}
        \end{pmatrix}\\
        &\hspace*{1cm}\times\begin{pmatrix}
           \big((\ms^\vx\circ\ms^\vx) / (\ms^{\vx_{\text{\normalfont{B}}}}\circ\ms^{\vx_{\text{\normalfont{B}}}})\big)^{-1} & \mathbf{0}_{k_{\text{\normalfont{A}}} \times k_{\text{\normalfont{B}}}} \\
           \mathbf{0}_{k_{\text{\normalfont{B}}} \times k_{\text{\normalfont{A}}}} & \big((\ms^\vx\circ\ms^\vx) / (\ms^{\vx_{\text{\normalfont{A}}}}\circ\ms^{\vx_{\text{\normalfont{A}}}})\big)^{-1}
       \end{pmatrix}\\
       &=\begin{pmatrix}
            \mathbf{0}_{l\times k_{\text{\normalfont{A}}}} & \big((\ms^\vy\circ \ms^{\vx_{\text{\normalfont{B}}}})+\mr{\nf}(\ml^{\vx_{\text{\normalfont{B}}}})^\mathtt{T}\big)\big((\ms^\vx\circ\ms^\vx) / (\ms^{\vx_{\text{\normalfont{A}}}}\circ\ms^{\vx_{\text{\normalfont{A}}}})\big)^{-1}
        \end{pmatrix}\\
        &=\begin{pmatrix}
            \mathbf{0}_{l\times k_{\text{\normalfont{A}}}} & \big((\ms^\vy\circ \ms^{\vx_{\text{\normalfont{B}}}})+\mr{\nf}(\ml^{\vx_{\text{\normalfont{B}}}})^\mathtt{T}\big)(\ms^{\vx_{\text{\normalfont{B}}}}\circ\ms^{\vx_{\text{\normalfont{B}}}})^{-1}
        \end{pmatrix},
    \end{align*}
    since under the conditions of this theorem we simply have
    \begin{equation*}
        (\ms^\vx\circ\ms^\vx)^{-1}=\begin{pmatrix}
           (\ms^{\vx_{\text{\normalfont{A}}}}\circ\ms^{\vx_{\text{\normalfont{A}}}})^{-1} & \mathbf{0}_{k_{\text{\normalfont{A}}} \times k_{\text{\normalfont{B}}}} \\
           \mathbf{0}_{k_{\text{\normalfont{B}}} \times k_{\text{\normalfont{A}}}} & (\ms^{\vx_{\text{\normalfont{B}}}}\circ\ms^{\vx_{\text{\normalfont{B}}}})^{-1}
       \end{pmatrix}.
    \end{equation*}
    As such, by interpreting $\mathrm{Var}(\vx_{\text{\normalfont{B}}}(t))\to+\infty$ as to mean $(\ms^{\vx_{\text{\normalfont{B}}}}\circ\ms^{\vx_{\text{\normalfont{B}}}})^{-1}\to\mathbf{0}_{k_{\text{\normalfont{B}}}\times k_{\text{\normalfont{B}}}}$ for each $t$ and $\vx$ in this regime, by the result we just established we see that
    \begin{equation*}
        \mathbf{K}_{\text{\nf}|\vx_{\text{\normalfont{B}}}}(t,\vx)=\lim_{\mathrm{Var}(\vx_{\text{\normalfont{B}}}(t))\to+\infty}\mathbf{K}_{\text{\nf}}(t,\vx)=\mathbf{0}_{l\times k}.
    \end{equation*}
    Analogously, we can see that for the smoother Kalman gain operator we have from \eqref{eq:kalmangainsmoother}:
    \begin{align*}
        \mathbf{K}_{\text{\ns}}(t,\vx)&=(\ms^\vy\circ \ms^\vx)(\ms^\vx\circ\ms^\vx)^{-1}\\
        &=\begin{pmatrix}
            \mathbf{0}_{l\times k_{\text{\normalfont{A}}}} & (\ms^\vy\circ \ms^{\vx_{\text{\normalfont{B}}}})
        \end{pmatrix}\begin{pmatrix}
           (\ms^{\vx_{\text{\normalfont{A}}}}\circ\ms^{\vx_{\text{\normalfont{A}}}})^{-1} & \mathbf{0}_{k_{\text{\normalfont{A}}} \times k_{\text{\normalfont{B}}}} \\
           \mathbf{0}_{k_{\text{\normalfont{B}}} \times k_{\text{\normalfont{A}}}} & (\ms^{\vx_{\text{\normalfont{B}}}}\circ\ms^{\vx_{\text{\normalfont{B}}}})^{-1}
       \end{pmatrix}\\
       &=\begin{pmatrix}
            \mathbf{0}_{l\times k_{\text{\normalfont{A}}}} & (\ms^\vy\circ \ms^{\vx_{\text{\normalfont{B}}}})(\ms^{\vx_{\text{\normalfont{B}}}}\circ\ms^{\vx_{\text{\normalfont{B}}}})^{-1}
        \end{pmatrix},
    \end{align*}
    and so, in the same vain as in the filter-based result, we end up with
    \begin{equation*}
        \mathbf{K}_{\text{\ns}|\vx_{\text{\normalfont{B}}}}(t,\vx)=\lim_{\mathrm{Var}(\vx_{\text{\normalfont{B}}}(t))\to+\infty}\mathbf{K}_{\text{\ns}}(t,\vx)=\mathbf{0}_{l\times k}.
    \end{equation*}
    Applying now these results to \eqref{eq:ancillary_filter_cgns1} and \eqref{eq:ancillary_smoother_cgns1}, we retrieve
    \begin{align*}
        \rmd\boldsymbol{\mu}_{\text{\nf}|\vx_{\text{\normalfont{B}}}}(t)&=(\ml^\vy\boldsymbol{\mu}_{\text{\nf}|\vx_{\text{\normalfont{B}}}}+\vf^\vy)\rmd t,\\
        \smooth{\rmd\boldsymbol{\mu}_{\text{\ns}|\vx_{\text{\normalfont{B}}}}}(t)&=-\big(\ml^\vy\boldsymbol{\mu}_{\text{\ns}|\vx_{\text{\normalfont{B}}}}+\vf^\vy-\mathbf{B}_{|\vx_{\text{\normalfont{B}}}}\mathbf{R}_{\text{\nf}|\vx_{\text{\normalfont{B}}}}^{-1}(\boldsymbol{\mu}_{\text{\nf}|\vx_{\text{\normalfont{B}}}}-\boldsymbol{\mu}_{\text{\ns}|\vx_{\text{\normalfont{B}}}})\big)\rmd t.
    \end{align*}
    Then, as in the proof of Theorem \ref{thm:nilcausality}, we have
    \begin{equation} \label{eq:diff_filter_smoother_cond}
        \smooth{\rmd (\boldsymbol{\mu}_{\text{\ns}|\vx_{\text{\normalfont{B}}}}-\boldsymbol{\mu}_{\text{\nf}|\vx_{\text{\normalfont{B}}}})}(t)=-(\ml^\vy+\mathbf{B}_{|\vx_{\text{\normalfont{B}}}}\mathbf{R}_{\text{\nf}|\vx_{\text{\normalfont{B}}}}^{-1})(\boldsymbol{\mu}_{\text{\ns}|\vx_{\text{\normalfont{B}}}}-\boldsymbol{\mu}_{\text{\nf}|\vx_{\text{\normalfont{B}}}})\rmd t, \quad T\geq t\geq 0,
    \end{equation}
    and since $\vm{\ns}(T)=\vm{\nf}(T)$, which by the continuous dependence of the filter and smoother means on their initial and terminal conditions, respectively, translates to
    \begin{equation*}
        \boldsymbol{\mu}_{\text{\ns}|\vx_{\text{\normalfont{B}}}}(T)=\boldsymbol{\mu}_{\text{\nf}|\vx_{\text{\normalfont{B}}}}(T),
    \end{equation*}
    then due to this, the linearity of \eqref{eq:diff_filter_smoother_cond}, and its uniqueness of solution, we recover
    \begin{equation*}
        \boldsymbol{\mu}_{\text{\ns}|\vx_{\text{\normalfont{B}}}}(t)=\boldsymbol{\mu}_{\text{\nf}|\vx_{\text{\normalfont{B}}}}(t), \quad t\in[0,T].
    \end{equation*}
    As already mentioned, this has the immediate consequence that
    \begin{equation*}
        \mathbf{R}_{\text{\ns}|\vx_{\text{\normalfont{B}}}}(t)=\mathbf{R}_{\text{\nf}|\vx_{\text{\normalfont{B}}}}(t), \quad t\in[0,T],
    \end{equation*}
    which yields the desired result of
    \begin{equation*}
        p_t^{\text{\ns}|\vx_{\text{\normalfont{B}}}}(\vy|\vx_{\text{\normalfont{A}}})=p_t^{\text{\nf}|\vx_{\text{\normalfont{B}}}}(\vy|\vx_{\text{\normalfont{A}}}), \quad t\in[0,T],
    \end{equation*}
    which in turn, by \eqref{eq:RE_filter_smoother_general}, establishes $\big(\vy(t)\ \cancel{\rightarrow}\ \vx_{\text{\normalfont{A}}}\big)\big|\vx_{\text{\normalfont{B}}}$ for each $t\in[0,T]$.
\end{proof}

\begin{rem}[Analytical Meaning of $\mathrm{Var}(\vx_{\text{\normalfont{B}}}(t))\to+\infty$]
When $(\ms^{\vx_{\text{A}}} \circ \ms^{\vx_{\text{B}}}) \equiv \mathbf{0}_{k_{\text{A}} \times k_{\text{B}}}$, the interpretation of ``assigning infinite uncertainty to $\vx_{\text{B}}$'s marginal likelihood'' (Section~\ref{sec:ACI_Generalization}) becomes straightforward. In this case, $\mathrm{Var}(\vx_{\text{B}}(t)) \to +\infty$ simply corresponds to $(\ms^{\vx_{\text{B}}} \circ \ms^{\vx_{\text{B}}})^{-1} \to \mathbf{0}_{k_{\text{B}} \times k_{\text{B}}}$ for all $t$ and $\vx$, since $(\ms^\vx \circ \ms^\vx)$ becomes block-diagonal. For general turbulent systems or when this condition fails, the interpretation of this limiting procedure requires more care. Here, $\mathrm{Var}(\vx_{\text{B}}(t)) \to +\infty$ demands rigorous analysis of the structure of $\vx_{\text{B}}$'s marginal likelihood and its impact on the Kalman gain operators. Only through such analysis can we properly nullify $\vx_{\text{B}}$'s influence when testing for conditional assimilative causal links $\big(\vy(t) \rightarrow \vx_{\text{A}}\big) \big| \vx_{\text{B}}$. We defer this detailed investigation to future work.
\end{rem}

\begin{rem}[Special Case: ACI Framework and Reduced CGNS Dynamics] \label{rem:reduceddynamics}
    When $(\ms^{\vx_{\text{A}}} \circ \ms^{\vx_{\text{B}}}) \equiv \mathbf{0}_{k_{\text{A}} \times k_{\text{B}}}$ in a CGNS, implementing the condition within the generalized ACI framework $(\ms^{\vx_{\text{B}}} \circ \ms^{\vx_{\text{B}}})^{-1} \to \mathbf{0}_{k_{\text{B}} \times k_{\text{B}}}$ (Section~\ref{sec:ACI_Generalization}) leads to an equivalent reduced system. In this reduced system, the state estimation of $\vyt$ becomes governed by the CGNS defined in \eqref{eq:CGNS_cond1} and \eqref{eq:CGNS_cond3}, where $(\vx_{\text{A}},\vy)$ form the state variables, with $\vx_{\text{B}}$ being reduced to a deterministic forcing term defined by its observed values (analogous to a control or input term).
\end{rem}

\section{Numerical Studies: ACI Performance in Nonlinear Systems with Intermittency, Regime Switching and Extreme Events}\label{Sec:Numerics}
This section includes three numerical studies that exploit the ACI framework to study complex dynamical systems with intermittency, non-Gaussian features, regime switching, and extreme events. All test cases employ CGNS models, leveraging their analytical tractability while capturing these nonlinear phenomena. Additionally, unless otherwise specified, the efficient approximation in \eqref{eq:objective_CIR_length_efficient} is used to compute the objective (C)CIR lengths.

\subsection{A Nonlinear Dyad Model with Intermittent Extreme Events}\label{Sec:Dyad}
For completeness, we restate the nonlinear dyad model described in the main text (in differential form):
\begin{subequations}\label{eq:Dyad_model_SI}
\begin{align}
\d x &= (-d_x x + \gamma xy + f_x)\d t + \sigma_x\rmd W_x\label{eq:Dyad_model_x_SI}\\
\d y &= (-d_y y - \gamma x^2 + f_y)\d t + \sigma_y \rmd W_y.\label{eq:Dyad_model_y_SI}
\end{align}
\end{subequations}
This is a reduced-order conceptual model for atmospheric variability. It has been used to analyze the effects of various coarse-grained procedures on processes exhibiting intermittency, large-scale bifurcations, and microscale phase transitions. It is defined by an energy-conserving condition on its quadratic nonlinearities \cite{majda2012physics}. The coupling parameter $\gamma > 0$ plays a crucial role by ensuring significant positive $y$ values ($y>d_x/\gamma$) trigger extreme events in $x$.

Supplementary Figure \ref{Fig:dyad_interaction_fig_1} presents the data assimilation results for the dyad model in \eqref{eq:Dyad_model_SI}. The key distinction between the filter and smoother distributions in estimating $y$ occurs prior to extreme events in $x$. This behavior is expected since the filter, operating without knowledge of future observations, cannot fully anticipate the triggering mechanism in $y$. In contrast, the smoother benefits from future extreme event data, enabling more accurate state estimation of y with reduced uncertainty.

\begin{figure}[!ht]%
\centering
\includegraphics[width=1\textwidth]{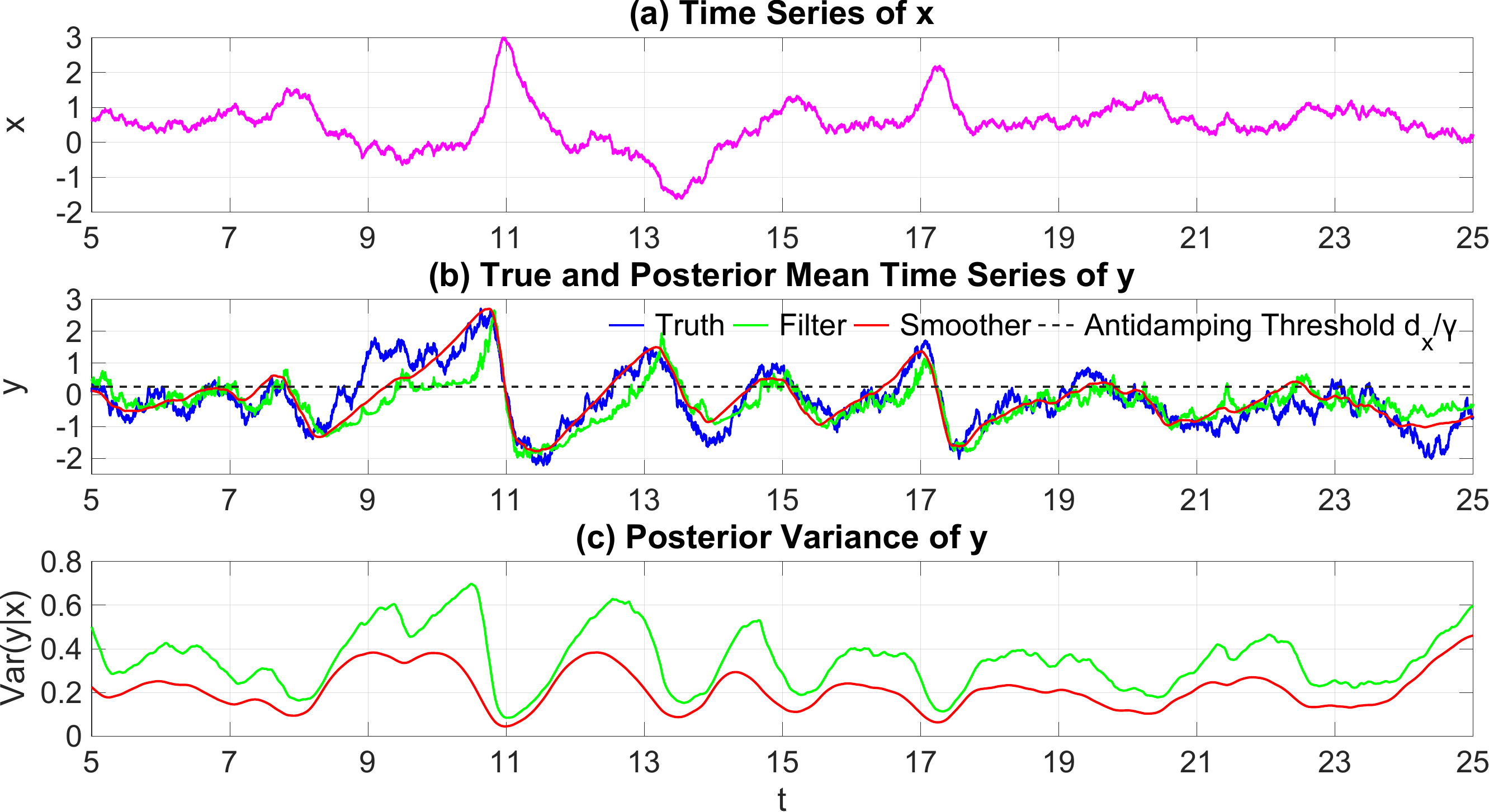}
\caption{Data assimilation of the dyad model \eqref{eq:Dyad_model_SI}. Panel (a): A single realization of the observed variable $x$. Panel (b): The true hidden signal $y$ (blue) alongside the posterior mean estimates from filtering (green) and smoothing (red), where the smoother is the complete smoother using all the information in future. The dashed line marks the anti-damping threshold, above which the net damping $-d_x + \gamma y$ in \eqref{eq:Dyad_model_x_SI} becomes positive. Panel (c): Posterior variance of the filtered and smoothed estimates of $y$.}
\label{Fig:dyad_interaction_fig_1}
\end{figure}

Supplementary Figure \ref{Fig:dyad_interaction_fig_2} displays the ACI and CIR analyses. Several important patterns emerge:
\begin{itemize}
  \item First, the ACI value (Panel (d)) reaches its maximum during $y$'s strongest anti-damping phase, corresponding to $y$'s peak instantaneous influence on $x$.
  \item Second, the subjective CIR attains its highest values (shown by the deep red shading in Panel (b)) slightly before $y>d_x/\gamma$ (initiating an extreme event) and for small $\varepsilon$ thresholds (indicating a longer-range influence). This reveals that extreme events develop gradually, with triggering conditions established well in advance. Notably, the objective CIR's temporal extent does not reach the actual peak of extreme events. This finding mirrors the decorrelation time (the integration of the autocorrelation function) in complex dynamical systems, where the true causal influence often persists weakly beyond the formal timescale indicated by the objective function.
  
  \item Third, and most significantly, the objective CIR shortens as the system approaches $y$'s positive peak (corresponding to $x$'s extreme event buildup phase). At this stage, the filter can reliably detect the emerging pattern without requiring future information. This transition naturally partitions the time series at each extreme event, marking distinct dynamical regimes: a build-up phase with long-range dependence (where future information improves estimation) and an event phase where the trajectory becomes locally predictable. The short CIR also persists during $y$'s demise, where $x$ is the driving factor behind the system dynamics (high signal-to-noise ratio) and controls $y$ via $-\gamma x^2$.

  \item Finally, Panel (c) showcases the temporal evolution of the objective CIR for $y(t)\rightarrow x$, calculated using both the definition in \eqref{eq:objective_CIR_length} and its approximation in \eqref{eq:objective_CIR_length_efficient}. This confirms the results of Theorem \ref{thm:obj_subj_CIR_connection} and \eqref{eq:objective_CIR_length_efficient}, and clearly demonstrates the high skill that this computationally efficient approximation achieves while circumventing the expensive quadrature entailed by the $\varepsilon$-integrals of the subjective CIRs in the definition. Significantly, it captures the exact profile of the exact definition. This is true both during the quiescent phases of the signal, but importantly also at the onset of the extreme events generated by $y$.
\end{itemize}
These results demonstrate that extreme events in this system are not sudden occurrences, but rather the outcome of gradually evolving conditions. The triggering mechanism begins significantly earlier than the actual event, with effects that propagate both forward and backward in time, as evidenced by the CIR patterns.

\begin{figure}[!ht]%
\centering
\includegraphics[width=1\textwidth]{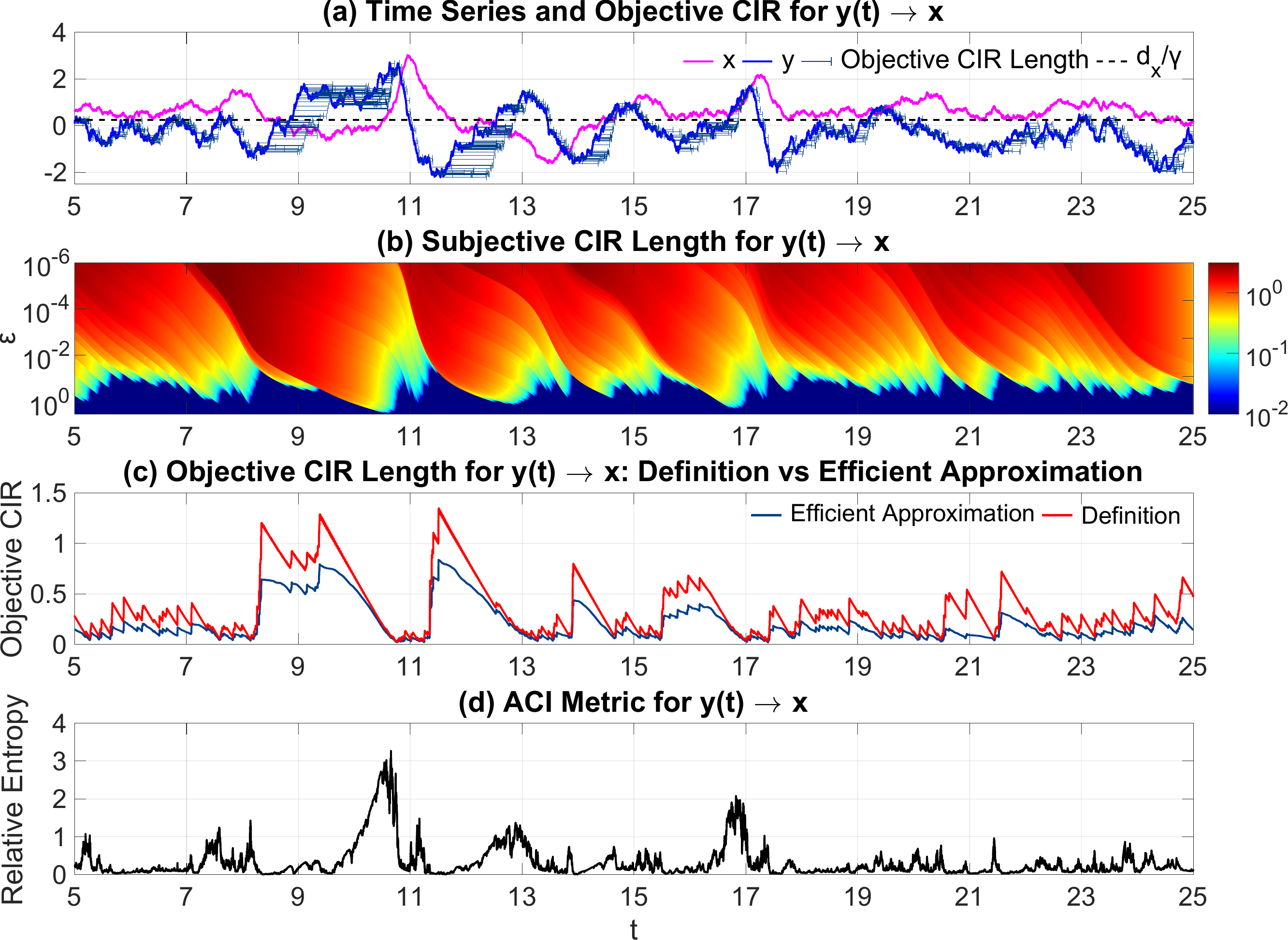}
\caption{ACI values and CIRs for the nonlinear dyad model \eqref{eq:Dyad_model_SI} from $y$ to $x$ as functions of time. Panel (a): Time series of $x$ (magenta) and $y$ (blue), with the objective CIR depicted as whiskers extending forward in time from each $y(t)$. The dashed horizontal line marks the anti-damping threshold $d_x/\gamma$. Panel (b): Subjective CIR (shaded region, logarithmically scaled) as a function of the threshold $\varepsilon$ (logarithmic, reversed $y$-axis). Panel (c): Objective CIRs from $y$ to $x$ over time, calculated using both the definition \eqref{eq:objective_CIR_length} (in red) and the computationally efficient approximation \eqref{eq:objective_CIR_length_efficient} (in navy blue), corresponding to the whiskers in Panel (a). Panel (d): ACI values from $y$ to $x$ over time.}
\label{Fig:dyad_interaction_fig_2}
\end{figure}

\subsection{A Noisy Predator--Prey Model}

The predator--prey model (also known as the Lotka-Volterra model) is fundamentally important across scientific disciplines as it captures the universal dynamics of interacting populations through simple yet powerful mathematics  \cite{cooke1981mathematical}. Originating in ecology to explain cyclical fluctuations between species like lynx and hares, its core principles have been adapted to model diverse systems, from disease spread in epidemiology to competition in economics and even chemical oscillations. The nonlinear feedback mechanisms in the model provide crucial insights into understanding stability, resilience, and emergent patterns of many natural and scientific problems.

Let us consider a noisy version of the predator--prey model:
\begin{subequations} \label{eq:predator_prey}
\begin{align}
    \d x &= (\beta xy -\alpha x)\d t + \sigma_x\rmd W_x,\label{eq:predator} \\
    \d y &= (\gamma y - \delta xy)\d t + \sigma_y\rmd W_y, \label{eq:prey}
\end{align}
\end{subequations}
In \eqref{eq:predator_prey}, $x$ and $y$ represent the population densities of a predator species and its prey, respectively, with their time derivatives ($\mathrm{d}x/\mathrm{d}t$ and $\mathrm{d}y/\mathrm{d}t$) describing their instantaneous population growth rates. The predator dynamics are governed by two parameters: $\alpha$, the predator's natural death rate, and $\beta$, which quantifies how prey availability enhances predator growth. The prey dynamics depend on $\gamma$, which is the maximum intrinsic growth rate of the prey, and $\delta$, which captures the negative impact of predators on the prey population. To ensure results are biologically realistic in finite-length simulations, small additive noise terms ($\sigma_x$ and $\sigma_y$) are included, preventing populations from reaching nonphysical negative values. The parameter values in the study here are as follows:
\begin{gather*}
\alpha = 0.4,\quad \beta = 0.1,\quad \sigma_x = 0.3,\quad \gamma = 1.1,\quad \delta = 0.4, \quad \sigma_y = 0.3.
\end{gather*}
Since both $\beta$ and $\delta$ are positive, larger prey population $y$ enhances the anti-damping effect in the $x$ equation, while larger predator population $x$ intensifies the damping in $y$. The resulting coupled variations in $x$ and $y$ produce intermittent phase alternations in the system dynamics. Note that since \eqref{eq:predator_prey} is conditionally linear both in $x$ and $y$, it is a bidirectional CGNS, meaning both posterior distributions, $x|y$ and $y|x$, are Gaussian.

Supplementary Figure \ref{Fig:predator_prey_fig_1} displays the data assimilation results. Panels (a)--(c) and (d)--(f) show the results by observing $y$ (recovering $x$) and observing $x$ (recovering $y$), respectively. The state estimation is more informative when the observed signal has a large value, corresponding to when the estimated variable induces an extreme event in the former's evolution. Additionally, the uncertainty reduction in the smoother related to the filter is more significant when the predator $x$ is the observed variable, which accounts for the choice of the coupling, quadratic feedback parameter values; $\delta>\beta$.

\begin{figure}[!ht]%
\centering
\includegraphics[width=1\textwidth]{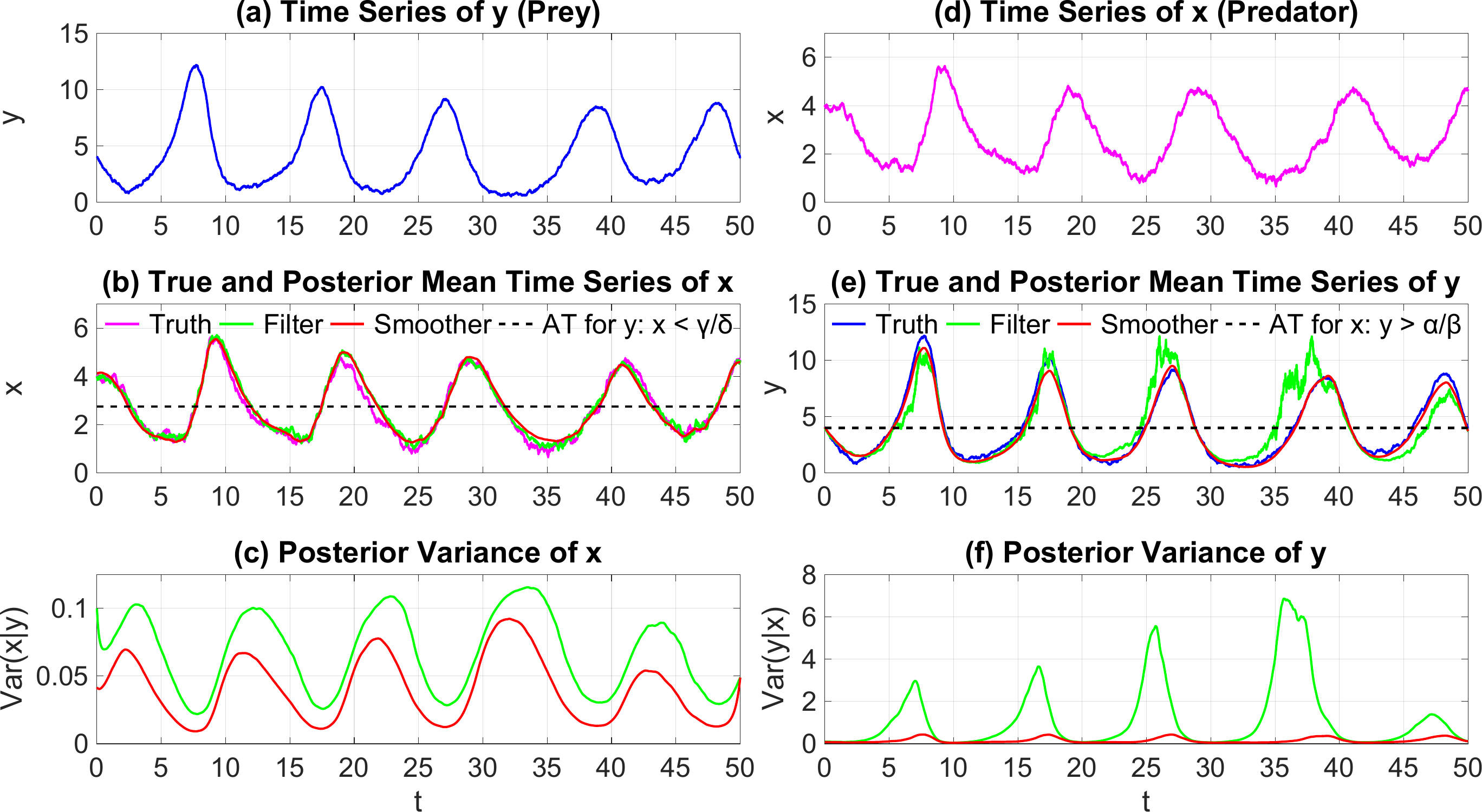}
\caption{Data assimilation of the noisy predator--prey model \eqref{eq:predator_prey}. Panels (a)--(c) and (d)--(f) show the results by observing $y$ (recovering $x$) and observing $x$ (recovering $y$), respectively. The dashed lines in Panels (b) and (e) indicate the anti-damping threshold values.}
\label{Fig:predator_prey_fig_1}
\end{figure}

Supplementary Figure \ref{Fig:predator_prey_fig_2} presents the ACI value (instantaneous causal strength) and CIR (influence duration) between predator ($x$) and prey ($y$) populations. The interaction exhibits two distinct regimes:
\begin{enumerate}
  \item Predator-to-Prey Causality ($x$-to-$y$; Panels (a)--(b)):
  \begin{itemize}
    \item When the predator population $x$ is below the threshold $\gamma/\delta$, the net damping in the prey equation becomes positive (anti-damping phase), allowing $y$ to grow, though at a progressively slower rate as $x$ increases.
    \item Once $x$ exceeds $\gamma/\delta$, strong positive damping emerges causing $y$ to decline. This demonstrates how predator growth first suppresses then reverses prey population trends.
    \item Notably, when $y$ is small, $x$ shows weak instantaneous influence (low ACI values) but exhibits extended CIRs, which reveals how predator reduction leads to delayed prey resurgence.
  \end{itemize}
  \item Prey-to-Predator Causality ($y$-to-$x$; Panels (c)--(d)):
  \begin{itemize}
    \item The prey population $y$ acts as an anti-damping term for predators during its abundance (when $y>\alpha/\beta$), directly driving $x$ growth with persistent temporal effects (long CIRs).
    \item The subsequent prey collapse (sharp $y$ decrease) when the predator population reaches the critical quantity ($x>\gamma/\delta$) reflects predator overconsumption rather than causing predator dynamics (causal link reversal), which is consistent with the $x$-to-$y$ causality shown in Panels (a)--(b).
    \item Due to the stronger coupling feedback in $y$, $\delta>\beta$, the ACI metric for $y(t)\rightarrow x$ is stronger. Furthermore, the damping effect that $x$'s growth induces has a more immediate impact on the prey population $y$, when compared to the more delayed effect that prey prosperity $y$ has on the predators $x$.
  \end{itemize}
\end{enumerate}
Remarkably, the causal relationship between predator ($x$) and prey ($y$) is bidirectional during specific phases. Prior to the prey population peak, while $y$ is growing but $x$ remains below the anti-damping threshold, the variables exhibit strong mutual interaction: the increasing prey population $y$ drives predator growth ($x$), while simultaneously, the rising predator population $x$ suppresses (but does not yet reverse) the prey's growth rate. This creates a transient period of coupled positive feedback ($y \to x$) and negative feedback ($x \to y$).

\begin{figure}[!ht]%
\centering
\includegraphics[width=1\textwidth]{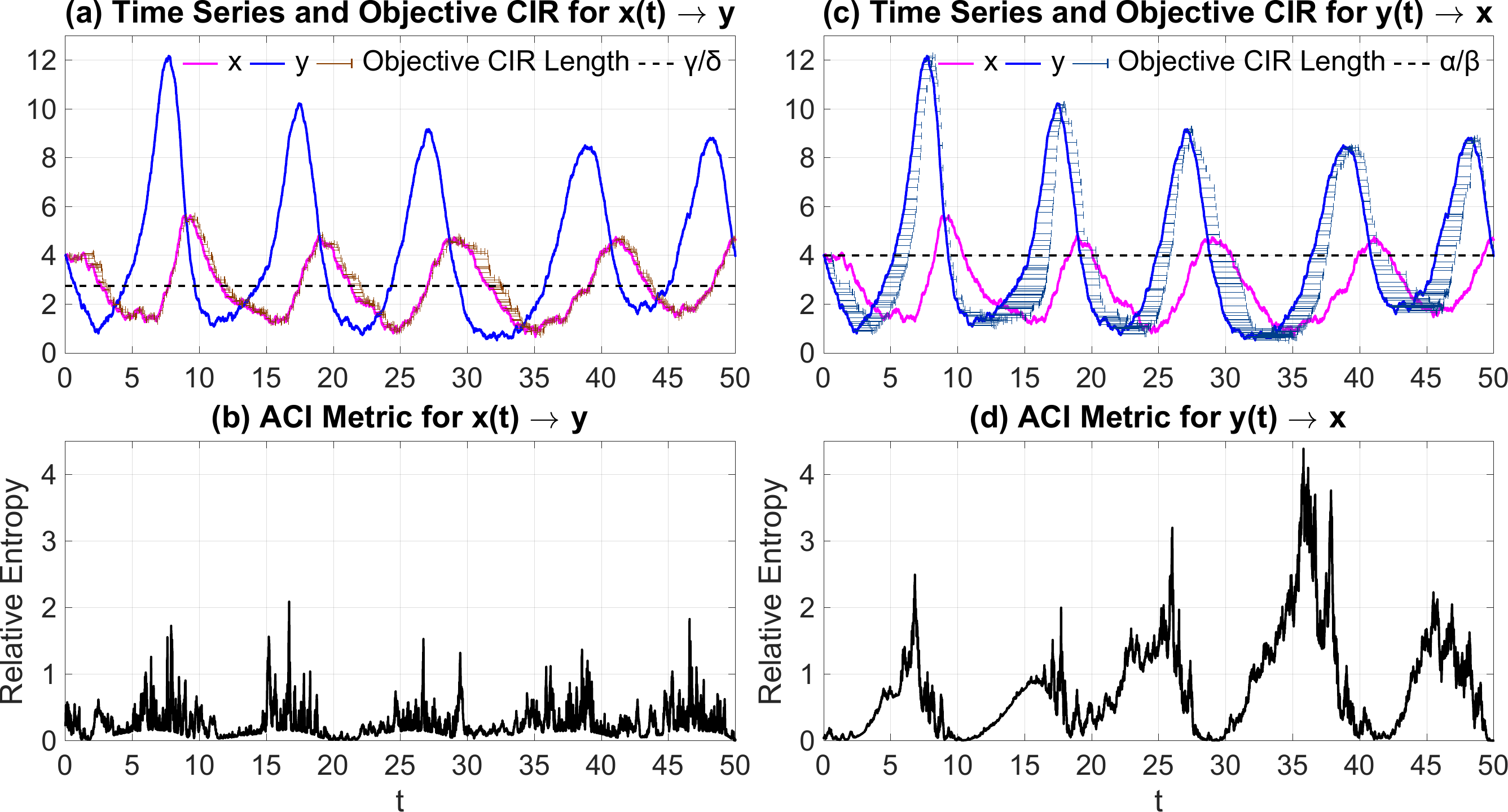}
\caption{ACI values and CIRs for the noisy predator--prey model \eqref{eq:predator_prey}. Only the objective CIRs are shown. Panels (a)--(b) and (c)--(d) show the results from $x(t)$ to $y$ and from $y(t)$ to $x$, respectively. The dashed lines in Panels (a) and (b) indicate the anti-damping threshold values in the equations of $y$ and $x$, respectively.}
\label{Fig:predator_prey_fig_2}
\end{figure}

\subsection{A Stochastic Model Capturing the El Ni\~no--Southern Oscillation (ENSO) Diversity}

\subsubsection{Model Details}

As was described in the main text, although few models can accurately capture ENSO diversity, a recently developed stochastic conceptual model successfully reproduces its diverse behaviors and non-Gaussian statistics \cite{chen2022multiscale}. This model has been highlighted in a recent review \cite{vialard2025nino}, making it a suitable testbed for studying El Ni\~no diversity. Using the defined abbreviations from the main text, the model consists of six state variables: ocean zonal current in the CP ($u$), WP thermocline depth ($h_W$), CP SST ($T_C$), EP SST ($T_E$), atmospheric winds ($\tau$, intraseasonal), and background Walker circulation ($I$, decadal). The variables ($u$, $h_W$, $T_C$, and $T_E$) operate on interannual timescales and model the anomalies away from the corresponding climatology. As a nonlinear system with state-dependent noise, the model generates extreme events and intermittency. Note that the decadal variable represents the strength of the Walker circulation. The two SST variables ($T_C$ and $T_E$) allow reconstruction of spatiotemporal SST patterns across the equatorial Pacific, providing an intuitive way to identify different ENSO event types. The model reads:
{\allowdisplaybreaks
\begin{subequations}\label{eq:conceptual_model}
\begin{align}
    \d u&=\big(-ru-\delta_u\frac{T_C+T_E}{2}+\beta_u(I)\tau\big)\rmd t+\sigma_u\rmd W_u, \label{eq:adv}\\
    \d h_W&=\big(-rh_W-\delta_h\frac{T_C+T_E}{2}+\beta_h(I)\tau\big)\rmd t+\sigma_h\rmd W_h, \label{eq:thermocline}\\
    \d T_C&=\big(\left(r_C-c_1(t,T_C)\right)T_C+\zeta_CT_E+\gamma_Ch_W+\sigma(I)u+C_u+\beta_C(I)\tau\big)\rmd t+\sigma_C\rmd W_C,\label{eq:sstac}\\
    \d T_E&= \big(\left(r_E-c_2(t)\right)T_E-\zeta_ET_C+\gamma_Eh_W+\beta_E(I)\tau\big)\rmd t+\sigma_E\rmd W_E,\label{eq:sstae}\\
    \d \tau&=-d_\tau\tau\rmd t+\sigma_\tau(t,T_C)\rmd W_\tau, \label{eq:wind}\\
    \d I&=-\lambda(I-m)\rmd t+\sigma_I(I)\rmd W_I. \label{eq:walker}
\end{align}
\end{subequations}
}
While the full model details are available in \cite{chen2022multiscale}, its basic mechanism can be summarized as follows. EP El Ni\~no events are typically associated with a strong thermocline buildup in the WP, whereas CP EN and La Ni\~na events are primarily triggered by advective processes. In the latter case, when $\sigma(I) \propto I$ is larger, indicating a strengthening Walker circulation, the advection term $u$ becomes dynamically significant in driving $T_C$. Unlike the standard discharge--recharge oscillator model, which couples $h_W$ and $T_E$, the $T_C$ variable acts as a transitional component that is expected to have a more direct influence on $T_E$.

\subsubsection{Additional Results Using Model-Generated Data}

Supplementary Figure \ref{Fig:ENSO_T_E_fig} shows conditional ACI values and CCIRs for model-simulated EP events (target variable: $T_E$). We note that the variables not appearing in the title of each column are resolved in the corresponding causal relationship through the conditional ACI framework (similarly for Supplementary Figure \ref{Fig:ENSO_T_C_fig}). Among the three potential causal variables, $T_C$ has the strongest ACI value, peaking slightly before $T_E$ during EP El Ni\~no events (red positive $T_E$ anomalies). This timing makes physical sense because SSTs in these regions are strongly coupled: during El Ni\~no, warm water moves from CP to EP, explaining why $T_C$ leads. The $\tau$-$T_E$ ACI value is noisier due to $\tau$'s short-term variability, but still shows $\tau$'s clear impact on $T_E$. Winds both push warm water and directly affect SSTs quickly. While $h_W$ does influence $T_E$, its ACI value is weaker than $T_C$ or $\tau$. According to discharge--recharge theory \cite{jin1997equatorial}, $h_W$ and $T_E$ form an oscillator, but in models with CP resolution, $h_W$ affects $T_E$ indirectly: first changing $T_C$, which then affects $T_E$ (WP$\rightarrow$CP$\rightarrow$EP). This explains why $h_W$'s ACI value peaks months before EP El Ni\~no maxima, specifically during periods where $h_W$ starts to increase thus initiating upwelling processes that generate extreme warming centers in the EP. The CIRs further corroborate these underlying mechanisms in the temporal sense: $T_C$ has the longest influence, $h_W$'s more indirect role gives medium-length CIRs, and $\tau$'s shortest-term variability leads to the briefest impacts.

\begin{figure}[!ht]%
\centering
\includegraphics[width=1\textwidth]{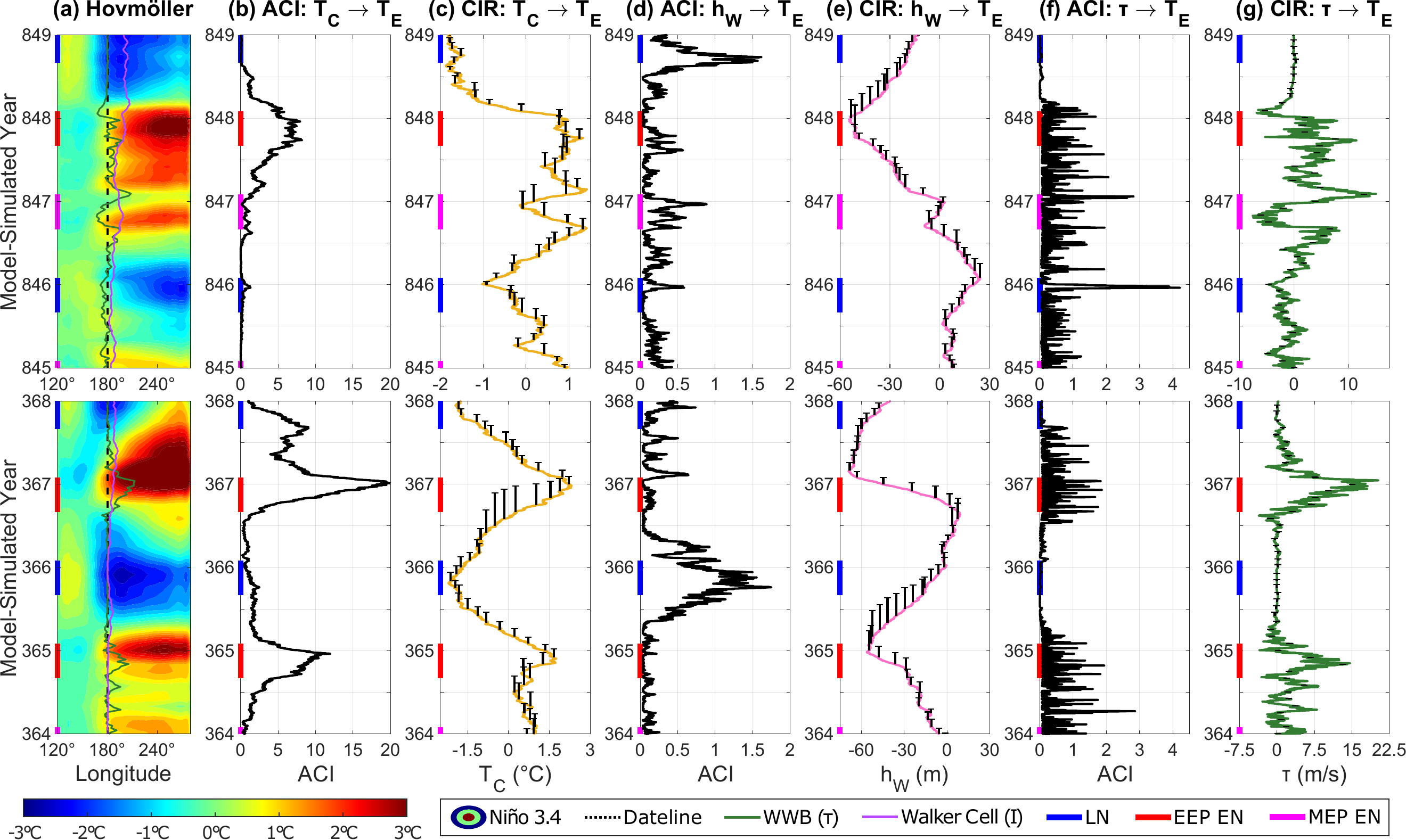}
\caption{Conditional ACI values and CCIRs for some model-simulated EP events (EEP EN and MEP EN stand for extreme and moderate EP El Ni\~no, respectively, and LN for La Ni\~na); $t\in[364,368]$ and $t\in[845,849]$. Panel (a): Hovm\"{o}ller diagram of the SST anomalies in the equatorial Pacific reconstructed from $T_C$ and $T_E$  via a spatiotemporal multivariate regression technique \cite{chen2022multiscale}. The wind profile $\tau$, plotted above and centered around the dateline, is superimposed along with the decadal variable $I$, whose time series is also positioned relative to the dateline: Positive $\tau$ (indicating westerly winds) appear to the right of the dateline, while negative values appear to the left, and a weakening Walker circulation appears close to the dateline, while a stronger $I$ appears further to its right. Panel (b): Conditional ACI values from $T_C$ to $T_E$. Panel (c): Time series of $T_C$ and the corresponding CCIRs from $T_C$ to $T_E$ (depicted as whiskers at each month, similar to Panel (a) of Supplementary Figure \ref{Fig:dyad_interaction_fig_2}). Panels (d)--(e) and (f)--(g): Similar to (a)--(b) but for $h_W$ and $\tau$. The two rows represent two distinct periods. The Walker circulation in both cases is nearly zero, which is the usual condition to trigger EP events.}
\label{Fig:ENSO_T_E_fig}
\end{figure}

Supplementary Figure \ref{Fig:ENSO_T_C_fig} shows conditional ACI values and CCIRs for model-simulated CP events (target variable: $T_C$). When the decadal variable $I$ is significant, $u$ contributes substantially to CP events through the enhanced zonal advective feedback ($\sigma(I) \propto I$) according to \eqref{eq:sstac}, peaking slightly before CP El Ni\~no events. The WP thermocline ($h_W$) also affects $T_C$ via positive feedback in \eqref{eq:sstac}, but peaks earlier than the CP event. While the $u$ and $h_W$ ACIs sometimes coincide due to state-space correlations, $h_W$'s CIR persists longer due to WP-CP information transfer delays. Wind ($\tau$) influences CP events on shorter intraseasonal timescales. Unlike EP events, CP events show balanced contributions from $u$, $h_W$, and $\tau$. The strong ACI value from $u$ to $T_C$ around $t=1985$ corresponds to the rapid transition from CP El Ni\~no to La Ni\~na, where the value of $u$ changes dramatically from positive to negative. Furthermore, during the second year of the multi-year CP La Ni\~na events (i.e., $t=1986$), all three ACIs drop suddenly, as the signals of all these three variables remain near zero. This corresponds to a typical discharge phase of ENSO.

\begin{figure}[!ht]%
\centering
\includegraphics[width=1\textwidth]{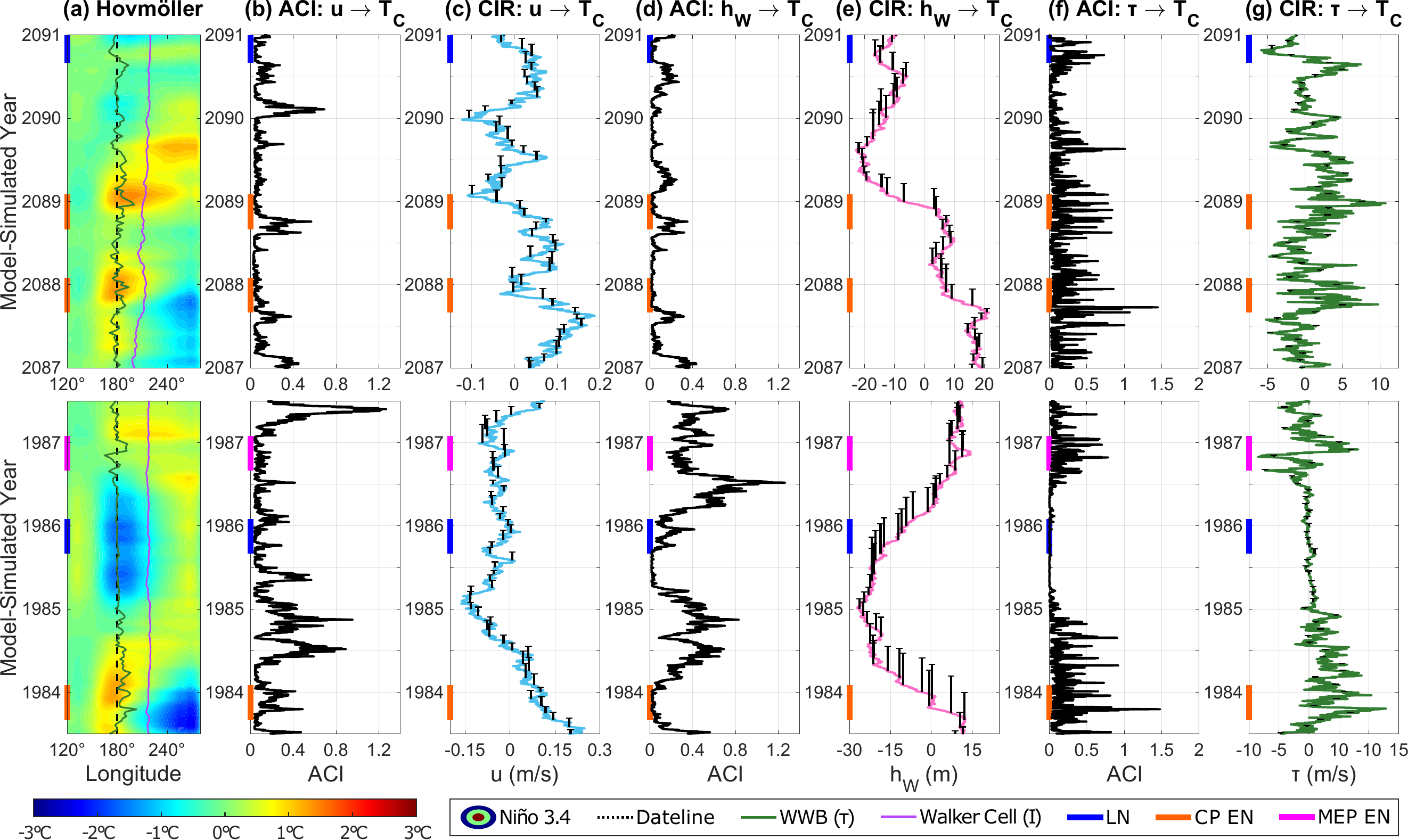}
\caption{Conditional ACI values and CCIRs for some model-simulated CP events (CP EN and MEP EN stand for CP and moderate EP El Ni\~no, respectively, and LN for La Ni\~na); $t\in[1983.5,1987.5]$ and $t\in[2087,2091]$. Panel (a): Same as Panel (a) of Figure \ref{Fig:ENSO_T_E_fig}. Panel (b): Conditional ACI values from $u$ to $T_C$. Panel (c): Time series of $u$ and the corresponding CCIRs from $u$ to $T_C$ (depicted as whiskers at each month, similar to Panel (a) of Supplementary Figure \ref{Fig:dyad_interaction_fig_2}). Panels (d)--(e) and (f)--(g): Similar to (a)--(b) but for $h_W$ and $\tau$. The Walker circulation in this case is strong, which is the usual condition to trigger CP events.}
\label{Fig:ENSO_T_C_fig}
\end{figure}

\subsubsection{ACI on Real-World ENSO Data}
In this subsection, we apply the ACI framework on real-world observational data to study the ENSO diversity through \eqref{eq:conceptual_model}.

The oceanic data used here, including the monthly SST, zonal current velocity, and thermocline depth (calculated from the potential temperature at the depth of the 20$^{\circ}$C isotherm) measurements, are sourced from the GODAS dataset (\url{https://www.psl.noaa.gov/data/gridded/data.godas.html}) \cite{behringer2004evaluation}. The present study solely uses oceanic anomalies, calculated by removing the monthly mean climatology over the whole analysis period. The data spans from January 1982 to December 2017, chosen to fall within the satellite era when observations are more reliable. The atmospheric wind data is obtained from the National Centers for Environmental Prediction-National Center for Atmospheric Research (NCEP-NCAR) reanalysis data set (\url{https://psl.noaa.gov/data/gridded/data.ncep.reanalysis.html}) \cite{kalnay2018ncep} over the same period.

The model used within ACI remains the stochastic conceptual model defined in \eqref{eq:adv}--\eqref{eq:walker}. However, because the observations are now drawn from the real world rather than from synthetic data, this study no longer has a perfect-model setup. Nevertheless, as shown in \cite{chen2022multiscale}, the conceptual model in \eqref{eq:adv}--\eqref{eq:walker} successfully reproduces key statistical features of observed ENSO variability and generates time series that are qualitatively consistent with real ENSO episodes. Hence, the model error is expected to be relatively small.

Daily data are available for the wind burst amplitudes $\tau$, whereas all other variables are monthly. To ensure the stability of the filter, smoother, and online-smoother algorithms required for the conditional ACI analysis and the computation of the CCIRs, we linearly interpolate all data onto a uniform temporal grid with a spacing of approximately 7 hours and 20 minutes. Linear interpolation has been shown to be more compatible with data assimilation than smoother interpolation methods \cite{harlim2011interpolating}. Since the oceanic variables primarily evolve on interannual timescales, this approximation is not expected to alter the data’s underlying dynamics. This procedure also provides an opportunity to test the robustness of the ACI framework with respect to potential numerical discretization errors.

\begin{figure}[!ht]%
\centering
\includegraphics[width=1\textwidth]{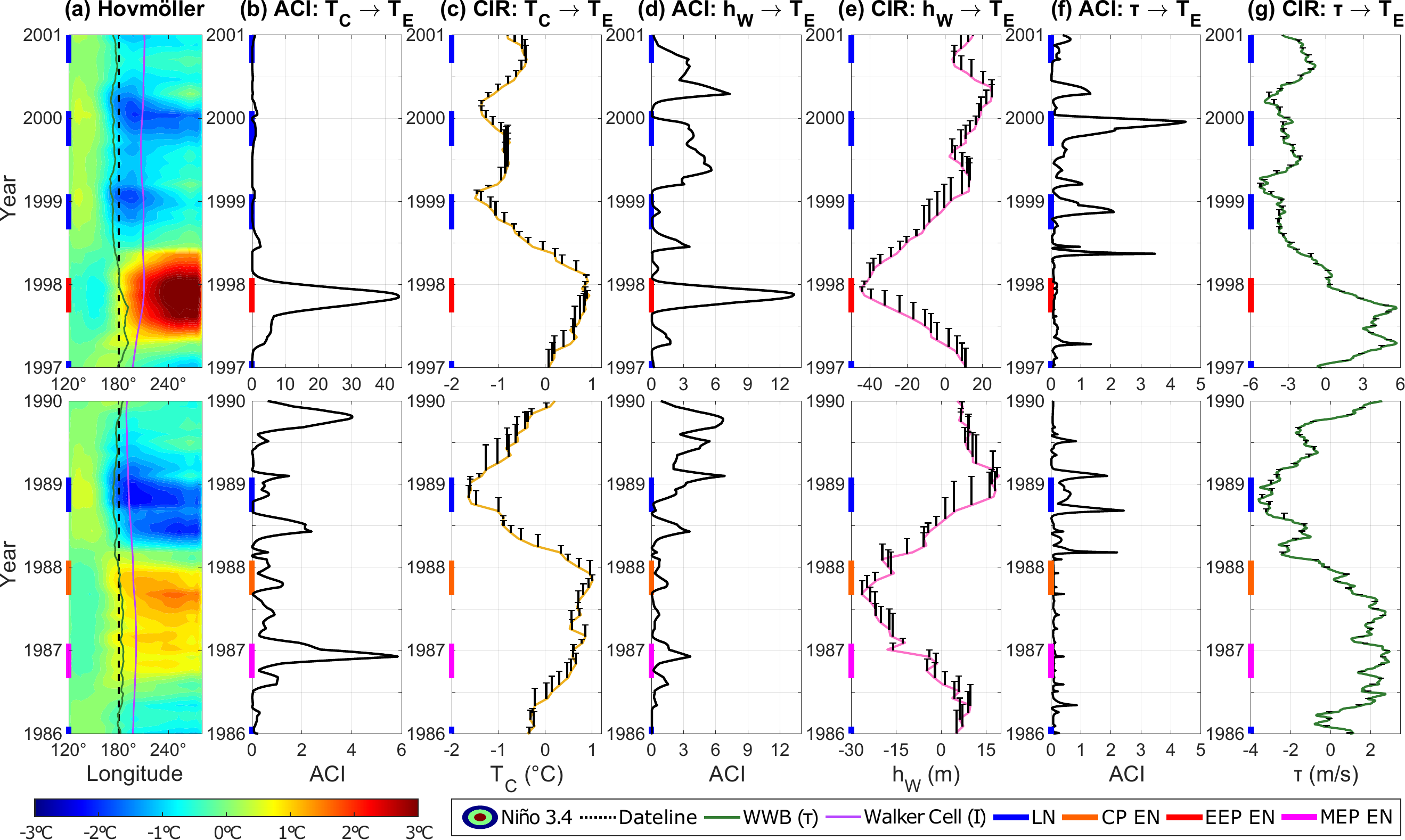}
\caption{Conditional ACI values and CCIRs for some real-world EP events; Similar to Supplementary Figure \ref{Fig:ENSO_T_E_fig} but the observations are collected from real-world data.}
\label{Fig:ENSO_real_obs_T_E_fig}
\end{figure}

\begin{figure}[!ht]%
\centering
\includegraphics[width=1\textwidth]{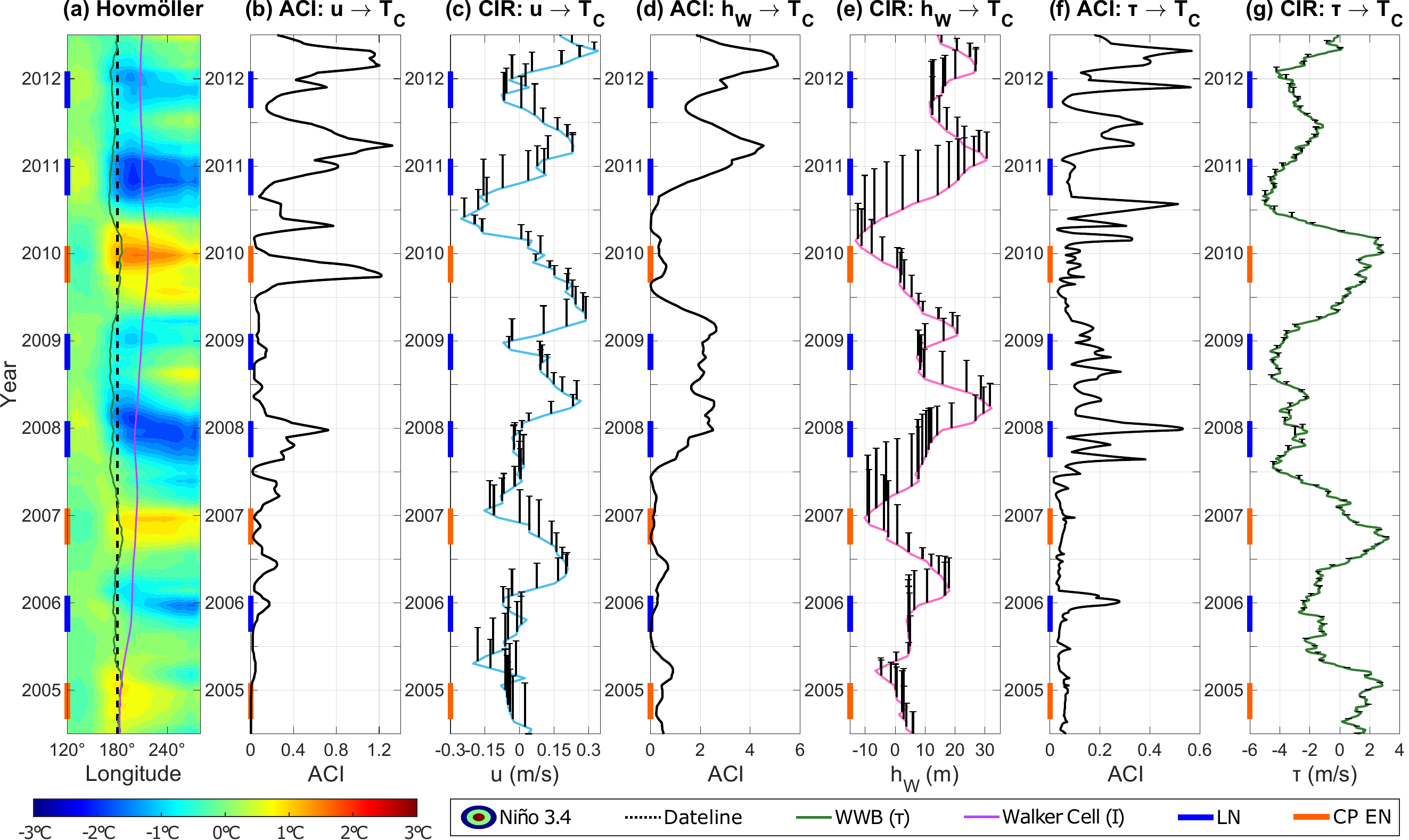}
\caption{Conditional ACI values and CCIRs for some real-world CP events; Similar to Supplementary Figure \ref{Fig:ENSO_T_C_fig} but the observations are collected from real-world data.}
\label{Fig:ENSO_real_obs_T_C_fig}
\end{figure}

Supplementary Figure \ref{Fig:ENSO_real_obs_T_E_fig} includes the historically strongest EP El Ni\~no event in 1998 and a moderately strong EP event in 1987, together with their subsequent La Ni\~na phases. Supplementary Figure \ref{Fig:ENSO_real_obs_T_C_fig} presents several CP events that occurred in the early 21st century. The results based on real-world observations show qualitatively similar conclusions to those in Supplementary Figures \ref{Fig:ENSO_T_E_fig}--\ref{Fig:ENSO_T_C_fig}, which use model-generated synthetic data.

Specifically, during EP El Ni\~no phases, the causal links $T_C \rightarrow T_E$, $h_W \rightarrow T_E$, and $\tau \rightarrow T_E$ are identified. In addition, the thermocline continues to play a key role during EP La Ni\~na events through the recharge--discharge mechanism, while there are no clear causal influences from $T_C$ and $\tau$ to $T_E$, consistent with the classical discharge process. Likewise, during CP El Ni\~no phases, $u \rightarrow T_C$ and $\tau \rightarrow T_C$ emerge as the two most significant causal relationships, while during CP La Ni\~na phases $h_W \rightarrow T_C$ is also detected. All of these findings are consistent with those obtained from the model-generated experiments.

The results presented here suggest that ACI performs reasonably well in the presence of noisy observations, moderate model error, and some numerical discretization effects. While more systematic studies are needed to fully quantify the influence of these factors, the results here provide some preliminary evidence that the ACI framework maintains a degree of robustness under these conditions.

Finally, it is worth emphasizing the distinction between ACI and many baseline causal inference methods for analyzing real-world ENSO datasets \cite{risser2025granger, runge2023causal, falasca2024data}. Most of the current approaches are purely data-driven and focus on identifying broad, time-averaged cause-and-effect relationships. Their findings are generally consistent with physical mechanisms, such as the causal influence of the thermocline depth on the SST and the discharge--recharge loop linking El Ni\~no and La Ni\~na. Yet, ACI provides a more refined and event-specific indicator that reveals how individual variables contribute to the causal structure of each ENSO episode, an aspect often missing from state-of-the-art techniques. ACI yields markedly different causal strengths and CIR patterns across events, providing valuable insight into instantaneous, time-localized cause-and-effect behavior. As a natural direction for future work, the ACI framework can be coupled with a range of ENSO models and expanded to include additional variables, thereby enhancing causal discovery with a particular focus on instantaneous relationships.

\end{appendices}

\renewcommand{\refname}{Supplementary Information References}

\putbib

\end{bibunit}

\end{document}